\documentclass[11pt,a4paper]{article}

\usepackage[a4paper,margin=1in]{geometry}
\usepackage{fontspec}
\usepackage{amsmath,amssymb,amsthm}
\usepackage{unicode-math}
\usepackage{graphicx}
\usepackage{booktabs}
\usepackage{longtable}
\usepackage{array}
\usepackage{xcolor}
\usepackage{caption}
\usepackage[hidelinks,colorlinks=true,citecolor=blue!60!black,linkcolor=blue!50!black,urlcolor=blue!70!black]{hyperref}
\usepackage{enumitem}
\usepackage[numbers,sort&compress]{natbib}
\usepackage{microtype}
\usepackage{authblk}
\usepackage{float}
\usepackage{tikz}
\usetikzlibrary{arrows.meta,positioning,calc}


\theoremstyle{plain}
\newtheorem{proposition}{Proposition}
\newtheorem{corollary}{Corollary}
\theoremstyle{definition}
\newtheorem{definition}{Definition}
\theoremstyle{remark}
\newtheorem{remark}{Remark}

\newcommand{\shuf}{\textsc{shuf}}
\newcommand{\only}{\textsc{only}}
\newcommand{\blocked}{\textsc{blocked}}
\newcommand{\purerand}{\textsc{purerand}}
\newcommand{\Astep}{\text{A}\!\to\!\text{B}}
\definecolor{cV}{RGB}{31,110,168}
\definecolor{cA}{RGB}{201,120,20}
\definecolor{cD}{RGB}{158,50,120}
\definecolor{cGrey}{RGB}{110,110,110}

\title{\textbf{Conflicting Supervision Moves Commitment, Not Capability}\\[4pt]
{\large A $12.29\sigma$ arrangement effect that is exactly zero under a convention-agnostic score}}

\newif\ifanon
\anonfalse                 
\ifanon
  \author[ ]{Anonymous Author(s)}
  \affil[ ]{Paper under double-blind review}
\else
  \author[1]{Wenhui Chen\thanks{\texttt{mc35092@um.edu.mo}}}
  \affil[1]{University of Macau}
\fi
\date{}

\ifanon

  \newcommand{\thecompanion}{the concurrent}
  \newcommand{\Thecompanion}{The concurrent}
\else

  \newcommand{\thecompanion}{the companion}
  \newcommand{\Thecompanion}{The companion}
\fi



\begin{document}
\maketitle

\begin{abstract}
\noindent
Train a model on the same problems written under two incompatible conventions, both correct, and
ask what the \emph{ordering} of that data writes into the parameters.

\textbf{The learning-rate schedule is not a background condition for that question. It is the
averaging operator, and it decides the answer.} We prove a bound in which the arrangement and
the schedule enter the ordering effect as \emph{separate multiplied factors}: the arrangement only as
a block period, the schedule only as how much weight the endpoint can place on any one moment of the
run. A decaying schedule cannot put a large step size and an uncontracted remainder at the same
moment; a constant one does exactly that at the last step. That decay moderates ordering effects has
been reported in pretraining \citep{lrdecay2025curriculum}; the mechanism, the separation, and a
controlled measurement of both halves are ours. Ten orderings of one corpus, one
budget, everything but the path held fixed, run twice under families differing in
\texttt{lr\_scheduler\_type} and nothing else: at a constant rate the interior spans $0.2221$ in
allocation, $11.63$ contrast floors, monotone in how blocked the arrangement is. Under the single cosine every published arm uses, \emph{the same ten arms occupy \textbf{two}
distinguishable states where their own resolution would allow about ten}, across a $54\times$
change of block length; the dispersion between arms does not exceed the seed noise within them
(intraclass correlation $-0.076$, $[-0.409,+0.166]$, three seeds per arm). \emph{``Order matters''
and ``order does not matter'' are the two ends of one knob}, which is what the divided record on
ordering looks like from here.

The theory predicted, and we falsified, the three mechanisms we had pre-registered. It also
explains the one arm that does not move between the two schedules: at blocked training the
arrangement factor is the whole run rather than a short period, so the endpoint is set by the whole
of the schedule's profile and deleting its tail does almost nothing.

\textbf{What the path writes is which convention the model commits to, and no exact-match
benchmark can see it.} Across twelve arms $\mathrm{acc}_A+\mathrm{acc}_B$ is constant to within
$9.7\%$ while the allocation share runs $0.04$ to $0.87$, so the $12.29\sigma$ arrangement switch
this paper measures is \emph{exactly zero} under a convention-agnostic metric. That conservation is
quoted from the decayed family throughout, the constant-rate one being a noisier place to read it,
and we say where each of the two results is measured rather than merging them. Marking the convention in the prompt collapses the switch and reaches $87.5\%$
of the union ceiling.
\end{abstract}

\section{Newton's Apples That Disagree}\label{sec:limit}

Train a model on the same problems written under two incompatible conventions, both correct, and
one thing is obvious in advance: something is lost. The interesting question is \emph{what}.

\emph{The answer is that nothing need be lost from what the model can solve, and a great deal moves
in which correct form it commits to.} Twelve arrangements of one corpus hold
$\mathrm{acc}_A+\mathrm{acc}_B$ constant to within $9.7\%$ while the share written under one
convention runs from $0.04$ to $0.87$. The same measurement is a $12.29\sigma$ effect on one
coordinate and \emph{exactly zero} on the other, and which one a benchmark reports is a property of
the benchmark rather than of the model.

\emph{Whether that commitment survives training at all is decided by the learning-rate schedule},
and this is the paper's central result. We prove that the ordering-dependent part of an endpoint is
bounded by a product of two factors that do not communicate: the arrangement enters only through
the period of its alternation, and the schedule only through how much weight the endpoint can place
on any one moment of the run (Proposition~\ref{prop:sched}). A schedule decaying to zero cannot put
a large step size and an uncontracted remainder at the same moment; a constant one does exactly that
at the last step. The prediction is a knob, and the measurement is the knob at two settings: ten
arrangements of one corpus, one budget, the same three seeds, the two families differing in
\texttt{lr\_scheduler\_type} and in nothing else, span $2.20$ contrast floors under a cosine and
$11.63$ under a constant rate. \emph{``Order matters'' and ``order does not matter'' are the two ends
of one knob}, which is what the divided record on data ordering looks like from here.

\begin{quote}
\textbf{The claim in one sentence.} Conflicting supervision does not have to change what a model can
solve; it changes which correct convention the model commits to, and whether that commitment
survives to the endpoint is set by the schedule, not by the arrangement.
\end{quote}

\noindent \emph{Where the question comes from.} The Capability Convergence Hypothesis (CCH) organises
\emph{inference} around a bounded state fed by an unbounded stream, and separates a
\emph{compressive} channel that mixes the stream into $O(1)$ state from a \emph{verbatim index} that
pays to keep bindings addressable \citep{chen2026cch}. Let the bindings \emph{disagree} and the two
channels stop being interchangeable: a compressive learner forced to mix has a well-defined optimum,
the mixture, while an indexed one can keep both bindings and answer either way when the query names
one. This paper asks the training-time form of that question, with the data path as the stream and
the parameters as the bounded state. \emph{That is where the question came from and not what the
answer depends on}: every result below is derivable without any of the family's vocabulary, the
correspondence is audited in \S\ref{sec:family}, and Remark~\ref{rem:wallkind} states exactly where
it stops.

\paragraph{The path space, and the limit that collapses it.} Every arrangement of a two-source
corpus sits on one ladder: finish source $A$ entirely, then $B$ (\blocked); alternate blocks of
$L$ optimiser steps ($A^LB^L\cdots$); alternate every step ($L{=}1$); mix within each batch
(\shuf). \emph{Only the order varies}, and the order is a message: with $n$ rows from each source
it carries $\log_2\binom{2n}{n}$ bits, which grows without bound. The receiver is the parameter
vector at the end of training, and the question this paper asks is how many \emph{distinguishable
endpoints} it produces.

At the ideal end of the ladder the answer is none. Divide infinitely, $\eta\to0$ with alternation
frequency $\to\infty$, and the averaging theorem
\citep{bogoliubov1961asymptotic,kapitza1951pendulum} says the trajectory follows the mean-field
flow
\begin{equation}\label{eq:meanfield}
\dot\theta \;=\; -\nabla\!\big[\tfrac12 L_A+\tfrac12 L_B\big],
\end{equation}
whose cross-entropy minimiser at a contested input is $q^\star=(p_A+p_B)/2$. \emph{Under conflict
the infinitely divided limit installs the Bayes-optimal mixture}, a weighted coin flip between the
two conventions, and every one of those $\log_2\binom{2n}{n}$ orders lands in the same place. At
the budget anyone actually trains at the receiver is not that deaf, and \emph{the gap between the
limit and the budget is what this paper measures}.

\paragraph{How deaf it is, at the schedule the field uses.} Across every arm that differs only in
path ($L{=}1$ to $54$, \purerand, \shuf, \blocked; three seeds throughout) the allocation runs
$0.407$ to $0.869$. Against a seed dispersion of $\hat\sigma=0.0234$ that range would separate
about ten values at a resolution of $2\hat\sigma$. It separates \textbf{two}: nine arms between
$0.407$ and $0.492$, and \blocked\ alone at $0.869$. The count is not a threshold artefact, the
largest gap inside the occupied group being $0.0431$ against $0.3770$ to \blocked, a ratio of
$\mathbf{8.7}$, and a second pretraining family gives two groups at $\mathbf{20.7}$.

\emph{Every arm just counted trains under a cosine, and that is the point rather than a caveat.}
The same ten paths at a constant rate occupy \emph{four} states. So the two-state count is not a
fact about what a path can carry; it is what survives after a decaying schedule has integrated the
path away, which is Proposition~\ref{prop:sched} read as a measurement. The wall is a resolution
limit that the schedule sets, and the two mechanisms that get anything through it, a stopping
phase and a query-time key, are this paper's other results rather than its exceptions.

\paragraph{Contributions.}
\begin{enumerate}[leftmargin=1.5em,itemsep=3pt,topsep=2pt]
\item \textbf{The ordering effect factors, and the schedule is one of the two factors}
(Proposition~\ref{prop:sched}). The $\pi$-dependent part of an endpoint is bounded by
$\tau_L\cdot\sup_t\|w\|$: the arrangement enters only as a block period, the schedule only as the
weight the endpoint can place on one moment. \emph{We do not claim a new averaging theorem}: the
$\tau_L\to0$ corner is classical incremental-gradient theory (\S\ref{sec:related}), and we do not
claim the bound predicts a magnitude. What is new is that the two factors separate, which turns a
divided empirical record into one knob and makes three of this paper's measurements consequences of
each other rather than separate findings.
\item \textbf{The knob, measured at both settings} (\S\ref{sec:average}). Ten arrangements, one
corpus, one budget, the same three seeds, families differing only in \texttt{lr\_scheduler\_type}:
$2.20$ contrast floors under a cosine against $\mathbf{11.63}$ under a constant rate, monotone in how
blocked the arrangement is. \blocked\ moves $1.0$ floor between the two against $17.0$ at $L{=}54$,
which is the bound's corner case ($\tau_L=T$) and not a coincidence.
\item \textbf{A decomposition that separates what a model can do from which form it writes}
(Definition~\ref{def:decomp}), computable from the per-problem scores an evaluation already
produces. Under it this paper's largest effect and a null are one measurement on two axes, which is
a warning about any benchmark whose material admits more than one correct form.
\item \textbf{The stopping phase, and a query-time key.} \blocked\ is not a distinct mechanism but
the lowest-frequency alternation stopped at maximum swing; a mirror test predicts a sign and a
symmetry in advance and finds them on two corpora (\S\ref{sec:phase}). Marking the convention in the
prompt then \emph{collapses the arrangement effect} from $5.16$ floors to $0.03$ and reaches $87.5\%$
of the union ceiling. \emph{The point is not that a marker helps but that it removes the path's
influence}, which no schedule could (\S\ref{sec:key}).
\item \textbf{A record of what did not work.} Three pre-registered mechanisms died (batch purity;
momentum-window cancellation; the palindromic schedule splitting theory recommends), one positive
control inverted, one control came back void, and all are reported with the thresholds that decided
them and the commits that froze them (Table~\ref{tab:ledger}, Appendix~\ref{app:registry}).
\end{enumerate}

\paragraph{What this paper does not claim.} Five statements a reader could reasonably extract from
the above are not supported by what is measured here, and separating them from the ones that are is
worth more than another result.
\begin{itemize}[leftmargin=1.5em,itemsep=2pt,topsep=2pt]
\item \emph{Not ``order does not matter under decay.''} The cosine interior spans $2.20$ floors
against a minimum detectable difference this design never reaches at any evaluation budget, so the
flat interior is a statement about the instrument. What carries the null is the pooled intraclass
correlation ($-0.076$, $[-0.409,+0.166]$) and the two-state count, both weaker than ``no effect.''
\item \emph{Not a general law of LLM training.} The ladder, mirror, ratio and key arms are one
budget on Qwen2.5-7B with two conflict constructions; Qwen3-8B-Base replicates the switch and the
key and does \emph{not} replicate the ladder's registered criterion or the mirror
(\S\ref{sec:average}, \S\ref{sec:phase}). Figure~\ref{fig:travel} is the travel record, misses
included.
\item \emph{Not independence of the two coordinates.} We claim \emph{conservation of $C$}, which is
measured. $\mathrm{Cov}(p,s)\neq0$ is also measured and fails at $5.17$ floors, so $S$ is an
allocation over a population whose difficulty and convention preference are correlated
(Definition~\ref{def:decomp}).
\item \emph{Not that conflict is harmless.} An exact-match benchmark reporting one convention sees
a real loss, and choosing what a model \emph{commits} to is precisely what production
post-training is for. What is conserved is the union, and only a scorer reading both conventions
recovers it. What the measurement does contradict is the field's instinct that \emph{conflict
destroys and arrangement decides how much}, an instinct two of our own pre-registered mechanisms
shared: nothing is destroyed, and below corpus scale nothing is even moved.
\item \emph{Not a quantitative theory.} Proposition~\ref{prop:sched} is used for signs, orderings
and which factor a knob enters. The direction is derived and every magnitude is measured.
\item \emph{Not a statement about every coordinate.} Everything is measured along a
one-parameter family of paths and on one behavioural coordinate, so another coordinate could carry
order information we did not look at, and the checkpoints that would settle it were reclaimed after
evaluation. Two controls searched \emph{off} that family and found nothing outside the same
interval (batch purity moves $3\%$ of the span; the mirror relocates within it), and adding a
coordinate can only \emph{raise} the number of occupied states, so two is a floor.
\end{itemize}

\S\ref{sec:conserve} shows which coordinate separates the two states, allocation and not
capability, so what the path selects is a policy and not a competence. \S\ref{sec:method} is
where the path, the channel and the two coordinates stop being prose: it defines them, states the
three propositions the measurements are read against, and shows that two of the mechanisms we
pre-registered were dead before either ran. Figure~\ref{fig:arch} is that itinerary as a map:
one instrument, one wall, and the three terms measured to get past it, each with the verdict that
decided it, so the three can be seen as three of a kind rather than met twenty pages apart.
Figure~\ref{fig:roadmap} is the object itself, drawn twice, and it is where the schedule result
can be seen rather than read: the same paths, the same budget and the same seeds produce two
different pictures when one field of the trainer configuration changes.

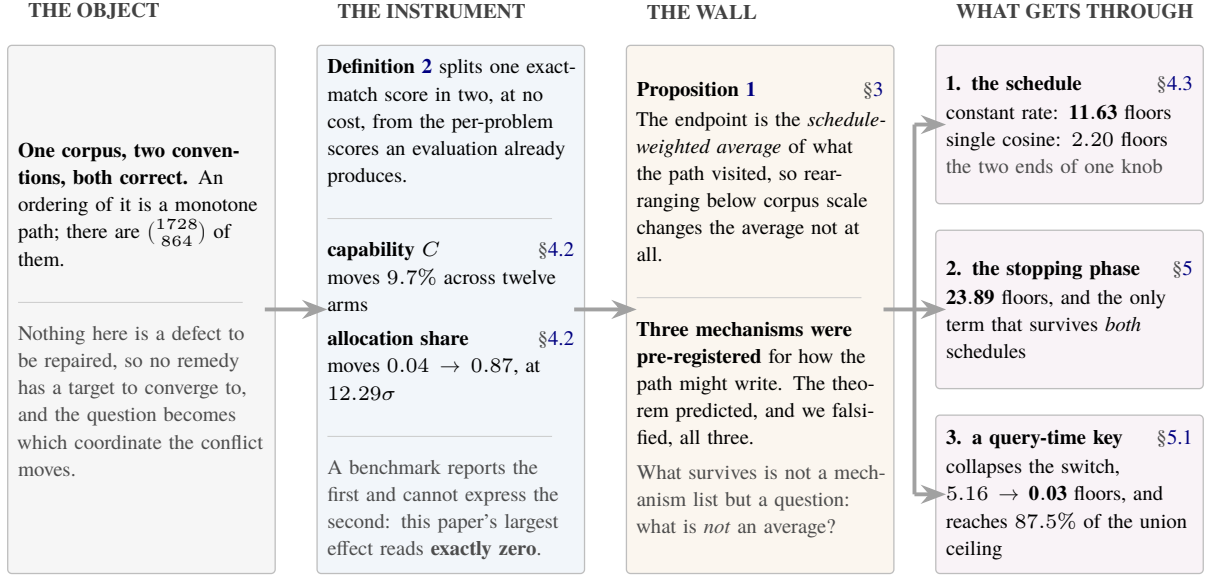
\begin{figure}[t]
\centering
\resizebox{\linewidth}{!}{\begin{tikzpicture}[font=\small,>=Stealth,
 tk/.style={font=\scriptsize,cGrey},
 lb/.style={font=\scriptsize,cGrey!90!black},
 stage/.style={font=\scriptsize\bfseries,cGrey!60!black},
 box/.style={draw=cGrey!45,line width=0.5pt,rounded corners=1.6pt,inner sep=4pt,
             align=left,font=\scriptsize,text width=3.35cm,minimum height=#1,anchor=center},
 sec/.style={font=\scriptsize,cGrey!70!black},
 flow/.style={cGrey!65,line width=1.5pt,->}]

\def\hr{\textcolor{cGrey!35}{\rule{3.05cm}{0.4pt}}}

\node[stage,anchor=south west] at (0.15,3.81) {THE OBJECT};
\node[box=7.17cm,fill=cGrey!6] (obj) at (1.825,0)
 {\textbf{One corpus, two conventions, both correct.} An ordering of it is a monotone
  path; there are $\binom{1728}{864}$ of them.\\[4pt]
  \hr\\[4pt]
  \textcolor{cGrey!70!black}{Nothing here is a defect to be repaired, so no remedy has a target
  to converge to, and the question becomes which coordinate the conflict moves.}};

\node[stage,anchor=south west] at (4.35,3.81) {THE INSTRUMENT};
\node[box=7.17cm,fill=cV!6] (inst) at (6.025,0)
 {\textbf{Definition~\ref{def:decomp}} splits one exact-match score in two, at no cost, from the
  per-problem scores an evaluation already produces.\\[4pt]
  \hr\\[4pt]
  \textbf{capability $C$} \hfill \textcolor{cGrey!70!black}{\S\ref{sec:conserve}}\\
  moves $9.7\%$ across twelve arms\\[3pt]
  \textbf{allocation share} \hfill \textcolor{cGrey!70!black}{\S\ref{sec:conserve}}\\
  moves $0.04\to0.87$, at $12.29\sigma$\\[4pt]
  \hr\\[4pt]
  \textcolor{cGrey!70!black}{A benchmark reports the first and cannot express the second: this
  paper's largest effect reads \textbf{exactly zero}.}};

\node[stage,anchor=south west] at (8.55,3.81) {THE WALL};
\node[box=7.17cm,fill=cA!7] (wall) at (10.225,0)
 {\textbf{Proposition~\ref{prop:avg}} \hfill \textcolor{cGrey!70!black}{\S\ref{sec:method}}\\[1pt]
  The endpoint is the \emph{schedule-weighted average} of what the path visited, so rearranging
  below corpus scale changes the average not at all.\\[4pt]
  \hr\\[4pt]
  \textbf{Three mechanisms were pre-registered} for how the path might write. The theorem
  predicted, and we falsified, all three.\\[3pt]
  \textcolor{cGrey!70!black}{What survives is not a mechanism list but a question: what is
  \emph{not} an average?}};

\node[stage,anchor=south west] at (12.75,3.81) {WHAT GETS THROUGH};
\node[box=2.15cm,fill=cD!6] (sched) at (14.425,2.51)
 {\textbf{1. the schedule} \hfill \textcolor{cGrey!70!black}{\S\ref{sec:average}}\\[1pt]
  constant rate: $\mathbf{11.63}$ floors\\
  single cosine: $2.20$ floors\\
  \textcolor{cGrey!70!black}{the two ends of one knob}};
\node[box=2.15cm,fill=cD!6] (stop) at (14.425,0)
 {\textbf{2. the stopping phase} \hfill \textcolor{cGrey!70!black}{\S\ref{sec:phase}}\\[1pt]
  $\mathbf{23.89}$ floors, and the only term that survives \emph{both} schedules};
\node[box=2.15cm,fill=cD!6] (key) at (14.425,-2.51)
 {\textbf{3. a query-time key} \hfill \textcolor{cGrey!70!black}{\S\ref{sec:key}}\\[1pt]
  collapses the switch, $5.16\to\mathbf{0.03}$ floors, and reaches $87.5\%$ of the union
  ceiling};

\draw[flow] (3.50,0) -- (4.35,0);
\draw[flow] (7.70,0) -- (8.55,0);
\draw[cGrey!65,line width=1.5pt] (11.90,0) -- (12.32,0);
\draw[cGrey!65,line width=1.5pt] (12.32,2.51) -- (12.32,-2.51);
\foreach \y in {2.51,0,-2.51} {\draw[flow] (12.32,\y) -- (12.75,\y);}
\end{tikzpicture}}
\caption{\textbf{The paper as a map: one instrument, one wall, and the three terms that get
past it.} Read left to right. \emph{Every box carries a measured verdict rather than a
description}, which is the property that decides whether a map of this kind is worth printing:
change the experiments and this picture changes, because two of its boxes would say the opposite
of what they say and the wall's three exits would be a different three. The instrument
(Definition~\ref{def:decomp}) is what makes the rest expressible, since the quantity every term
here moves is invisible to the score a benchmark reports. The three exits are not alternatives to
one another: the schedule decides whether an ordering survives to the endpoint at all, the
stopping phase is the one term that survives either schedule, and the key is a second channel
rather than more bandwidth on this one. Magnitudes are in contrast floors and are the body's,
not recomputed here.}
\label{fig:arch}
\end{figure}

\begin{figure}[t]
\centering
\resizebox{\linewidth}{!}{\input{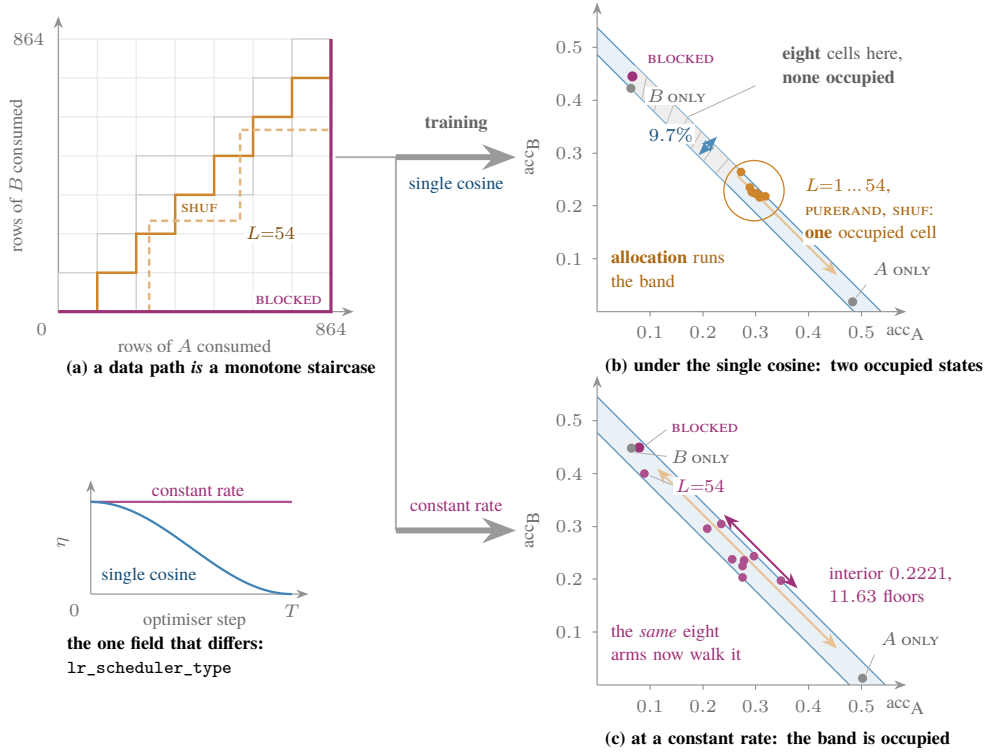}}
\caption{\textbf{The space of data paths, and its image under each of the two schedules.}
\textbf{(a)} A data path \emph{is} a monotone lattice path: at each optimiser step the trainer
consumes a row from $A$ or from $B$. There are $\binom{1728}{864}$ of them: the size of the space
sampled, not a quantity anything here transmits. \shuf\ hugs the diagonal, a ladder rung is
a coarser staircase, \blocked\ is the corner. \textbf{(b)} and \textbf{(c)} Its image, in the
read-out's own coordinates (Definition~\ref{def:decomp}), every arm of the synthetic conflict,
on the three seeds the two families share. \emph{The two panels differ in one field of the
trainer configuration and in nothing else}: same corpus, same paths, same budget, same seeds,
same step. In both, the arms lie in a narrow capability band and spread along it, so capability
is the band's width and allocation is its length. Under the single cosine every published arm
uses, the nine ladder arms fall in \emph{one} cell of the $2\hat\sigma$ resolution of
Definition~\ref{def:channel} and \blocked\ in another: two states, with eight empty cells between
them where the range would support about ten, the eight interior arms spanning $0.042$ of share,
$2.2$ contrast floors, below anything this design can resolve. Under a constant rate those same
eight span $0.2221$, $11.63$ floors, and $L{=}54$ leaves them for the corner.
\emph{The difference between the panels is Proposition~\ref{prop:avg}}: the endpoint is the
schedule-weighted average of what the path visited, so a schedule that decays to zero has
averaged the ordering away by the time anyone scores it. Axes are square and identical across the
two panels, so both bands are true $45^\circ$ strips and the pictures can be compared by eye.
Source: \texttt{constlr\_ladder\_table}.}
\label{fig:roadmap}
\end{figure}

\begin{figure}[t]
\centering
%
%
%
\input{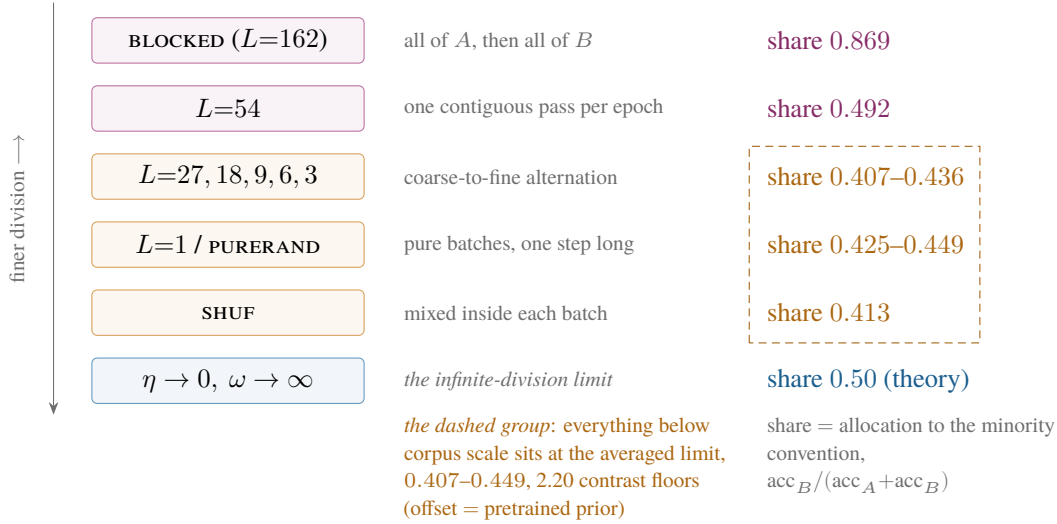}
\begin{tikzpicture}[font=\small,>=Stealth]
\def\x{0.0}
\def\fxa{6.6}
\def\fxb{11.4}
\def\fy{0.85}
\def\fw{4.25cm}
\def\nohyph{\hyphenpenalty=10000\exhyphenpenalty=10000\relax}
\draw[cGrey,thin,->] (\x+2.1,6.20) -- (\x+2.1,0.75);
\node[cGrey,font=\scriptsize,rotate=90,anchor=south] at (\x+1.85,3.475) {finer division $\longrightarrow$};
\foreach \y/\name/\desc/\share/\col in {%
 5.7/{\blocked\ ($L{=}162$)}/{all of $A$, then all of $B$}/{$\ldrblocked$}/cD,
 4.8/{$L{=}54$}/{one contiguous pass per epoch}/{$\ldrlfour$}/cD,
 3.9/{$L{=}27,18,9,6,3$}/{coarse-to-fine alternation}/{$\ldrcoarselo$--$\ldrcoarsehi$}/cA,
 3.0/{$L{=}1$ / \purerand}/{pure batches, one step long}/{$\ldrfinelo$--$\ldrfinehi$}/cA,
 2.1/{\shuf}/{mixed inside each batch}/{$\ldrshuf$}/cA,
 1.2/{$\eta\to0,\ \omega\to\infty$}/{\emph{the infinite-division limit}}/{$0.50$ (theory)}/cV}
 {\node[draw=\col!70,fill=\col!6,rounded corners=2pt,inner sep=4pt,minimum width=3.6cm,
  minimum height=0.60cm,anchor=west] (n) at (\x+2.6,\y) {\name};
 \node[cGrey,font=\scriptsize,anchor=west] at (\x+6.6,\y) {\desc};
 \node[\col!85!black,font=\small,anchor=west] at (\x+11.4,\y) {share \share};}
\node[cGrey,font=\scriptsize,anchor=north west,text width=\fw,align=flush left] at (\x+\fxb,\fy)
 {\nohyph share $=$ allocation to the minority convention,\\
  $\mathrm{acc}_B/(\mathrm{acc}_A{+}\mathrm{acc}_B)$};
\draw[cA!80!black,densely dashed] (\x+11.3,1.70) rectangle (\x+14.35,4.30);
\node[cA!85!black,font=\scriptsize,anchor=north west,text width=\fw,align=flush left] at (\x+\fxa,\fy)
 {\nohyph \emph{the dashed group}: everything below corpus scale sits at the averaged limit,
  $\ldrintlo$--$\ldrinthi$, \ldrintfloors\ contrast floors (offset $=$ pretrained prior)};
\end{tikzpicture}%
\caption{\textbf{The division ladder under the single cosine, with the measured allocation at
every rung} (synthetic conflict, Qwen2.5-7B, three seeds; \S\ref{sec:average}). From one
optimiser step to a quarter of the epoch, arrangement does not move the allocation: the whole
interior of the ladder sits at the averaging limit. \emph{That flatness is the schedule's and
not the path's}: the identical ladder at a constant rate spans $11.63$ contrast floors in the
same interior. Under this schedule the only escapes are at the top, stopping the alternation
mid-swing (\S\ref{sec:phase}), and off the ladder entirely, via a query-time key
(\S\ref{sec:key}). Source: \texttt{constlr\_ladder\_table}.}
\label{fig:ladder}
\end{figure}

\paragraph{The instrument, which outlives the result.} Score a
model on a benchmark whose answers admit more than one correct convention and exact match reports
a product of two things: what the model can do, and which form it decided to write in.
Definition~\ref{def:decomp} splits them at essentially no cost, using only the per-problem scores
an evaluation already produces, and the split is load-bearing in a way that is easy to state and
hard to unsee.

\emph{This paper's largest effect is exactly zero on the coordinate that measures capability.}
The arrangement switch we measure at $12.29\sigma$, against its own control inert at
$0.15\sigma$, moves the allocation share from $0.04$ to $0.87$ and moves
$\mathrm{acc}_A+\mathrm{acc}_B$ not at all beyond $9.7\%$ arm to arm. A twelve-sigma result and
a null are \emph{the same measurement} read on two axes. Any benchmark carrying contested
conventions is silently reporting the first number and calling it the second, and any
intervention evaluated that way (an ordering, a schedule, a data mixture, a decoding change)
can post a large effect while changing nothing a user would call capability. The recipe-level
member reports the same effect on a different coordinate \citep{frl2026}; the number here is
ours, with its own same-convention control (Table~\ref{tab:builds}).

We therefore state the decomposition as a contribution in its own right rather than as
apparatus for the ladder, and we state its price with it: $C$ is comparable only within a
corpus, and $S$ equals the convention policy only when the covariance condition of
Definition~\ref{def:decomp} holds, which we test and which fails at $5.17$ floors. Both limits
are smaller than the effects the coordinates separate, and neither is a reason to keep reporting
the product.

\section{Related Work}\label{sec:related}

We group the literature by which side of the averaging wall its object sits on, a distinction
that cuts across the usual subfield boundaries. One assumption runs through most of it, and this
paper's result is what happens when it is dropped.

\paragraph{The premise almost everyone shares: a conflict is an error.} The knowledge-conflict
literature names our object exactly. \citet{xu2024knowconflict} call it \emph{intra-memory}
conflict, discrepancy inside the parameters traced to inconsistency in the training data, and
every remedy it surveys is a repair: refine the parametric knowledge, regulate the behaviour,
reweight or filter the offending rows \citep{incompletelearning2026}. A repair presupposes a
target, and a target presupposes that one of the conflicting forms is \emph{wrong}. Our natural
corpus is built so that neither is (Table~\ref{tab:corpora}; the synthetic one stipulates a
convention wrong against mathematical ground truth, and \S\ref{sec:threats} prices the
difference), and under that construction Proposition~\ref{prop:avg} makes the mixture the
\emph{loss-minimising} policy rather than a malfunction, so there is nothing for a repair to
converge to and the question becomes which coordinate the conflict moves. The nearest work to
drop the premise is \citet{krestnikov2026truth}, which trains small transformers on mathematics
corpora carrying both correct and incorrect solutions and finds that a \emph{coherent} alternative
rule system destroys the preference for the true answer entirely, while adding a \emph{second}
competing rule restores most of it. That is our construction reached from the truth-tracking side,
and it predicts what we measure: two coherent conventions do not degrade capability, they split
allocation. Their corpora make one form wrong and ours make neither, so their restored accuracy
and our conserved $C$ are different quantities that happen to move together, and we know of no
measurement in that line of the allocation coordinate or of a query-time key. This paper is a
limit of that literature rather than a contribution to it: send ``one of them is wrong'' to zero
and the remedies lose their referent while the phenomenon does not.

The same premise, inverted, organises the evaluation side. \citet{plank2022variation} argues that
human label variation is signal rather than noise and that a single gold label is inadequate where
annotators legitimately differ, and the perspectivist programme that follows fixes evaluation by
matching the \emph{distribution} of human labels. That repair also needs a target unavailable
here: with two correct conventions every allocation is equally correct. An undecomposed accuracy
does not merely undercount capability in the familiar way that exact match penalises \emph{three}
against \emph{3} \citep{schaeffer2023mirage,sftrlquagmire2025}; it cannot express the coordinate
along which our arms move, which is what Definition~\ref{def:decomp} is for.
\citet{schaeffer2023mirage} is the closest precedent and the closest warning: emergence turned out
to be a property of the metric rather than of the model, and \S\ref{sec:conserve} performs the
same move on a different axis when it reports its own $12.29\sigma$ arrangement switch as exactly
zero under a convention-agnostic score.

Two lines reach our allocation coordinate from the metric side. \citet{holtzman2021surface} names
the mechanism: several surface forms of one correct answer compete for probability mass, so a
scorer reading only the highest-probability string reports the competition rather than the
knowledge. That competition \emph{is} our $S$, and Definition~\ref{def:decomp} adds only that it
can be divided out of $C$. \citet{yeom2026commitment} measure the same split at inference time,
finding $16$--$47\%$ of instruct-model hallucinations occur with substantial mass already on the
correct concept, the distinguishing factor being whether that mass concentrates on one surface
form or disperses across alternatives; their sharpening rises with scale and with instruction
tuning, which is a training-path property, and we read it as the query-time image of what
\S\ref{sec:conserve} installs. \citet{janeiro2026paraeval} price what it costs an evaluation: on a
$1$--$8$B testbed, models trained on identical knowledge post false gaps above two points from
answer phrasing alone, narrowing to under one point when several paraphrases per option are
queried, and the artefact persists at $70$--$120$B. Their remedy and ours point in opposite
directions on purpose. ParaEval \emph{averages the surface-form term away}, which is right when
the phrasing is nuisance; we keep it as a coordinate, because in a conflicted corpus the phrasing
is exactly what arrangement moves, and our convention-agnostic score is ParaEval's move applied to
our own headline, duly returning zero (\S\ref{sec:conserve}). A surface-form term is nuisance when
the training data agree on the convention and signal when they do not.

\paragraph{Where Proposition~\ref{prop:avg} comes from.} The proposition is not new as
optimisation. In the deterministic cyclic case it is the central dichotomy of the
incremental-gradient literature \citep{bertsekas2011incremental,nedic2001incremental}: with a
step size decaying to zero the iterates converge to a minimiser of the \emph{summed} objective and
the order does not survive, while at a constant step they enter a limit cycle whose
\emph{position} depends on the order. Proposition~\ref{prop:sched} interpolates between those
regimes and Proposition~\ref{prop:phase} is that limit cycle. The without-replacement line prices
the orderings against each other
\citep{gurbuzbalaban2021why,haochen2019random,mishchenko2020random,lu2022grab}, and the
constant-step-size bias our constant-rate family reads is under current study in its own right
\citep{emmanouilidis2026shuffling}. What we add is not the theorem but its transport: that
literature states its results for a fixed objective, and the question here is what the same
dichotomy does to a \emph{policy} when the summed objective's minimiser is a mixture rather than a
point. The transport supplies an allocation coordinate that moves while the loss does not, and the
observation that the field's default schedule puts nearly all of published fine-tuning practice at
one end of the dichotomy without saying so.

\paragraph{Inside the wall: batch composition, shuffling, and per-step gradient surgery.} The
contradiction that seeded our own v1 (purify the batch \citep{commonit2024}, mix it
\citep{dataorg2026}, mix it with a theorem \citep{lu2022grab}) is resolved sideways rather than
adjudicated: at conflict, purity carries $3\%$ of the span, and below corpus scale nothing
batch-sized moves the allocation at all (\S\ref{sec:average}). The axis those three papers dispute
lies strictly inside the wall, which is why it can persist without either side being wrong about
its own measurements. \citet{sweeney2026shuffle} shows that optimiser state makes shuffle order a
first-order \emph{noise} source; our block sweep shows the momentum window moves the conflict
\emph{signal} not at all, and the two are consistent under Proposition~\ref{prop:avg}, buffers
adding variance about a mean-field point the time-average sets. \citet{sweeney2026liebracket}
proposes the sharpest positive claim we could find, that the Lie bracket of two tasks' update
operators predicts which order transfers better, and Appendix~\ref{app:commutator} measures a
one-shot commutator score built in its spirit and finds it inverted on our pairs, at Spearman
$-0.543$ and $-0.600$. That is a range boundary and not a refutation, because the two experiments
do not meet: their tournament scores Hessian-vector products against a shared $\theta_0$ reference
and reports $98.1\%/98.9\%$ pairwise accuracy at block length $k{=}1$ falling to $73.1\%/72.2\%$
at $k{=}20$, whereas our budgets are $162$--$324$ steps per block.
Proposition~\ref{prop:sched} says why a score computed once at $\theta_0$ must decay with
$\tau_L$: it is the leading term of the interior integral in Eq.~\eqref{eq:hatch}, whose neglected
remainder grows with block length. Per-step gradient-conflict methods
\citep{yu2020pcgrad,liu2021cagrad,navon2022nashmtl} and ordered-shuffle schemes
\citep{lu2022grab} likewise operate inside the wall: they change optimisation and can change
variance, but Remark~\ref{prop:noescape} says they cannot change the installed allocation, because
that is fixed by what a held-out query may condition on.

\paragraph{The wall from the other side: data mixing at corpus scale.} Methods that reweight
\emph{proportions} (DoReMi \citep{doremi}, DoGE \citep{doge}, and phase-scheduled mixtures
\citep{gu2025datamixphase,midtraining2025survey}) act on exactly the quantity
Proposition~\ref{prop:avg} leaves free, the source weights $n_A/(n_A+n_B)$ in the mean
field. Our ratio experiment is the controlled version of their premise: moving $1{:}1$ to
$2{:}1$ moves the allocation share $0.453\to0.312$ against mean-field predictions $0.500$ and
$0.333$. Corpus composition is the lever, path arrangement is not, and the boundary between them
is measurable.

\paragraph{Sequential training and forgetting.} Catastrophic forgetting
\citep{mccloskey1989catastrophic,french1999catastrophic} and its mitigations
\citep{kirkpatrick2017ewc,lopezpaz2017gem,lee2021continual} concern capability lost when a
second task overwrites a first, and the continual-learning literature measures it as such
\citep{wu2024cllmsurvey,diazrodriguez2018metrics}. Our decomposition separates that from what
conflict does: under Definition~\ref{def:decomp} forgetting is a movement of $C$, whereas the
conflict switch is a movement of $S$ at $C$ held to within $9.7\%$. \citet{evron2022catastrophic}
treat blocked linear regression as alternating projections, and the stopping phase
(Proposition~\ref{prop:phase}) is the nonlinear-policy face of the same recency. That the order
survives at all has support one level down: \citet{krasheninnikov2025recency} fine-tune
sequentially on six datasets and find training-order recency \emph{linearly encoded} in the
activations, with a linear probe separating early- from late-learned entities at about $90\%$.
Their read-out is on the representation and ours on the policy, the same statement at different
depths; what Proposition~\ref{prop:avg} adds is the condition under which it survives to the
\emph{endpoint} of a decayed schedule, which is where a reported score is taken.
\citet{conklin2026forgetting} and \citet{forgettingmech2026} characterise forgetting
mechanistically, and neither predicts a term that switches on when two sources \emph{disagree}
while capability holds. \citet{incompletelearning2026} isolates internal SFT-data inconsistency
per sample, and our decomposition says what that inconsistency does: it moves the model from a
deterministic to a stochastic policy at fixed capability, whose per-sample signature is the
coin-flip fingerprint of \S\ref{sec:conserve}.

\paragraph{Curriculum, ordering, and their nulls.} The record on ordering is divided
\citep{bengio2009curriculum,rohrer2007interleaving,kornell2008interleaving,interleaved2025,
goodcurriculum2025,beyondrandom2025,curriculum2026dynamics,implicitcurriculum2026,
lrdecay2025curriculum}, and the division is what the averaging wall predicts: these are paths
below corpus scale, where Proposition~\ref{prop:avg} says the endpoint is the same and only
transients and stopping phases differ. \citet{lrdecay2025curriculum} finds curriculum
advantages over random shuffling that hold at a constant rate and diminish under standard decay,
in pretraining at $1.5$B parameters over $30$B tokens. We reproduce that moderator under control,
at $7$B in supervised fine-tuning, over ten arrangements of one conflicted corpus and two families
differing only in \texttt{lr\_scheduler\_type}: interior span $2.20$ floors under a cosine against
$11.63$ under a constant rate (Table~\ref{tab:constlr}). Their reading is that decay wastes the
curriculum; ours is that decay \emph{is} the averaging operator of Proposition~\ref{prop:avg}, so
``order matters'' and ``order does not matter'' are the two ends of one schedule knob rather than
two findings to be reconciled. The divided record should then sort by how much step size survives
to the end of training, so a paper reporting an ordering effect without its schedule has not
reported a moderator its own effect depends on; Eq.~\eqref{eq:hatchbound} makes that a consequence
rather than a caution, since the two factors multiply. One case can be checked without new runs:
\citet{curriculum2026dynamics} holds Pythia's configuration fixed across orderings, decaying a
cosine to a tenth of peak at every size, and reports ordering effects on stability largely gone by
$410$M, which is the sign Proposition~\ref{prop:avg} requires. The recipe paper \citep{frl2026}
prices the resolution at which any of these comparisons can be made at all, and
\citet{pairwisefragile2026} reports pairwise order predictions degrading with block length, which
Proposition~\ref{prop:avg} explains: a score computed once at initialisation is the leading term
of a series whose error accumulates with the step budget.

\paragraph{The sharpest counter-result, and the regime boundary it draws with ours.}
\citet{ledoux2026order} reports the opposite of everything above, and reports it cleanly. Training
small networks on modular arithmetic \emph{from scratch}, two fixed orderings reach $99.5\%$ test
accuracy from a training set covering $0.3\%$ of the input space, where random ordering does not,
and the learned Fourier representation's fundamental frequency is the mathematical dual of the
ordering's own structure. Order there is the mechanism, wide enough that the paper names it a
covert information channel able to bypass content-level auditing.

We measure ten orderings of one corpus into two distinguishable endpoint states, on a range its
own resolution would divide into about ten. Both results are right, and what the path can write
depends on how much the prior has already fixed. From scratch the parameters carry no structure
the order must compete with, so a sufficiently regular order supplies a great deal. On a
pretrained $7$B prior the same channel writes into parameters that already encode the answer
format, the arithmetic and the convention preference, and it moves the one coordinate pretraining
left underdetermined: which of two admissible conventions to commit to. That is one bit, and
\S\ref{sec:phase} shows it is the stopping phase. \emph{Order is a wide channel into an empty
state and a narrow one into a full state}, a resource trade neither paper states alone, and it
predicts that the channel narrows monotonically with pretraining scale. The two safety readings
converge from opposite regimes: \citet{ledoux2026order} names a covert channel that content
auditing misses, and \S\ref{sec:threats} reports the pretrained-model version, where an attacker
controlling only data-loader order and touching no byte of the corpus moves the allocation $0.22$
in exact match. Narrow is not zero, and the bit that survives is the one a user would call the
model's commitment.

\paragraph{Order effects at LLM fine-tuning scale.} \citet{ju2024imbalance} is the closest
published claim to a positive result in our own setting: data order produces training imbalance in
LLM SFT and degrades performance, with the proposed fix being to merge models fine-tuned under
different orderings. Read against Proposition~\ref{prop:avg} the remedy is the informative half.
Averaging endpoints across orders explicitly constructs the quantity the infinitely divided limit
installs implicitly, so a method that works by merging over orderings is evidence that the
orderings differ mainly by a term the mean removes, which is what a flat interior plus a stopping
phase predicts. Their measurements are on heterogeneous instruction data where our coherence
premise does not hold, so we report this as consistent rather than as replication.

\paragraph{Theory of the training path.} The lazy and kernel regimes
\citep{jacot2018ntk,chizat2019lazy} give the setting in which the averaging argument is exact,
and full-parameter SFT is measurably not fully lazy, which is why Proposition~\ref{prop:avg} is
reported as governing signs and orderings rather than magnitudes. Averaging itself is classical
\citep{bogoliubov1961asymptotic,kapitza1951pendulum}, and its modern empirical form is
\citet{ajroldi2025averaging}, who benchmark weight averaging across seven workloads, ask whether
it can replace learning-rate decay, and conclude that the two are best combined rather than
exchanged. We use the same pairing as an instrument rather than as a method: if decay and
averaging were interchangeable the schedule could not be the operator that decides whether an
ordering survives, and their finding that the two compose is what leaves
\texttt{lr\_scheduler\_type} free to be varied on its own, which is the one contrast
\S\ref{sec:average} runs. Phenomena that live in the transient rather than the endpoint, grokking
\citep{power2022grokking} and double descent
\citep{belkin2019doubledescent,nakkiran2020doubledescent}, are outside our budget regime but share
the moral that an endpoint measurement can be a statement about where the path stopped.
Attribution methods \citep{pruthi2020tracin,influencedynamics2025} track which examples moved the
parameters; the write-channel rate of Definition~\ref{def:channel} is the complementary statistic,
asking how much of the path's \emph{order} survives at all. Reading training as a channel with a
rate has an ancestor in the information bottleneck \citep{tishby2015ib}, and the difference is not
one of degree: the bottleneck compresses the \emph{input} while preserving the \emph{label},
whereas the quantity compressed here is the \emph{order} of a fixed multiset and the receiver is
the endpoint's behaviour.

\paragraph{Family.} The framework, the thought experiment and the walls are
\citet{chen2026cch}; this paper is its training-time member, and the write-time separability
their construction assumes is here a measured quantity rather than an assumption. \Thecompanion\
recipe-level paper \citep{frl2026} measures the same corpora on a different coordinate and prices what
recipe search can buy.

\section{Method: the Path, the Channel, and the Two Coordinates}\label{sec:method}

\S\ref{sec:limit} named three objects in prose: a path, a channel, and two coordinates on the
endpoint. This section defines them, so that what the rest of the paper measures are statements
rather than descriptions. It opens with the question a reader should settle before any of it,
which is what \emph{kind} of wall the paper is about, and closes with the three propositions
that make two of our own pre-registered mechanisms dead on arrival.

\begin{definition}[Data path and the division ladder]\label{def:path}
Fix a multiset $\mathcal{D}=\mathcal{D}_A\uplus\mathcal{D}_B$ of $n_A+n_B$ examples from two
sources. A \emph{data path} is an ordering $\pi$ of $\mathcal{D}$; training is $T$ optimiser
steps of step size $\eta$ along $\pi$. The \emph{division ladder} is the one-parameter family
$\pi_L$ that alternates same-source blocks of exactly $L$ consecutive \emph{optimiser steps},
with $L=1$ the finest alternation and $\shuf$ the within-batch mixture.

\emph{$L$ is counted in steps, not rows, and the two are reported in different places, so the
conversion is fixed here once.} The optimiser consumes $16$ rows per step, and each source
contributes $864$, so one source's per-epoch allocation is $54$ steps and an epoch is $108$. A
rung is one arranged one-epoch file passed over three times, $324$ steps, which is where its
result files are read. \blocked\ is \emph{not} a rung: it is two stages, three epochs on $A$ then
three on $B$ resumed from the first, so each source is one contiguous run of $162$ steps, the
second stage's counter starts from zero and its files are read at $162$, and its effective block
is $L=162$ against the ladder's largest rung at $54$. The budget is the same $324$ steps either
way, and that largest rung is $A^{54}B^{54}$ three times over, not a blocked order. Every arm
differs from every other \emph{only} in $\pi$, at fixed $\mathcal{D}$, $T$, $\eta$ and seed, but
for \blocked's two cosines where a rung has one, which \S\ref{sec:threats} controls.
\end{definition}

\begin{definition}[Capability and allocation]\label{def:decomp}
Let the two evaluation sets be the same problems under mutually exclusive conventions, so a
generation matches at most one gold. Write
\[
\underbrace{C=\mathrm{acc}_A+\mathrm{acc}_B}_{\text{capability}}
\qquad\text{and}\qquad
\underbrace{S=\mathrm{acc}_B/C}_{\text{allocation}},
\]
so that $(\mathrm{acc}_A,\mathrm{acc}_B)=(C(1-S),\,CS)$ is a change of coordinates, not a
model. Write $p(x)$ for the probability the model solves $x$ and $s(x)$ for the probability it then
answers under $B$. Mutual exclusivity gives $C=\mathbb{E}[p]$ for any $s$ whatever, which is why
capability is the robust half of this pair. The allocation is
$S=\mathbb{E}[p\,s]/\mathbb{E}[p]$, and it equals the convention policy $\mathbb{E}[s]$
\emph{exactly when $\mathrm{Cov}(p,s)=0$ across problems}. The condition is a covariance rather
than an independence: it is weaker than per-problem independence, and it is the one the
estimator can test. We measure it below, and it does not vanish.
\end{definition}

\noindent Definition~\ref{def:decomp} is the paper's instrument and also its sharpest
limitation, and we state both here. It is nearly an identity (that is why nobody reports it) and what it buys is that \emph{arrangement claims become claims about which coordinate
moves}, which an undecomposed accuracy cannot express. What it costs is that $C$ means
different things in different corpora; \S\ref{sec:threats} returns to this.

\paragraph{The independence clause is testable, and it fails.} The sentence above is a
conditional, and what matters is what happens when its antecedent is false: if hard problems
fall back to the pretraining convention then $s$ depends on $p$, the coordinates are not
orthogonal, and part of what we report as conserved capability is a property of the
construction. The per-problem arrays settle it without new training. Each problem is drawn
$k{=}4$ times, so for a problem carrying any credit we observe the fraction solved,
$t=\mathrm{a}+\mathrm{b}$, and the convention preference, $s=\mathrm{b}/t$; under the
independence clause $\mathbb{E}[s\mid t]$ is constant in $t$. It is not. Pooling the three
seeds of the interleaved arm, problems the model solves on all four draws are answered under
$B$ with probability $0.4478$, and problems it solves on some but not all draws with
probability $0.3491$: a difference of $0.0988$, $5.17$ contrast floors, Mann--Whitney
$z=3.56$, Spearman $\rho_s=+0.21$ between solve rate and share. The same sign and a comparable
size appear on \purerand\ ($0.1081$, $z=3.67$) and on $L{=}27$ ($0.0657$, $z=2.57$).
\emph{Problems the model half-solves drift toward the pretrained convention.}

\emph{Which independence this refutes, because the two are easy to conflate.} The statistic
above is computed \emph{across}
problems; the clause it refutes is therefore the across-problem one, $\mathrm{Cov}(p,s)=0$, and
that is exactly the clause the decomposition needs, since
$S=\mathbb{E}[p\,s]/\mathbb{E}[p]$ equals $\mathbb{E}[s]$ if and only if that covariance
vanishes. What it does \emph{not} refute is per-problem conditional independence: a model whose
convention choice is independent of solving on every single problem will still show this
correlation whenever the problems differ from one another in both quantities, which they do. So
Definition~\ref{def:decomp}'s antecedent should be read as the covariance condition and not as a
statement about any individual problem, and Definition~\ref{def:decomp} states it that way.

Two consequences, and we separate them because they are not equally severe. Capability
conservation is unaffected: $C=\mathrm{acc}_A+\mathrm{acc}_B$ is measured, not derived from the
independence clause, so every conservation number in this paper stands as reported. What does
not stand is reading $S$ as a convention policy that a capability change cannot touch. $S$ is
an allocation \emph{averaged over a problem population whose difficulty and convention
preference are correlated}, so a manipulation that changes which problems are solved will move
$S$ a little even with the policy fixed. The effect is bounded by what we measured: across the
solved/half-solved split the share moves $0.099$, against the $0.4621$ the path family spans
and the $0.4201$ that separates \blocked\ from the top of the interior. It is a real coupling, it is an
order of magnitude below the effects the paper reports, and it is the reason
\S\ref{sec:conserve}'s claim is stated as conservation of $C$ rather than as independence of
the two coordinates. Source: \texttt{review\_statistics}.

\begin{proposition}[The infinite-division limit]\label{prop:avg}
Along $\pi_L$, let $\theta$ evolve by gradient descent on the per-example losses. As
$\eta\to0$ with $L\eta\to0$ and $T\eta$ fixed, the trajectory converges uniformly on compacts
to the solution of the mean-field flow
\[
\dot\theta=-\nabla\Bigl[\tfrac{n_A}{n_A+n_B}L_A+\tfrac{n_B}{n_A+n_B}L_B\Bigr],
\]
which depends on $\pi$ only through the source proportions. If the losses are cross-entropy and
a contested input $x$ carries $p_A(\cdot\mid x)$ and $p_B(\cdot\mid x)$, the flow's minimiser at
$x$ is the mixture $q^\star=\tfrac{n_A}{n_A+n_B}p_A+\tfrac{n_B}{n_A+n_B}p_B$.
\end{proposition}

\begin{proof}[Proof sketch]
The first statement is Bogoliubov--Krylov averaging \citep{bogoliubov1961asymptotic} applied to
a piecewise-constant vector field whose period $2L\eta$ tends to zero: the trajectory tracks the
period-average of the field, which is the proportion-weighted gradient. The second is the
first-order condition for $\min_q \alpha\,\mathrm{CE}(p_A,q)+(1-\alpha)\,\mathrm{CE}(p_B,q)$
with $\alpha=n_A/(n_A+n_B)$, whose solution is $\alpha p_A+(1-\alpha)p_B$. That second statement is
about the \emph{unconstrained} minimiser over conditional distributions; identifying it with the flow's
stationary point in parameter space additionally requires the family to be rich enough to represent
$q^\star$ at the contested inputs, which is the assumption Remark~\ref{rem:wallkind} argues is
comfortably met at $7$B and which we state here rather than leave to the prose.
\end{proof}

\begin{corollary}[Two mechanisms that could not have worked]\label{cor:dead}
Under Proposition~\ref{prop:avg}, any intervention that leaves the period-average of the field
unchanged leaves the endpoint unchanged in the limit. Batch composition at fixed block structure
and any time-symmetric reordering of a block are two such interventions. Both were
pre-registered as mechanisms and both are dead (\S\ref{sec:average}); the averaging theorem
predicted them dead before either ran.
\end{corollary}

\begin{remark}[Arrangement cannot escape underdetermination]\label{prop:noescape}
Let $I$ be what the model may condition on when answering a held-out query. If the applicable
convention is not a function of $I$ ($H(\mathrm{conv}\mid I)>0$) then no path $\pi$
attains zero loss on the contested inputs, and the minimiser of expected loss is the mixture of
Proposition~\ref{prop:avg} for every $\pi$. In particular $S$ is constant along the ladder
except through the stopping term of \S\ref{sec:phase}.
\end{remark}

\noindent This is a remark rather than a proposition because its argument is one line and is
not deep: $\pi$ orders the training stream and does not enter the conditional
distribution of the convention given a held-out $I$, so the Bayes-optimal predictor is the same
for every $\pi$, and any $\pi$-dependence of the endpoint must come from failure to reach that
optimum. What it does is locate where an arrangement effect is allowed to live: \emph{only in
the escape hatch}, the distance from the optimum. That is not a
throwaway, because the escape hatch is where all three of this paper's positive findings turn
out to sit, but the content is in the escape hatch and not in the claim about the optimum.

\noindent \emph{What Remark~\ref{prop:noescape} does and does not require.} It requires the
\emph{optimum} to be $\pi$-independent, not a finite run's \emph{endpoint}, because a finite run
sits some distance from that optimum and that distance is where every positive finding in this
paper lives. It therefore does not make the block-length sweep's flatness \emph{required} rather
than observed, and Table~\ref{tab:constlr} settles it: the same ladder at a constant rate spans
$11.63$ floors. What the remark licenses is the narrower statement that any ladder effect must be
an escape-hatch effect, a failure to reach the mixture. How wide that hatch is, only
Proposition~\ref{prop:sched} says: the width is how much endpoint weight any moment of the run can
carry, a decaying schedule can never put a large step size and an uncontracted remainder at the
same moment, and a constant one does exactly that at the last step. A decaying schedule therefore
narrows the hatch and the ladder flattens; a constant one leaves it open and the ladder resolves.
That is also what makes \S\ref{sec:average}'s palindromic failure a confirmation rather than a
curiosity: a scheme designed against the discretisation cannot help, because the discretisation is
not what is costing anything.

\begin{proposition}[The schedule sets the reach, the arrangement sets the period]\label{prop:sched}
Let $\theta$ follow $\dot\theta=-\eta(t)\nabla L_{\pi(t)}(\theta)$ on $[0,T]$, and let $\bar\theta$
follow the same flow with $\nabla L_{\pi(t)}$ replaced by the proportion-weighted mean
$\bar g=\nabla[\alpha L_A+(1-\alpha)L_B]$, $\alpha=n_A/(n_A+n_B)$. Write $\Delta=\nabla L_B-\nabla L_A$
and let $s_\pi$ be the square wave equal to $1-\alpha$ on $A$-blocks and $-\alpha$ on $B$-blocks, so
that $\bar g-\nabla L_{\pi(t)}=s_\pi(t)\Delta$ and $s_\pi$ has zero mean over a period $\tau_L$. Define
the \emph{endpoint weight} of a deviation at time $t$,
\[
w(t)\;=\;\Phi(T,t)\,\eta(t)\,\Delta\bigl(\bar\theta(t)\bigr),
\]
with $\Phi(T,t)$ the state-transition operator of the mean flow linearised about $\bar\theta$. Then to
first order in the displacement, writing $S_\pi(t)=\int_0^t s_\pi$,
\begin{equation}\label{eq:hatch}
\theta(T)-\bar\theta(T)\;=\;\underbrace{S_\pi(T)\,\eta(T)\,\Delta\bigl(\bar\theta(T)\bigr)}_{\text{terminal}}
\;-\;\underbrace{\int_0^T S_\pi(t)\,\dot w(t)\,\mathrm{d}t}_{\text{interior}},
\end{equation}
and when the path runs to a whole number of periods the terminal term vanishes and
\begin{equation}\label{eq:hatchbound}
\bigl\|\theta(T)-\bar\theta(T)\bigr\|\;\le\;2\,\alpha(1-\alpha)\;\cdot\;\underbrace{\tau_L}_{\text{arrangement}}\;\cdot\;\underbrace{\textstyle\sup_t\|w(t)\|}_{\text{schedule}} .
\end{equation}
\end{proposition}

\begin{proof}[Proof sketch]
Subtract the two flows and linearise: $u=\theta-\bar\theta$ obeys
$\dot u=-\eta(\bar H-s_\pi\nabla\Delta)u+\eta s_\pi\Delta+O(\|u\|^2)$ with
$\bar H=\nabla^2\bar L(\bar\theta)$. The term in $\nabla\Delta$ is the instantaneous Hessian's
departure from the averaged one; it is not smaller than the term we keep, so it has to be disposed of
rather than dropped. At leading order $u$ oscillates as $S_\pi\,\eta\Delta$, and
$\langle s_\pi S_\pi\rangle=\tau_L^{-1}\!\int\tfrac{\mathrm{d}}{\mathrm{d}t}(S_\pi^2/2)=0$ over a whole
period because $S_\pi$ is periodic, so it first contributes at $O(\tau_L^2)$. Variation of constants on
what remains gives $u(T)=\int_0^T s_\pi w$; integrating by parts with $S_\pi(0)=0$ gives
Eq.~\eqref{eq:hatch}. For the bound, $S_\pi$ is the triangle wave of $s_\pi$, so
$\sup_t|S_\pi|=\alpha(1-\alpha)\tau_L$, and $w$ rises and falls at most once under a monotone or
single-peaked schedule, so $\int\|\dot w\|\le2\sup_t\|w\|$.
\end{proof}

\noindent Equation~\eqref{eq:hatchbound} separates the two factors, and \emph{the separation, not
either factor's value, is the content}: the arrangement enters only through the period $\tau_L$ and
the schedule only through how much endpoint weight any moment of the run can carry. Three things
follow, and we state the fourth thing that does \emph{not}.

\emph{The interior scales with block length.} $\tau_L\propto L$, so a ladder's interior is linear in
$L$ and vanishes as $\tau_L\to0$; Proposition~\ref{prop:avg} is that corner. A ladder should therefore
be \emph{monotone in $L$}, which a limit theorem alone does not predict and which the constant-rate
ladder measures (\S\ref{sec:average}).

\emph{The corner is not a small-$\tau_L$ object at all, which is why it does not move.} At \blocked\
the path is one period, $\tau_L=T$: the bound buys nothing, $S_\pi$ is a single triangle peaking at
$t=\alpha T$ rather than a fast oscillation, and the displacement is set by the whole profile of $w$
instead of by any part of it a schedule can delete. \blocked\ differs by $1.0$ floor between the two
schedule families against $17.0$ floors at $L{=}54$ (\S\ref{sec:average}).

\emph{A decaying schedule has a strictly shorter reach than a constant one at the same peak rate.}
Where $\eta$ is large a decaying schedule still has the rest of the run to contract through, and where
$\Phi\approx I$ its $\eta$ has already decayed, so it cannot have both at once; a constant schedule has
both at $t=T$, where $w(T)=\eta(T)\Delta$ carries no contraction discount at all. Hence
$\sup_t\|w\|$ is attained at the endpoint and equals $\eta(T)\|\Delta\|$ for a constant rate, and is
strictly below $\eta_{\max}\|\Delta\|$ for any schedule decayed to zero, by a margin that widens with
the contraction the run undergoes. \emph{This is the sense in which the schedule is the averaging
operator}, and it is what Table~\ref{tab:constlr} measures at fixed $\tau_L$: $2.20$ floors of interior
against $11.63$ when $\eta(T)$ is the only thing changed.

\emph{What the proposition does not do is predict the size of that gap.} The reach depends on $\Phi$
as well as on $\eta$, the two families do not share a $\Phi$, and a bound is not an estimate: a
different functional of the same profile, its total variation, orders the two schedules the other way
when the run contracts little. So the direction is derived and the magnitude is measured, and we say
which is which rather than let one borrow the other's authority. The proposition is a first-order
statement about a linearisation and we use it for signs, orderings and which factor a knob enters,
which is the standing it has in the averaging literature it comes from
\citep{bogoliubov1961asymptotic,hairer2006geometric}.

\emph{One prediction it makes that this paper has not tested.} The terminal term in
Eq.~\eqref{eq:hatch} is the stopping phase of Proposition~\ref{prop:phase}, and it is multiplied by
$\eta(T)$ with no contraction discount. A decaying schedule should therefore suppress the mirror
effect of \S\ref{sec:phase} in the same way it suppresses the ladder, and our mirror arms are all
cosine. That is one arm, it is listed with the others in \S\ref{sec:threats}, and until it is run the
terminal term is the one branch of Eq.~\eqref{eq:hatch} we have no schedule contrast for.

\begin{proposition}[The stopping phase]\label{prop:phase}
Treat $\pi_L$ as a periodically driven system whose state oscillates about the mean-field point
with amplitude increasing in $L$. Then (i) the blocked arm is $\pi_L$ at maximal $L$, stopped at
maximal displacement rather than a distinct mechanism; and (ii) reversing which source occupies
the final block relocates the endpoint to the opposite side of the cycle, moving
$(\mathrm{acc}_A,\mathrm{acc}_B)$ antisymmetrically while leaving $C$ fixed to first order.
\end{proposition}

\noindent Part~(ii) is a prediction with a sign and a symmetry, and \S\ref{sec:phase} reports the
measurement that was frozen against it. We label the limit-cycle picture \emph{post hoc}: it was
formed after the block-length sweep and before the mirror test, and only the mirror test is
evidence for it.

\begin{remark}[Which kind of wall this is, and why the distinction is load-bearing]
\label{rem:wallkind}
It would be natural to file this next to CCH's Shannon wall as a capacity limit, and that
would be wrong. A $7$B model has ample room to hold both conventions; nothing here is
\emph{unable}. The mixture appears because the two halves are disjoint problem sets, so the
convention that applies to a \emph{held-out} question is not a function of anything the model
may condition on at answer time, and when $H(\text{convention}\mid\text{query})>0$, the
loss-minimising policy \emph{is} the conditional mixture. The model is not failing; it is
correct.

The two wall species have different escapes, which is what makes the distinction operational
rather than semantic. A capacity wall is escaped by adding addressable bits: more state, or
an index channel, which is CCH's move. This one is escaped only by putting the convention into
what the query can condition on, which is what the key of \S\ref{sec:key} does: it does not add
capacity, it changes which policy is optimal. That is also why no arrangement of the data
escapes it (\S\ref{sec:average}): reordering the path cannot change what a held-out query
conditions on, so no arrangement moves the \emph{optimum}. It does not follow that no arrangement
moves a finite run's endpoint, and Table~\ref{tab:constlr} shows one that does; what is required
is that such an effect be a distance-from-optimum effect, which is the escape hatch the schedule
governs.

\emph{What the family frame is worth here, asked plainly.} The sharp form of the question is
whether removing CCH costs Proposition~\ref{prop:avg} anything. It does not. The averaging theorem is
self-contained, its predictions follow from the mean-field flow alone, and every measured result
in this paper is derivable without any of the family's vocabulary. We state that rather than defend a
frame the evidence does not need. The paragraph above is the honest limit of the correspondence:
once this wall is \emph{not} a capacity wall, the mapping from CCH's budget-bounded state to a
parameter vector that is not budget-bounded is an analogy about escapes rather than an
isomorphism, and Table~\ref{tab:family}'s two columns should be read as two systems that happen
to be escaped the same two ways. What the frame did supply is the key: it was built because the
family's access-complete construction predicted a query-time index would work at training time,
and it did. A frame that suggested one experiment which then succeeded earns a section of
discussion. It does not earn being the object the paper is organised around, and it no longer is.
\end{remark}

\section{The Instrument, and What the Path Writes}\label{sec:exp}

\subsection{The Instrument}\label{sec:design}

\paragraph{Corpora.} Two conflict constructions over competition-mathematics problems with
integer answers drawn from the standard benchmarks \citep{hendrycks2021math,openr1math}, each split into halves of matched size with byte-identical prompts:
\emph{synthetic} (cf2: one half's boxed answers shifted $+1$ (contradictory by construction, $864$ rows per condition) and \emph{natural} (nat: numeral against spelled-out answers) both correct, $720$ rows per condition), with same-convention controls for each.
The ratio corpora (rr) rebuild the natural conflict at $936$ rows with mixture weights
$1{:}1$ and $2{:}1$. Corpora, splits and row counts are asserted by fixed-seed build scripts before any
training, and Figure~\ref{fig:instrument} draws the \texttt{nat} construction end to end.
\emph{The natural construction is built, and the rate it is built at is borrowed.} The survey
\citep{gold2026} puts form disagreement at $12.7\%$ of shared problems on the corpus pair
with the most multiplicity, rising to $27.1\%$ against externally authored answer keys; both
figures are measured there and neither is re-derived here, so a reader who wants them checked
has that paper and not this one. What they license is the choice of construction rather than
any number below: they say the contested case is common enough to be worth an instrument, and
every arm we report is built rather than found.

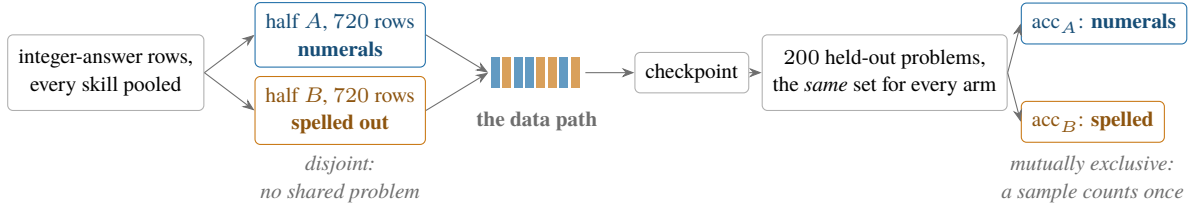
\begin{figure}[t]
\centering
\begin{tikzpicture}[halo/.style={fill=white,inner sep=1pt},font=\small,>=Stealth,
 bx/.style={draw=cGrey!55,rounded corners=2pt,align=center,inner sep=4pt,font=\scriptsize},
 hd/.style={font=\scriptsize\bfseries,cGrey},
 tag/.style={font=\scriptsize\itshape,cGrey,align=center},
 note/.style={font=\scriptsize,cGrey,align=left,inner sep=0pt}]

\node[bx,minimum width=2.5cm,minimum height=1cm] (pool) at (1.25,2.7)
 {integer-answer rows,\\ every skill pooled};

\node[bx,draw=cV,text=cV!70!black,minimum width=1.9cm] (hA) at (4.35,3.2)
 {half $A$, $720$ rows\\ \textbf{numerals}};
\node[bx,draw=cA,text=cA!70!black,minimum width=1.9cm] (hB) at (4.35,2.2)
 {half $B$, $720$ rows\\ \textbf{spelled out}};
\draw[->,cGrey] (pool.east) -- (hA.west);
\draw[->,cGrey] (pool.east) -- (hB.west);

\node[inner sep=0pt,minimum width=1.27cm,minimum height=0.4cm] (path) at (6.95,2.7) {};
\foreach \i/\c in {0/cV,1/cA,2/cV,3/cV,4/cA,5/cA,6/cV,7/cA}{
 \fill[\c!70] (6.35+0.15*\i,2.5) rectangle (6.46+0.15*\i,2.9);}
\node[hd,anchor=north] at (6.95,2.32) {the data path};

\node[bx] (ck) at (9.0,2.7) {checkpoint};
\node[bx,minimum height=1cm,align=center] (ev) at (11.55,2.7)
 {$200$ held-out problems,\\ the \emph{same} set for every arm};
\draw[->,cGrey] (hA.east) -- (path.west);
\draw[->,cGrey] (hB.east) -- (path.west);
\draw[->,cGrey] (path.east) -- (ck.west);
\draw[->,cGrey] (ck.east) -- (ev.west);

\node[bx,draw=cV,text=cV!70!black,anchor=west] (aA) at (13.35,3.35)
 {$\mathrm{acc}_A$: \textbf{numerals}};
\node[bx,draw=cA,text=cA!70!black,anchor=west] (aB) at (13.35,2.05)
 {$\mathrm{acc}_B$: \textbf{spelled}};
\draw[->,cGrey] (ev.east) -- (aA.west);
\draw[->,cGrey] (ev.east) -- (aB.west);

\node[tag,anchor=north] at (4.35,1.74) {disjoint:\\ no shared problem};
\node[tag,anchor=north] at (14.30,1.74) {mutually exclusive:\\ a sample counts once};
\end{tikzpicture}
\caption{\textbf{The instrument: one corpus that disagrees with itself.} The two
\emph{training} halves are disjoint in problems and no single row disagrees with itself: the
disagreement is a property of the corpus. The \emph{evaluation} is the opposite, one held-out set
scored twice against mutually exclusive conventions, so a sample can be credited to at most one.
That is what makes the sum a probability of solving and the ratio a policy. Drawn for
\texttt{nat}; \texttt{cf2} is the same construction at $864$ rows per half.}
\label{fig:instrument}
\end{figure}

\paragraph{Arms and the ladder.} Per condition: \only\ for each half, both blocked orders,
the batch-mixed shuffle, pure batches in random order (\purerand), and alternating
same-source blocks of exactly $L$ consecutive optimiser steps,
$L\in\{1,3,6,9,18,27,54\}$, trainer shuffling off so the file order is the path
(Figure~\ref{fig:ladder}). Training is full-parameter SFT on Qwen2.5-7B \citep{qwen25} (AdamW,
$\beta_1{=}0.9$, $\beta_2{=}0.999$), three epochs at lr $3{\times}10^{-5}$ (the budget at
which these corpora are learnable), three seeds on every ladder rung and eight on the \only,
\shuf\ and \blocked\ arms and on the keyed corpus, run with \texttt{ms-swift}
\citep{zhao2024swift} and evaluated with vLLM \citep{kwon2023vllm}; the conflict switch itself is
additionally measured at $3$B and $14$B by \thecompanion\ paper \citep{frl2026}, which is a
pointer and not evidence a reader of this paper can check here (\S\ref{sec:threats}).

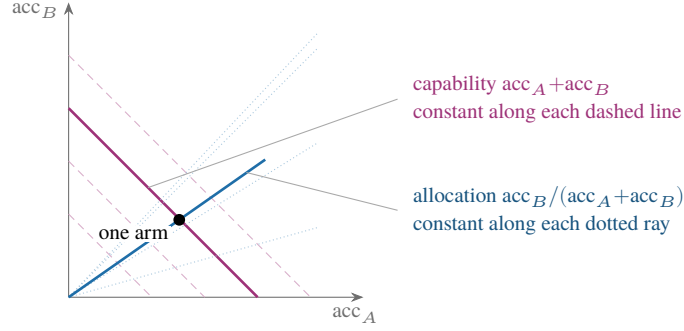
\begin{figure}[t]
\centering
%
%
%
%
%
\begin{tikzpicture}[halo/.style={fill=white,inner sep=1pt},font=\small,>=Stealth]
\begin{scope}
\def\axlen{3.4}
\draw[cGrey,thin,->] (0,0) -- (\axlen+0.5,0)
 node[anchor=north,font=\scriptsize,cGrey,xshift=-3pt] {$\mathrm{acc}_A$};
\draw[cGrey,thin,->] (0,0) -- (0,\axlen+0.5)
 node[anchor=east,font=\scriptsize,cGrey,yshift=-3pt] {$\mathrm{acc}_B$};

\foreach \k in {1.1,1.8,2.5,3.2}{\draw[cD!40,densely dashed] (0,\k) -- (\k,0);}
\draw[cD,line width=1pt] (0,2.5) -- (2.5,0);
\foreach \m in {0.28,0.62,1.0,1.7}{\draw[cV!40,densely dotted] (0,0) -- (3.3,{min(3.3*\m,3.5)});}
\draw[cV,line width=1pt] (0,0) -- (2.6,1.82);

\fill[black] (1.463,1.024) circle (2.2pt);
\node[font=\scriptsize,anchor=north east,halo] at (1.36,0.94) {one arm};

\draw[cGrey!60,line width=0.4pt] (1.05,1.45) -- (4.35,2.62);
\node[cD,font=\scriptsize,anchor=west,align=left] at (4.42,2.62)
 {capability $\mathrm{acc}_A{+}\mathrm{acc}_B$\\ constant along each \textcolor{cD}{dashed} line};
\draw[cGrey!60,line width=0.4pt] (2.35,1.65) -- (4.35,1.15);
\node[cV!75!black,font=\scriptsize,anchor=west,align=left] at (4.42,1.15)
 {allocation $\mathrm{acc}_B/(\mathrm{acc}_A{+}\mathrm{acc}_B)$\\ constant along each \textcolor{cV!75!black}{dotted} ray};
\end{scope}
\end{tikzpicture}
\caption{\textbf{The one thing that reads wrongly in prose.} The read-out of
Figure~\ref{fig:instrument} as a change of coordinates: capability is constant along the dashed
anti-diagonals, allocation along the dotted rays. A movement \emph{along a ray} therefore changes
capability at fixed allocation and a movement \emph{along a diagonal} does the reverse, and an
undecomposed accuracy cannot tell the two apart although their meanings are opposite.}
\label{fig:readout}
\end{figure}

\paragraph{What $\mathrm{acc}_A$ and $\mathrm{acc}_B$ score.} Two properties of the construction
are easy to mis-read and both are load-bearing. First, the two training halves are \emph{disjoint
in problems}: no training row disagrees with itself, and the disagreement is a property of the
corpus rather than of any example. What is contested is the convention a \emph{held-out} problem
should be answered under, which is not a function of anything the model may condition on, which is
the condition Remark~\ref{prop:noescape} needs; the disjointness creates the underdetermination
rather than weakening it. Second, on the synthetic corpus \texttt{cf2} the two conventions are a
boxed answer and that answer shifted by one, so $\mathrm{acc}_A$ and $\mathrm{acc}_B$ measure
\emph{conformance to a stipulated convention} rather than correctness against mathematical ground
truth. Every arm is scored against both, and Appendix~\ref{app:v1} reports what happens when only
one is installed. The phrase ``neither is wrong'' is exact on \texttt{nat} and is a statement
about the \emph{scoring rule} on \texttt{cf2}, which carries the flagship arms because its two
conventions are the ones the scorer separates cleanly.

\begin{table}[!t]
\centering\small
\setlength{\tabcolsep}{4pt}
\caption{\textbf{Two corpora, and which one carries which claim.} \emph{The philosophical claim of
this paper is a claim about} \texttt{nat}\emph{, and} \texttt{nat} \emph{is the weaker instrument};
we print the pair rather than leave it to be assembled from four sections. What travels between them
is every \emph{qualitative} statement the paper defends: capability conserved, allocation moved,
interior flat, mirror antisymmetric at fixed $C$. What does not travel is size, and the direction is
against us, so the headline figures in the abstract are \texttt{cf2} figures and a reader rebuilding
this on a both-correct conflict should expect two thirds the capability, and, once the step size is
matched (\S\ref{sec:family}, branch \textbf{NI-1}), about seven eighths the install.
The last block is the arms \texttt{nat} does not have; they are not claimed for it.
\S\ref{sec:family} gives the full accounting, including the one asymmetry the construction itself
creates and the budget gate that measures how much of the size gap is the step size. Source: \texttt{review\_statistics}, \texttt{nat\_replication}.}
\label{tab:corpora}
\begin{tabular}{@{}llll@{}}
\toprule
 & \texttt{cf2} \emph{(flagship)} & \texttt{nat} & status \\
\midrule
\multicolumn{4}{@{}l}{\emph{what the two conventions are}}\\
\quad construction & boxed answer vs.\ that $+1$ & numeral vs.\ spelled out & --- \\
\quad is either wrong? & yes, by construction & \textbf{neither} & claim is about \texttt{nat} \\
\quad rows per condition & $864$ & $720$ & --- \\
\quad learning rate & $3{\times}10^{-5}$ & $1{\times}10^{-5}$ & \texttt{nat} ran colder \\
\addlinespace[2pt]
\multicolumn{4}{@{}l}{\emph{how strong an instrument each is}}\\
\quad minority install ($B$-\only) & $0.4225$ & $0.2091$ & \emph{unmatched rates} \\
\quad\quad at lr $3{\times}10^{-5}$ & $0.4225$ & $\mathbf{0.3738}$ & \textbf{NI-1}: $7/8$ \\
\quad capability, interleaved arm & $0.5202$ & $0.3727$ & two thirds, unmatched \\
\quad minority form leaked untrained & $0.0169$ & $\mathbf{0.0000}$ & \emph{different regimes} \\
\addlinespace[2pt]
\multicolumn{4}{@{}l}{\emph{what each corpus is asked to carry}}\\
\quad capability conserved & yes & yes & \textbf{agrees} \\
\quad allocation moves & yes & yes & \textbf{agrees} \\
\quad ladder interior flat & $2.20$ floors ($10$ arms) & $1.92$ floors ($4$ rungs) & \textbf{agrees} \\
\quad mirror antisymmetric & $+0.0850/{-}0.0875$ & $+0.0183/{-}0.0229$ & \textbf{sign agrees} \\
\quad mirror leaves $C$ fixed & $0.0025$ & $0.0046$ & \textbf{agrees} \\
\addlinespace[2pt]
\multicolumn{4}{@{}l}{\emph{what only} \texttt{cf2} \emph{carries}}\\
\quad the $12.29\sigma$ switch & \checkmark & not measured & \emph{cf2 only} \\
\quad schedule contrast, $11.63$ floors & \checkmark & not measured & \emph{cf2 only} \\
\quad the key ($87.5\%$ of ceiling) & \checkmark & not measured & \emph{cf2 only} \\
\bottomrule
\end{tabular}

\end{table}

\paragraph{Observables.} Because the two evaluation sets are the same problems under
mutually exclusive conventions, every checkpoint yields a two-component read-out
(Figure~\ref{fig:readout}):
\begin{equation}\label{eq:decomp}
\underbrace{\mathrm{acc}_A+\mathrm{acc}_B}_{\text{capability: }P(\text{solved at all})}
\qquad\text{and}\qquad
\underbrace{\frac{\mathrm{acc}_B}{\mathrm{acc}_A+\mathrm{acc}_B}}_{\text{allocation:
}P(\text{convention }B\mid\text{solved})}.
\end{equation}
Arrangement claims are claims about which component moves. Evaluation is $k{=}4$ samples on
$200$ held-out problems, exact match on the boxed answer;

\emph{Those $200$ are a prefix of a larger pool, and the prefix is not skill-balanced.} The
held-out file carries $292$ problems on \texttt{nat} and $342$ on \texttt{cf2}, sorted by skill,
and the evaluation takes the first $200$. Every arm sees the identical problems, which is what the
contrasts require, but the set is not the pooled corpus: algebra and geometry are complete,
combinatorics is truncated ($97\to48$ on \texttt{nat}, $129\to33$ on \texttt{cf2}), and
\emph{calculus and physics are absent entirely}. Every absolute capability level in this paper is
therefore measured on three skills rather than five, and the difficulty coupling of
Definition~\ref{def:decomp} on the same restricted set. Nothing here affects a between-arm
comparison and everything here affects a level. It is reported rather than repaired because the
checkpoints for most of these arms were reclaimed after evaluation, so the pool cannot be
re-scored (\S\ref{sec:threats}).

The single-run noise
floor is $\hat\sigma=0.0234$ and the three-seed contrast floor $0.0191$, inherited with the
full protocol from \thecompanion\ result base; reliability of each contrast is reported as a
\emph{contrast reliability} (an intraclass correlation over repeated measurements of the same
contrast, \citealp{shrout1979icc}), and the arrangement contrast's is $0.197$, which
is why every claim below is stated on the decomposition of Definition~\ref{def:decomp} rather
than on a raw difference. Every number below is written by a named analyzer to a verdict file rather than transcribed by hand, and the
deciding tests' read-out rules were frozen in the job scripts before the runs; where a test
was \emph{redesigned} after a diagnosis (one was: \S\ref{sec:average}), the redesign is
labelled. Registering in the script the scheduler executes is checkable rather than promised, so we make it checkable: a registration index released with the paper
(Appendix~\ref{app:registry}) lists every registration against the commit
that froze each threshold and the timestamp of the result it decided, with the interval measured
in each case: twelve rows, ten running from under three hours to nearly three days ahead of their
result and two that do not, each of the two stated in the index rather than averaged away. The
palindromic schedule of \S\ref{sec:average} is one of the two, and that paragraph says so.

\paragraph{The instrument's first failure, and what it diagnosed.} Every
arrangement number in this paper depends on the conflict corpus actually creating a conflict
the training run can feel, so the positive control (conflict against its own
same-convention twin) is the load-bearing check, and on the first build it came out
\emph{backwards}. Drawing the integer-answer rows from algebra alone gave $238$ per condition
at one epoch, and the read-out is Table~\ref{tab:builds}.

\begin{table}[t]
\centering\small
\caption{\textbf{The instrument, across its three builds.} $D$ is the arrangement contrast on the
last-seen convention, \emph{signed}: $D=\mathrm{acc}(\text{last seen})_{\blocked}-
\mathrm{acc}(\text{last seen})_{\shuf}$ with the sources in the published order, so the switch
is negative. Elsewhere the same quantity is quoted as a \emph{span} between two arms and is
therefore positive; $-0.2346$ here and $+0.2275$ in \S\ref{sec:average} are the same measurement
at eight and at three seeds, not two results with different signs. The control is the same construction with both halves under one convention,
so it must stay inert; the paragraph below reads v1, where it did not. v1 and v2 are three-seed
builds; v3 and its control carry eight. Sources: the \texttt{conflict\_switch}, \texttt{conflict\_switch2} and
\texttt{conflict\_switch2b3lr3e5} verdicts.}
\label{tab:builds}
\setlength{\tabcolsep}{4pt}
\begin{tabular}{lcccl}
\toprule
build & control $D$ & conflict $D$ & conflict $\sigma$ & verdict \\
\midrule
v1 \quad algebra pool, $1$ epoch & $+0.0461$ ($2.41\sigma$) & $+0.0020$ & $0.10$ & positive control \textbf{fails} \\
v2 \quad pooled skills, $3$ epochs & $+0.0017$ ($0.09\sigma$) & $-0.1425$ & $7.46$ & control inert, switch fires \\
v3 \quad the learning budget, lr $3{\times}10^{-5}$ & $-0.0028$ ($0.15\sigma$) & $\mathbf{-0.2346}$ & $\mathbf{12.29}$ & control inert, switch fires \\
\bottomrule
\end{tabular}

\end{table}

In v1 the arm that was supposed to move did not ($0.10\sigma$) and the arm that was supposed to
be inert did ($2.41\sigma$), the exact inversion of the design. The conflict arms scored $0.02$
against the shifted golds while the same checkpoints scored $0.39$--$0.50$ against unshifted ones,
so the $+1$ convention had lost to the pretrained prior and $D$ was being measured at the floor of
an instrument whose treatment had never been installed. Pooling the integer-answer rows of every
skill and raising the epochs gave the conflicting convention roughly an order of magnitude more
gradient steps, and v2 fired. We report v1 rather than only the builds that worked because the
$12.29\sigma$ of v3 is otherwise unreadable: it is contingent on the conflict being
\emph{learnable at the budget}, a property of the corpus and the budget together and not of
conflict as such.

\subsection{Allocation Moves, Capability Does Not}\label{sec:conserve}

Table~\ref{tab:conserve} is the decomposition of Eq.~\eqref{eq:decomp} across every arm of
the synthetic conflict. The capability sum is constant to within $9.7\%$ of its mean across
all twelve arms (\emph{including} \only\ arms that never saw a row of the other convention) while the allocation share runs from $0.04$ to $0.87$. The coherent control shows the
same sum discipline at $1.01$ with the share pinned at $0.500$. The natural conflict
reproduces the pattern (sums $0.33$--$0.385$; shares $0.000\to0.473\to0.691$), and so do
both ratio corpora (sums $0.335$--$0.381$; Table~\ref{tab:ratio}).

\begin{table}[t]
\centering\small
\caption{\textbf{The conservation decomposition} (synthetic conflict, Qwen2.5-7B, three
seeds). Capability $\mathrm{acc}_A{+}\mathrm{acc}_B$ is flat across every arrangement;
only the allocation moves. Source: the \texttt{cf2b3lr3e5}/\texttt{ctl2b3lr3e5} result
families ($3$ epochs, lr $3{\times}10^{-5}$, $n{=}200$), named because a second conflict
family at a different budget also sits in the result base, and the two must not be pooled.}
\label{tab:conserve}
\begin{tabular}{lcccc}
\toprule
arm & $\mathrm{acc}_A$ & $\mathrm{acc}_B$ & capability & share of $B$ \\
\midrule
$A$ \only & $0.483$ & $0.018$ & $0.502$ & $0.04$ \\
$B$ \only & $0.064$ & $0.423$ & $0.487$ & $0.87$ \\
\shuf & $0.310$ & $0.218$ & $0.528$ & $0.41$ \\
\purerand & $0.303$ & $0.224$ & $0.527$ & $0.43$ \\
$L=1\ldots54$ & $0.272$--$0.318$ & $0.216$--$0.265$ & $0.519$--$0.537$ & $0.41$--$0.49$ \\
$\Astep$ (\blocked) & $0.067$ & $0.445$ & $0.513$ & $0.87$ \\
\midrule
coherent control (all arms) & & & $\approx1.01$ & $0.500$ \\
\bottomrule
\end{tabular}

\end{table}

Seen once, the decomposition is nearly an identity: the two evaluations partition the same
problems, so the sum is the probability of solving at all and the share is the conditional
convention choice. We state it that way rather than dressing it as a discovered law: the
finding is that \emph{nobody measures this decomposition}, and that it dissolves what the
undecomposed metric manufactures. The arrangement switch that exact match reports at
$-0.2346$, $12.29\times$ the noise floor and measured here (Table~\ref{tab:builds}), is a movement of the share at fixed
sum: under a metric that accepts either convention it is \emph{exactly zero by construction}.
Exact match conflates capability with commitment, and everything recipe-like about conflict
lives in the commitment component.

\emph{The agnostic rule matters, and there are two.} Crediting \emph{either} convention makes the
switch zero identically, because on a solved problem reallocating $k{=}4$ samples between the two
conventions conserves their sum. Crediting each problem's \emph{better} convention does not:
where a mixed arm splits its samples across conventions within one problem, the per-problem
maximum falls while the sum holds, and \thecompanion\ paper measures that contrast on this same
corpus at $4.65$ floors \citep{frl2026}. The two rules disagreeing is not a contradiction; it is
Figure~\ref{fig:fingerprint}'s mixing read by a rule that penalises mixing, and it is why this
paper's zero is stated for the sum rule and for no other.

\paragraph{The coin-flip fingerprint.} The mixture is visible per sample, and
Figure~\ref{fig:fingerprint} is the whole of the evidence for calling it a policy rather than a
deficit. For each problem take the fraction of its $k{=}4$ samples that land on whichever
convention that problem favours. Committed arms (\only, \blocked) put $\mathbf{26}$--$\mathbf{30\%}$
of problems at $4/4$ and $41$--$46\%$ in the intermediate bins; every mixed arm (\shuf,
\purerand, $L{=}27$) puts $\mathbf{7}$--$\mathbf{9\%}$ at $4/4$ and $64$--$68\%$ in the middle: the same problem answered under different conventions across four samples. The unsolved mass is
the same in both, $25$--$30\%$, which is the point: conflict has not made the model unable, it
has made the \emph{policy stochastic}, which is what the Bayes-optimal response to contradictory
supervision is.

\emph{One binning throughout.} The contrast depends on how a problem is binned, and the two
natural choices give different sizes: counting each problem's best convention gives a factor of
about three and a half, counting each (problem, convention) cell separately about three. The
figures above use the first throughout. Mixing the two inflates the contrast to about seven, so
the definition is stated rather than left to a reader to infer.

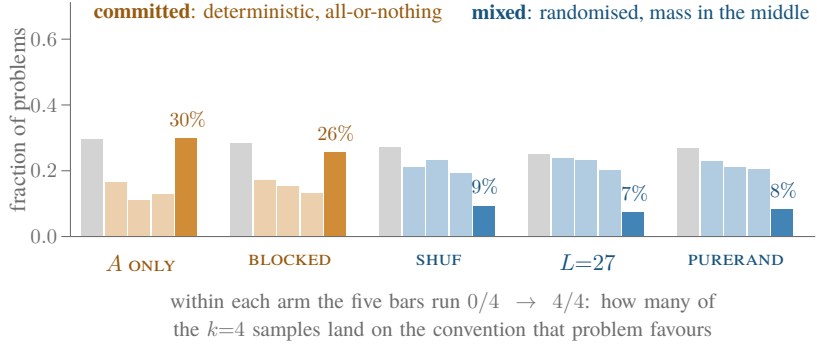
\begin{figure}[t]
\centering
\begin{tikzpicture}[font=\small,>=Stealth]
\node[anchor=north,scale=0.86,cA!75!black] at (1.68,-0.12) {$A$ \only};
\fill[cGrey!30] (0.90,0) rectangle (1.18,1.290);
\fill[cA!35] (1.21,0) rectangle (1.49,0.718);
\fill[cA!35] (1.52,0) rectangle (1.80,0.478);
\fill[cA!35] (1.83,0) rectangle (2.11,0.566);
\fill[cA!85] (2.14,0) rectangle (2.42,1.298);
\node[anchor=south,scale=0.75,cA!75!black] at (2.30,1.348) {$30\%$};
\node[anchor=north,scale=0.86,cA!75!black] at (3.65,-0.12) {\blocked};
\fill[cGrey!30] (2.87,0) rectangle (3.15,1.232);
\fill[cA!35] (3.18,0) rectangle (3.46,0.747);
\fill[cA!35] (3.49,0) rectangle (3.77,0.674);
\fill[cA!35] (3.80,0) rectangle (4.08,0.580);
\fill[cA!85] (4.11,0) rectangle (4.39,1.116);
\node[anchor=south,scale=0.75,cA!75!black] at (4.26,1.166) {$26\%$};
\node[anchor=north,scale=0.86,cV!75!black] at (5.62,-0.12) {\shuf};
\fill[cGrey!30] (4.84,0) rectangle (5.12,1.182);
\fill[cV!35] (5.15,0) rectangle (5.43,0.913);
\fill[cV!35] (5.46,0) rectangle (5.74,1.015);
\fill[cV!35] (5.77,0) rectangle (6.05,0.834);
\fill[cV!85] (6.08,0) rectangle (6.36,0.406);
\node[anchor=south,scale=0.75,cV!75!black] at (6.23,0.456) {$9\%$};
\node[anchor=north,scale=0.86,cV!75!black] at (7.59,-0.12) {$L{=}27$};
\fill[cGrey!30] (6.81,0) rectangle (7.09,1.087);
\fill[cV!35] (7.12,0) rectangle (7.40,1.044);
\fill[cV!35] (7.43,0) rectangle (7.71,1.015);
\fill[cV!35] (7.74,0) rectangle (8.02,0.884);
\fill[cV!85] (8.05,0) rectangle (8.33,0.319);
\node[anchor=south,scale=0.75,cV!75!black] at (8.21,0.369) {$7\%$};
\node[anchor=north,scale=0.86,cV!75!black] at (9.55,-0.12) {\purerand};
\fill[cGrey!30] (8.78,0) rectangle (9.06,1.174);
\fill[cV!35] (9.09,0) rectangle (9.37,1.000);
\fill[cV!35] (9.40,0) rectangle (9.68,0.913);
\fill[cV!35] (9.71,0) rectangle (9.99,0.899);
\fill[cV!85] (10.02,0) rectangle (10.30,0.362);
\node[anchor=south,scale=0.75,cV!75!black] at (10.18,0.412) {$8\%$};
\draw[cGrey,thin] (0.75,0) -- (10.6,0);
\draw[cGrey,thin] (0.75,0) -- (0.75,3.1);
\draw[cGrey,thin] (0.63,0.000) -- (0.75,0.000) node[left,xshift=-2pt,scale=0.78] {$0.0$};
\draw[cGrey,thin] (0.63,0.870) -- (0.75,0.870) node[left,xshift=-2pt,scale=0.78] {$0.2$};
\draw[cGrey,thin] (0.63,1.740) -- (0.75,1.740) node[left,xshift=-2pt,scale=0.78] {$0.4$};
\draw[cGrey,thin] (0.63,2.610) -- (0.75,2.610) node[left,xshift=-2pt,scale=0.78] {$0.6$};
\node[anchor=center,scale=0.85,rotate=90,cGrey] at (0.08,1.55) {fraction of problems};
\node[anchor=north,scale=0.8,cGrey,align=center,text width=11.4cm] at (5.675,-0.62)
 {within each arm the five bars run $0/4 \to 4/4$: how many of the $k{=}4$ samples
  land on the convention that problem favours};
\node[anchor=west,scale=0.80,cA!75!black] at (0.95,2.95) {\textbf{committed}: deterministic, all-or-nothing};
\node[anchor=west,scale=0.80,cV!75!black] at (5.95,2.95) {\textbf{mixed}: randomised, mass in the middle};
\end{tikzpicture}
\caption{\textbf{Conflict makes the policy stochastic, not the model weak.} Per-problem
commitment at $k{=}4$ samples on the synthetic conflict, three seeds, $200$ problems each. Bars
within an arm run $0/4$ to $4/4$. The committed arms are bimodal (the model either answers a
problem the same way every time or cannot answer it) while the mixed arms move that mass into
the middle without changing the unsolved fraction (grey, leftmost bar in each group). Same
corpus, same volume, same budget; only the data path differs. Source \texttt{cf2b3lr3e5}
per-problem arrays.}
\label{fig:fingerprint}
\end{figure}

\paragraph{Qualitative structure: the same problems, re-labelled.} The decomposition says
\emph{how much} moves; the per-problem arrays say \emph{which problems} carry it, and the answer
is the same ones. Crossing from \shuf\ to \blocked\ (pooled over three seeds, $600$ problem
instances), the \emph{set} of solved problems barely changes, $7.3\%$ newly solved, $8.5\%$
lost, net $-1.2\%$, but among the $386$ instances solved in \emph{both} arms, $48.2\%$ change
which convention they favour, and the change is one-directional: $179$ flip $A{\to}B$ against
$7$ flipping $B{\to}A$, a $26\times$ asymmetry pointing at the source the blocked arm ends on
(Table~\ref{tab:relabel}). Blocking does not teach different problems and does not unteach the
old ones; it re-labels the problems the model already solves. That is what ``a policy, not
damage'' looks like at the level of individual items, and it is the per-problem face of the
antisymmetric mirror shift of \S\ref{sec:phase}.

\begin{table}[H]
\centering\small
\caption{\textbf{Per-problem transitions from \shuf\ to \blocked} (synthetic conflict, three
seeds $\times$ $200$ problems). A problem's state is the convention receiving more of its $k{=}4$
samples ($A$, $B$, tie) or unsolved ($U$). The solved set is nearly fixed while nearly half of it
changes label, $26\times$ more often toward the convention trained last. The tie row carries the
$3$ transitions $U\to$tie and the $7$ tie$\to U$, which is why the prose's newly-solved and lost
counts ($44$ and $51$) exceed what the four $U$ rows sum to ($41$ and $44$).
Source \texttt{cf2b3lr3e5} per-problem arrays.}
\label{tab:relabel}
\begin{tabular}{lrr@{\hskip 2.2em}lrr}
\toprule
transition & count & frac & transition & count & frac \\
\midrule
$A \to B$ (re-labelled) & $179$ & $0.298$ & $U \to B$ (newly solved) & $26$ & $0.043$ \\
$U \to U$ (never solved) & $119$ & $0.198$ & $A \to U$ (lost) & $28$ & $0.047$ \\
$B \to B$ (stable) & $104$ & $0.173$ & $U \to A$ (newly solved) & $15$ & $0.025$ \\
$A \to A$ (stable) & $28$ & $0.047$ & $B \to U$ (lost) & $16$ & $0.027$ \\
$B \to A$ (re-labelled) & $7$ & $0.012$ & ties, either side & $78$ & $0.130$ \\
\bottomrule
\end{tabular}

\end{table}

\paragraph{How exact the conservation is.} It is an approximation, and three measurements bound
it. The measure comes first, because the numbers are otherwise comparable only by accident: it is
the \emph{full range} of the capability sum across arms, divided by its mean. On that measure the
fifteen synthetic arms give $9.97\%$ (a range of $0.0521$ about a mean of $0.5220$) and the
natural arms $15\%$. That no single arm sits more than $6.78\%$ from the mean is also true and is
the statement to ignore, since a conservation claim should be priced by its worst pair rather than
its worst arm. The range is carried by the two extreme arms, so the two dispersion measures that
are not are also recorded: across the same fifteen arms the standard deviation is $2.74\%$ of the
mean and the interquartile range $2.99\%$. \emph{Fifteen arms and not fourteen: the census was
missing \texttt{B\_then\_A}, the reverse of an ordering it already contained}, so it was asymmetric
in exactly the variable this section is about; restoring it leaves the range where the two extreme
arms had it and moves the interquartile range from $2.51\%$ to $2.99\%$. Third, \emph{along} a
single training trajectory the sum moves through a range of $0.1475$ before settling: endpoint
conservation is not path conservation. The sharpest statement we defend is \S\ref{sec:key}'s: the
key recovers $74.9\%$ of the gap to the additive ceiling, so ``conserved'' means \emph{to within an
eighth}, with a mechanism for the remainder.

\begin{figure}[t]
\centering
\begin{tikzpicture}[halo/.style={fill=white,inner sep=1pt},font=\small,>=Stealth]
\begin{scope}[yshift=3.5cm]
\draw[cGrey,thin] (-0.2,0) -- (10.3,0);
\draw[cGrey,thin] (-0.2,0) -- (-0.2,2.75);
\draw[cGrey,thin] (-0.32,0.000) -- (-0.2,0.000) node[left,xshift=-2pt,scale=0.8] {$0.0$};
\draw[cGrey,thin] (-0.32,0.640) -- (-0.2,0.640) node[left,xshift=-2pt,scale=0.8] {$0.2$};
\draw[cGrey,thin] (-0.32,1.280) -- (-0.2,1.280) node[left,xshift=-2pt,scale=0.8] {$0.4$};
\draw[cGrey,thin] (-0.32,1.920) -- (-0.2,1.920) node[left,xshift=-2pt,scale=0.8] {$0.6$};
\draw[cGrey,thin] (-0.32,2.560) -- (-0.2,2.560) node[left,xshift=-2pt,scale=0.8] {$0.8$};
\fill[cGrey!10] (0,0) rectangle (0.741,2.75);
\draw[cA,very thick] plot coordinates {(0.000,0.093) (0.741,0.196) (1.481,0.276) (2.222,1.382) (2.963,1.942) (3.704,2.112) (4.444,2.069) (5.185,2.244) (5.926,2.301) (6.667,2.305) (7.407,2.267) (8.148,2.337) (8.889,2.358) (9.630,2.328) (10.000,2.299)};
\fill[cA] (0.000,0.093) circle (1.4pt);
\fill[cA] (0.741,0.196) circle (1.4pt);
\fill[cA] (1.481,0.276) circle (1.4pt);
\fill[cA] (2.222,1.382) circle (1.4pt);
\fill[cA] (2.963,1.942) circle (1.4pt);
\fill[cA] (3.704,2.112) circle (1.4pt);
\fill[cA] (4.444,2.069) circle (1.4pt);
\fill[cA] (5.185,2.244) circle (1.4pt);
\fill[cA] (5.926,2.301) circle (1.4pt);
\fill[cA] (6.667,2.305) circle (1.4pt);
\fill[cA] (7.407,2.267) circle (1.4pt);
\fill[cA] (8.148,2.337) circle (1.4pt);
\fill[cA] (8.889,2.358) circle (1.4pt);
\fill[cA] (9.630,2.328) circle (1.4pt);
\fill[cA] (10.000,2.299) circle (1.4pt);
\node[cA!75!black,anchor=west,scale=0.95] at (4.9,0.95) {\textbf{allocation} $\;\mathrm{acc}_B/\mathrm{TOTAL}$};
\node[anchor=west,scale=0.85,cGrey] at (4.9,0.55) {sweeps $0.029\to0.718$ (range $0.708$)};
\node[anchor=west,scale=0.85,cD] at (4.9,0.2) {$4.5\%$ of that range is spent by step $4$};
\end{scope}
\draw[cGrey,thin] (-0.2,0) -- (10.3,0);
\draw[cGrey,thin] (-0.2,0) -- (-0.2,2.35);
\draw[cGrey,thin] (-0.32,0.000) -- (-0.2,0.000) node[left,xshift=-2pt,scale=0.8] {$0.30$};
\draw[cGrey,thin] (-0.32,0.550) -- (-0.2,0.550) node[left,xshift=-2pt,scale=0.8] {$0.35$};
\draw[cGrey,thin] (-0.32,1.100) -- (-0.2,1.100) node[left,xshift=-2pt,scale=0.8] {$0.40$};
\draw[cGrey,thin] (-0.32,1.650) -- (-0.2,1.650) node[left,xshift=-2pt,scale=0.8] {$0.45$};
\draw[cGrey,thin] (-0.32,2.200) -- (-0.2,2.200) node[left,xshift=-2pt,scale=0.8] {$0.50$};
\fill[cV!12] (0,1.013) rectangle (10.00,1.607);
\fill[cGrey!10] (0,0) rectangle (0.741,2.35);
\draw[cV,very thick] plot coordinates {(0.000,2.090) (0.741,0.509) (1.481,0.682) (2.222,0.468) (2.963,0.614) (3.704,0.967) (4.444,1.159) (5.185,1.246) (5.926,1.169) (6.667,1.311) (7.407,1.201) (8.148,1.246) (8.889,1.151) (9.630,1.051) (10.000,1.310)};
\fill[cV] (0.000,2.090) circle (1.4pt);
\fill[cV] (0.741,0.509) circle (1.4pt);
\fill[cV] (1.481,0.682) circle (1.4pt);
\fill[cV] (2.222,0.468) circle (1.4pt);
\fill[cV] (2.963,0.614) circle (1.4pt);
\fill[cV] (3.704,0.967) circle (1.4pt);
\fill[cV] (4.444,1.159) circle (1.4pt);
\fill[cV] (5.185,1.246) circle (1.4pt);
\fill[cV] (5.926,1.169) circle (1.4pt);
\fill[cV] (6.667,1.311) circle (1.4pt);
\fill[cV] (7.407,1.201) circle (1.4pt);
\fill[cV] (8.148,1.246) circle (1.4pt);
\fill[cV] (8.889,1.151) circle (1.4pt);
\fill[cV] (9.630,1.051) circle (1.4pt);
\fill[cV] (10.000,1.310) circle (1.4pt);
\node[cD,anchor=north west,scale=0.88,align=left] at (0.95,2.32) {$97.4\%$ of capability's whole excursion\\is spent in the first four steps};
\draw[cD,densely dashed,->] (1.077,1.44) -- (0.80,0.60);
\node[cV!75!black,anchor=west,scale=0.95] at (4.55,0.80) {\textbf{capability} $\;\mathrm{acc}_A+\mathrm{acc}_B$};
\node[anchor=west,scale=0.85,cGrey] at (4.55,0.34) {range $0.1475 = 5.5\times$ its floor; band is $\pm$floor};
\draw[cGrey,thin] (0.000,-0.1) -- (0.000,0) node[below,yshift=-3pt,scale=0.8] {$0$};
\draw[cGrey,thin] (0.741,-0.1) -- (0.741,0) node[below,yshift=-3pt,scale=0.8] {$4$};
\draw[cGrey,thin] (2.222,-0.1) -- (2.222,0) node[below,yshift=-3pt,scale=0.8] {$12$};
\draw[cGrey,thin] (4.444,-0.1) -- (4.444,0) node[below,yshift=-3pt,scale=0.8] {$24$};
\draw[cGrey,thin] (6.667,-0.1) -- (6.667,0) node[below,yshift=-3pt,scale=0.8] {$36$};
\draw[cGrey,thin] (10.000,-0.1) -- (10.000,0) node[below,yshift=-3pt,scale=0.8] {$54$};
\node[anchor=north,scale=0.9] at (5.00,-0.6) {optimiser step along a single $L{=}54$ run (three seeds averaged)};
\end{tikzpicture}
\caption{\textbf{Capability and allocation move on separate clocks, which the range alone
hides.} The \blocked\ construction's second stage: a model trained to completion on $A$ and
then trained on $B$, checkpointed every four optimiser steps over that stage's first $54$, three
seeds averaged. It is not a ladder rung, and its allocation ends at $0.718$ heading toward
\blocked's $0.869$ rather than at the $L{=}54$ rung's $0.492$; \textbf{(top)} allocation $\mathrm{acc}_B/\mathrm{TOTAL}$, \textbf{(bottom)}
capability $\mathrm{acc}_A+\mathrm{acc}_B$ with a band of $\pm$ its floor $0.0270$ around the
endpoint. The shaded column is the first four steps. Capability spends $97.4\%$ of its total
excursion there, while allocation spends $4.5\%$ of its, so the capability dip precedes the
commitment shift rather than accompanying it. The verdict file's audit fields record that both evaluations use
the same $200$ questions with disjoint golds.}
\label{fig:mixtraj}
\end{figure}
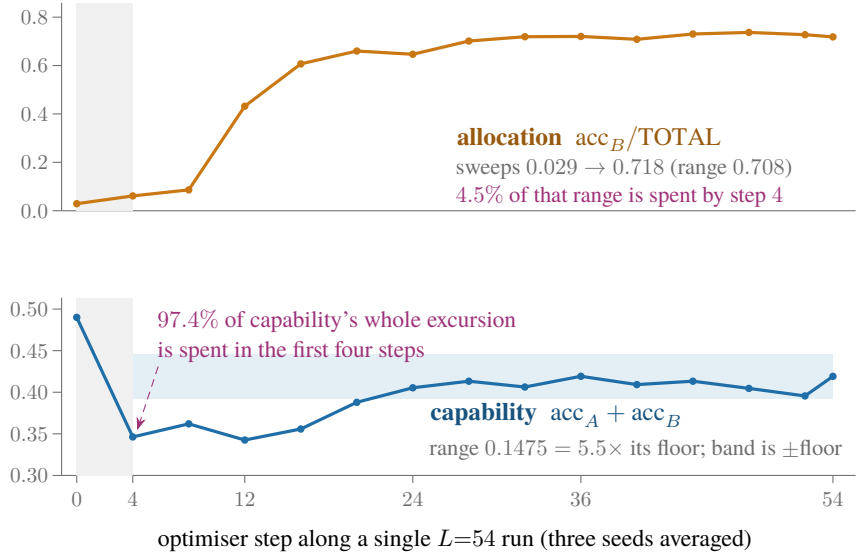

That third bound is the one worth looking at rather than quoting, because its \emph{shape}
says something the range does not (Figure~\ref{fig:mixtraj}). Sampled every four optimiser steps
through the first $54$ steps of \blocked's $B$ stage, resuming from the completed $A$-\only\
checkpoint, the two quantities move on separate clocks. Capability falls
from $0.4900$ to $0.3463$ between steps $0$ and $4$: $97.4\%$ of its entire excursion, and then climbs back to $0.4191$ and stays. Allocation over the same four steps goes $0.029$ to
$0.061$: $4.5\%$ of the range it eventually covers, essentially nothing. The sweep from $0.03$
to $0.72$ happens later, between steps $12$ and $32$, by which time capability has already
returned to within a floor and a half of where it ends.

Two consequences follow, the first against us. \emph{The dip is not the reassignment.} It is
almost entirely spent before allocation begins to move, so whatever costs capability at the
start of conflicted training is a different process from the commitment shift this paper is
about, and the endpoint conservation we report is the state after that process has largely
undone itself. Second, the dip is not a scoring artefact of the mixture: at $k{=}4$ soft
scoring a model that randomised between the two conventions would leave $\mathrm{acc}_A +
\mathrm{acc}_B$ untouched, since each problem contributes to one gold or the other. A drop in
the sum is answers matching \emph{neither}, so conflict does more than reassign, at least
transiently. What we cannot say is which: a trajectory that recorded only accuracy cannot
separate knowledge briefly lost from output the scorer briefly could not parse. That
separation needs a run that keeps generations, and it is listed as a missing cell rather than
argued.

\subsection{The Ladder: What the Path Moves, and What This Design Can See}\label{sec:average}

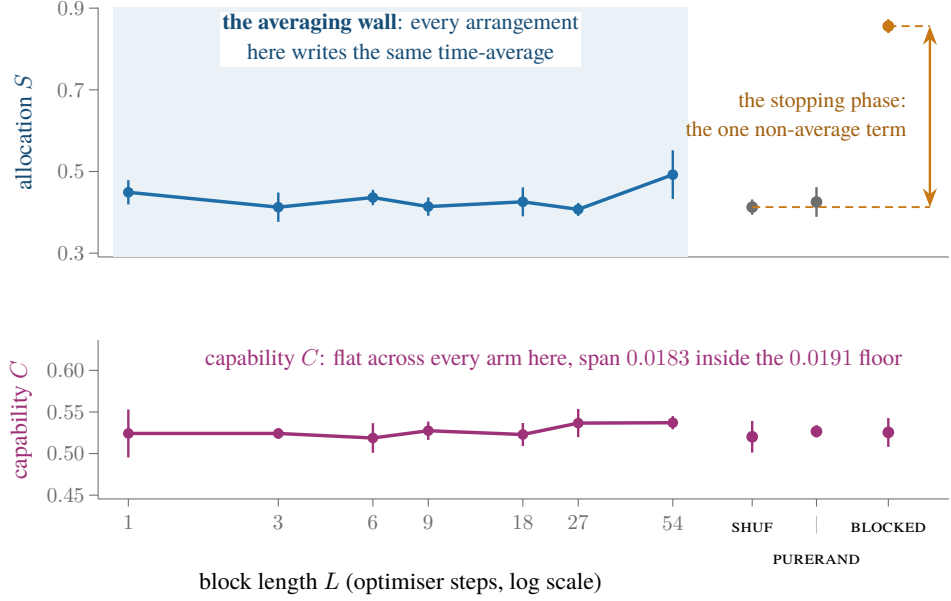
\begin{figure}[t]
\centering
\begin{tikzpicture}[halo/.style={fill=white,inner sep=1pt},font=\small,>=Stealth]
\begin{scope}[yshift=3.2cm]
\draw[cGrey,thin] (0.7,0.3) -- (11.95,0.3);
\draw[cGrey,thin] (0.7,0.3) -- (0.7,3.6);
\draw[cGrey,thin] (0.58,0.350) -- (0.7,0.350) node[left,xshift=-2pt,scale=0.8] {$0.3$};
\draw[cGrey,thin] (0.58,1.430) -- (0.7,1.430) node[left,xshift=-2pt,scale=0.8] {$0.5$};
\draw[cGrey,thin] (0.58,2.510) -- (0.7,2.510) node[left,xshift=-2pt,scale=0.8] {$0.7$};
\draw[cGrey,thin] (0.58,3.590) -- (0.7,3.590) node[left,xshift=-2pt,scale=0.8] {$0.9$};
\fill[cV!9] (0.9,0.3) rectangle (8.5,3.6);
\node[cV!70!black,anchor=north,scale=0.85,align=center,halo] at (4.7,3.55)
 {\textbf{the averaging wall}: every arrangement\\ here writes the same time-average};
\draw[cV,line width=0.9pt] (1.100,0.995) -- (1.100,1.315);
\draw[cV,line width=0.9pt] (3.083,0.764) -- (3.083,1.152);
\draw[cV,line width=0.9pt] (4.334,0.984) -- (4.334,1.188);
\draw[cV,line width=0.9pt] (5.066,0.845) -- (5.066,1.087);
\draw[cV,line width=0.9pt] (6.317,0.837) -- (6.317,1.220);
\draw[cV,line width=0.9pt] (7.049,0.843) -- (7.049,1.014);
\draw[cV,line width=0.9pt] (8.300,1.067) -- (8.300,1.710);
\draw[cV,line width=1.3pt] plot coordinates {(1.100,1.155) (3.083,0.958) (4.334,1.086) (5.066,0.966) (6.317,1.029) (7.049,0.929) (8.300,1.388)};
\fill[cV] (1.100,1.155) circle (2pt);
\fill[cV] (3.083,0.958) circle (2pt);
\fill[cV] (4.334,1.086) circle (2pt);
\fill[cV] (5.066,0.966) circle (2pt);
\fill[cV] (6.317,1.029) circle (2pt);
\fill[cV] (7.049,0.929) circle (2pt);
\fill[cV] (8.300,1.388) circle (2pt);
\draw[cGrey,line width=0.9pt] (9.35,0.858) -- (9.35,1.060);
\fill[cGrey] (9.35,0.959) circle (2.2pt);
\draw[cGrey,line width=0.9pt] (10.2,0.830) -- (10.2,1.223);
\fill[cGrey] (10.2,1.027) circle (2.2pt);
\draw[cA,line width=0.9pt] (11.15,3.260) -- (11.15,3.443);
\fill[cA] (11.15,3.352) circle (2.2pt);
\draw[cA,densely dashed,line width=0.7pt] (9.35,0.959) -- (11.70,0.959);
\draw[cA,densely dashed,line width=0.7pt] (11.15,3.352) -- (11.70,3.352);
\draw[cA,<->,line width=0.9pt] (11.70,0.959) -- (11.70,3.352);
\node[cA!80!black,anchor=east,scale=0.82,align=right] at (11.50,2.16) {the stopping phase:\\ the one non-average term};
\node[anchor=south,scale=0.9,cV!70!black,rotate=90] at (-0.05,1.95) {allocation $S$};
\end{scope}
\draw[cGrey,thin] (0.7,0.3) -- (11.95,0.3);
\draw[cGrey,thin] (0.7,0.3) -- (0.7,2.4);
\draw[cGrey,thin] (0.58,0.350) -- (0.7,0.350) node[left,xshift=-2pt,scale=0.8] {$0.45$};
\draw[cGrey,thin] (0.58,0.900) -- (0.7,0.900) node[left,xshift=-2pt,scale=0.8] {$0.50$};
\draw[cGrey,thin] (0.58,1.450) -- (0.7,1.450) node[left,xshift=-2pt,scale=0.8] {$0.55$};
\draw[cGrey,thin] (0.58,2.000) -- (0.7,2.000) node[left,xshift=-2pt,scale=0.8] {$0.60$};
\draw[cD,line width=0.9pt] (1.100,0.849) -- (1.100,1.483);
\draw[cD,line width=0.9pt] (3.083,1.096) -- (3.083,1.237);
\draw[cD,line width=0.9pt] (4.334,0.911) -- (4.334,1.302);
\draw[cD,line width=0.9pt] (5.066,1.080) -- (5.066,1.324);
\draw[cD,line width=0.9pt] (6.317,1.000) -- (6.317,1.304);
\draw[cD,line width=0.9pt] (7.049,1.117) -- (7.049,1.490);
\draw[cD,line width=0.9pt] (8.300,1.220) -- (8.300,1.396);
\draw[cD,line width=1.3pt] plot coordinates {(1.100,1.166) (3.083,1.166) (4.334,1.106) (5.066,1.202) (6.317,1.152) (7.049,1.303) (8.300,1.308)};
\fill[cD] (1.100,1.166) circle (2pt);
\fill[cD] (3.083,1.166) circle (2pt);
\fill[cD] (4.334,1.106) circle (2pt);
\fill[cD] (5.066,1.202) circle (2pt);
\fill[cD] (6.317,1.152) circle (2pt);
\fill[cD] (7.049,1.303) circle (2pt);
\fill[cD] (8.300,1.308) circle (2pt);
\draw[cD,line width=0.9pt] (9.35,0.913) -- (9.35,1.331);
\fill[cD] (9.35,1.122) circle (2.2pt);
\draw[cD,line width=0.9pt] (10.2,1.113) -- (10.2,1.274);
\fill[cD] (10.2,1.193) circle (2.2pt);
\draw[cD,line width=0.9pt] (11.15,0.989) -- (11.15,1.370);
\fill[cD] (11.15,1.180) circle (2.2pt);
\node[cD,anchor=west,scale=0.86,align=left] at (2.0,2.15) {capability $C$: flat across every arm here, span $0.0183$ inside the $0.0191$ floor};
\node[anchor=south,scale=0.9,cD,rotate=90] at (-0.05,1.35) {capability $C$};
\draw[cGrey,thin] (1.100,0.2) -- (1.100,0.3) node[below,yshift=-3pt,scale=0.8] {$1$};
\draw[cGrey,thin] (3.083,0.2) -- (3.083,0.3) node[below,yshift=-3pt,scale=0.8] {$3$};
\draw[cGrey,thin] (4.334,0.2) -- (4.334,0.3) node[below,yshift=-3pt,scale=0.8] {$6$};
\draw[cGrey,thin] (5.066,0.2) -- (5.066,0.3) node[below,yshift=-3pt,scale=0.8] {$9$};
\draw[cGrey,thin] (6.317,0.2) -- (6.317,0.3) node[below,yshift=-3pt,scale=0.8] {$18$};
\draw[cGrey,thin] (7.049,0.2) -- (7.049,0.3) node[below,yshift=-3pt,scale=0.8] {$27$};
\draw[cGrey,thin] (8.300,0.2) -- (8.300,0.3) node[below,yshift=-3pt,scale=0.8] {$54$};
\draw[cGrey,thin] (9.35,0.2) -- (9.35,0.3);
\node[below,yshift=-3pt,scale=0.78] at (9.35,0.2) {\shuf};
\draw[cGrey,thin] (10.2,0.2) -- (10.2,0.3);
\draw[cGrey!45,thin] (10.2,-0.16) -- (10.2,0.08);
\node[below,yshift=-3pt,scale=0.78] at (10.2,-0.22) {\purerand};
\draw[cGrey,thin] (11.15,0.2) -- (11.15,0.3);
\node[below,yshift=-3pt,scale=0.78] at (11.15,0.2) {\blocked};
\node[anchor=north,scale=0.9] at (4.7,-0.55) {block length $L$ (optimiser steps, log scale)};
\end{tikzpicture}
\caption{\textbf{The averaging wall, measured.} Same corpus, volume, budget and seeds; only the
data path differs (Definition~\ref{def:path}). \textbf{(a)} Allocation is flat from $L{=}1$ to
$L{=}27$ and across \shuf\ and \purerand: $0.407$ to $0.449$ over a $27\times$ change in block
length and a change of batch composition. $L{=}54$ is the first rung off that floor at $0.492$
and \blocked\ sits at $0.856$. \textbf{(b)} Capability over the identical arms spans $0.0183$,
inside the contrast floor: the path relocates commitment without changing what can be solved.
Bars are the seed standard deviation, three seeds per point (\shuf\ and \blocked\ carry eight).
Source \texttt{cf2b3lr3e5}.}
\label{fig:ladderdata}
\end{figure}

\emph{Two tops, two gaps, stated where the paper first needs both.} The occupied
group runs to $L{=}54$ at $0.4922$; the \emph{interior} of the ladder, which excludes $L{=}54$
because that rung is the stopping phase becoming visible, runs to $L{=}1$ at $0.4491$. So
\blocked's distance is $0.3770$ from the group and $0.4201$ from the interior, and both appear
below with the top they are measured from named.

Equation~\eqref{eq:meanfield} makes a strong claim: below corpus scale, \emph{how} the path
is arranged is irrelevant, because every arrangement averages to the same field; the only
thing the path can hand the parameters is its time-average. Four measurements test it, two
of them post-mortems of our own alternatives, and Figure~\ref{fig:ladderdata} is all four of
them on one axis: allocation flat across a $27\times$ change of block length and across a
change of batch composition, capability flat inside its own floor, and a single jump at the
blocked end that belongs to the stopping phase rather than to the path.

\paragraph{Batch purity is not the variable (v1, pre-registered, dead).} The literature
contradicts itself on whether batches should be pure or mixed
\citep{commonit2024,dataorg2026,lu2022grab}, and our first pre-registered mechanism sided
with purity: contested directions cancel \emph{inside a mixed batch}, so batch composition
should carry the conflict effect. The deciding cell (pure batches in \emph{random} order) lands at the shuffled arm, not the blocked one: control span $+0.0037$ (noise), conflict
span $+0.2275$ with \purerand\ at $0.2242$ against \shuf\ $0.2179$ and \blocked\ $0.4454$
(three seeds; that span is the switch $D$ itself, which reads $-0.2346$ at eight seeds in
\S\ref{sec:design}). Purity carries $3\%$ of the span; between-batch order carries the rest.
v1 is dead.

\paragraph{The optimiser's memory is not the variable either (v2, pre-registered, dead).}
One level down, AdamW's momentum is a $10$-step low-pass filter
($1/(1-\beta_1)$; the second-moment timescale exceeds the whole run), so the contested
directions could cancel \emph{inside the optimiser state}: the stateful-optimiser channel
of \citet{sweeney2026shuffle}. The pre-registered signature was a response turning on at
$L^\star\approx10$; a competing basin-escape account tied it instead to an independently
measured residence time $\tau_{\text{write}}=13.6$ steps. The block-length sweep kills both:
the normalised response $M(L)$ reads $+0.075$, $-0.007$, $+0.037$, $+0.002$, $+0.022$,
$+0.002$ at $L=1$, $3$, $6$, $9$, $18$, $27$ (noise $\sigma(M)\approx0.06$), flat through both predicted onsets and
through a quarter of the epoch, with only $L{=}54$ rising ($0.205$) toward the blocked
anchor. The guard passed (control at $L{=}9$: $0.5212$ against \shuf\ $0.5042$), so the
flatness is a measurement, not a floored instrument.

\begin{remark}[The lesson we paid two experiments for]\label{rem:lesson}
Both deaths were derivable in advance. At small $\eta$ averaging is associative:
within-batch averaging and across-step averaging converge to the same averaged field, their
difference higher order, so v1's prediction (the two differ) was impossible to leading
order, and v2's flat curve was equally forced, the momentum window being just another
averaging window strictly inside the corpus scale. We ran two experiments to discover a
theorem we already had. We keep both post-mortems in the paper because the instinct they
formalised (\emph{conflict destroys, and arrangement decides how much}) is the field's
instinct too (\S\ref{sec:related}), and it took the theorem plus two falsifications to
dislodge it in-house.
\end{remark}

\paragraph{The same ladder on a second pretraining family, and a threshold set on the wrong
statistic.} The wall
above is one model, and unlike the switch it had never been asked to travel. It is also a
\emph{null}, which is the hazard: a floored instrument delivers a flat ladder for free, and
\S\ref{sec:design}'s v1 is this paper's own account of what that looks like. So the replication was
registered in stages with the gate frozen first, and the gate is the good news: on Qwen3-8B-Base
the conflict installs \emph{better} than on Qwen2.5, $\mathrm{acc}_B(B\ \only)=0.4813$ against
$0.4225$, with the switch at $15.0$ floors against $11.9$. Whatever the ladder says there, it
cannot be dismissed as an instrument at its floor.

\emph{One convention, stated once, because two artefacts here disagree below the floor.} The
allocation of an arm is the mean over seeds of the per-seed share, which is what
\texttt{rate\_verdict} computes and what every number in this paper quotes. The registered
second-family read-out formed the share of the seed-mean accuracies instead, and the two differ
by at most $0.0007$, which is $0.04$ contrast floors: the published flat span reads $0.0416$
under the registered estimator and $0.0420$ under the reported one, $2.18$ against $2.20$
floors. Where a registered comparison is being quoted we quote the estimator that registration
froze and say so; nothing in the paper turns on a difference of a twentieth of a floor, but a
reader recomputing from the result base will see both and should not have to infer which.

The registered branch is \textbf{W2}: the flat region spans $0.0544$, which is $2.85$ floors
against a bar of two. We report that branch, and in the same breath the thing that makes it
uninterpretable as written. \emph{The published Qwen2.5 ladder spans $0.0416$, which is $2.18$
floors, and does not meet the same bar either.} That number was available before the run and we
did not check the threshold against it; the fault is in the registration, not in the second family.
The two spans differ by $0.67$ floors, less than one, and as a fraction of the excursion the
blocked arm makes they are $9\%$ and $12\%$. Both of those comparisons are \emph{post hoc} and
neither rescues W1: what we can say is that the registered bar was mis-specified and that the
ladder is about as flat on one family as on the other, and what we cannot say is that a bar chosen
after the fact was met.

\emph{The statistic the proposition actually predicts.} A span in floors is a proxy for what
Proposition~\ref{prop:avg} predicts, which is that the interior of the path space collapses to
\emph{one endpoint state}. Definition~\ref{def:channel} states that quantity directly, predates
both second-family runs, and needs no bar: count the clusters the ten arms occupy at the
$2\hat\sigma$ resolution. On Qwen2.5 the answer is \textbf{two} (nine arms in $[0.407,0.492]$,
\blocked\ at $0.869$) and on Qwen3-8B-Base it is \textbf{two} (nine arms in $[0.424,0.479]$,
\blocked\ at $0.891$). Both families train under a cosine, and Table~\ref{tab:constlr} is what
the same count returns when the schedule is changed instead of the pretraining family: at the same
resolution the constant-rate ladder occupies four states. The count replicates across models and
does not survive across schedules, which is the boundary this subsection draws everywhere else.
The margin is not marginal on either: the largest gap inside the occupied cluster is $0.0431$
against a between-cluster gap of $0.3770$ on the first family ($8.7\times$) and $0.0199$ against
$0.4121$ on the second ($20.7\times$), unchanged whether $\hat\sigma$ is the inherited $0.0234$ or
each ladder's own recomputed dispersion. \emph{The wall replicates exactly; the criterion we
registered for it did not measure it.} This reading is labelled \textbf{post hoc} and \textbf{W2}
stands as the registered outcome, because a statistic chosen after seeing a miss is worth less
than one chosen before it: the registered comparison was of two spans neither of which is
resolvable, and the unregistered one is of two integers that agree. Source:
\texttt{rate\_verdict}.

\paragraph{A second ladder, on the other corpus and at a third of the step size.} The rungs
above are \texttt{cf2} at lr $3\times10^{-5}$. A \texttt{nat} ladder exists at lr
$1\times10^{-5}$ over four rungs, $L=1,9,15,45$, and its interior spans $0.0367$, which is
$1.92$ floors against the published $2.20$. The flatness is therefore not a property of one
corpus or one step size. It is not evidence about the schedule, because a cosine decays to zero
at either learning rate. Source: \texttt{nat\_replication}.

\paragraph{The alternative this paper named, ran, and decided against.} Every arm above
trains under a single cosine, so the second half of every run has little step size left, and an
interior that is flat because arrangement does nothing looks exactly like an interior that is
flat because nothing much happens after the midpoint. Those two accounts are separated by one
experiment, the same ladder at a constant rate. It has been run, and it decides for the second.

\begin{table}[H]
\centering\small
\caption{\textbf{The same ladder under two learning-rate schedules.} Ten arms, one corpus, one
budget, the identical data path files, and \emph{the same three seeds}: the intersection,
computed rather than assumed, because the published \shuf\ and \blocked\ arms carry eight and
a mixed comparison would print most of the eight-against-three difference as a schedule effect.
The two families differ in \texttt{lr\_scheduler\_type} and in nothing else, which we checked
against the training arguments and the corpus checksum rather than assuming, and both run to the
same step ($324$ on the ladder, $162$ at the corner). Under a cosine the interior spans $0.0420$, which is $2.20$ contrast floors and below
anything this design can resolve. Under a constant rate the same interior spans $\mathbf{0.2221}$,
which is $11.63$ floors, with an intraclass correlation of $0.836$ whose bootstrap interval
$[0.315,0.926]$ lies entirely above zero. \emph{The corner does not move}: \blocked\ reads
$0.8692$ and $0.8501$, one contrast floor apart, against the $17.0$ floors the rung below it
moves. Source: \texttt{constlr\_ladder\_table},
\texttt{constant\_lr\_verdict}.}
\label{tab:constlr}
\begin{tabular}{@{}lcccc@{}}
\toprule
arm & cosine & constant & difference & floors \\
\midrule
\shuf\ \emph{(mixed within batch)} & $0.4129$ & $0.3630$ & ${-0.0499}$ & $2.6$ \\
$L{=}1$ & $0.4491$ & $0.4254$ & ${-0.0237}$ & $1.2$ \\
$L{=}3$ & $0.4126$ & $0.4505$ & ${+0.0379}$ & $2.0$ \\
$L{=}6$ & $0.4363$ & $0.4592$ & ${+0.0229}$ & $1.2$ \\
$L{=}9$ & $0.4140$ & $0.4498$ & ${+0.0358}$ & $1.9$ \\
$L{=}18$ & $0.4257$ & $0.4823$ & ${+0.0566}$ & $3.0$ \\
$L{=}27$ & $0.4071$ & $0.5650$ & $\mathbf{+0.1579}$ & $8.3$ \\
\purerand\ \emph{(pure batches, random order)} & $0.4253$ & $0.5851$ & $\mathbf{+0.1598}$ & $8.4$ \\
\midrule
$L{=}54$ & $0.4922$ & $0.8164$ & $\mathbf{+0.3242}$ & $17.0$ \\
\blocked\ \emph{(one source, then the other)} & $0.8692$ & $0.8501$ & ${-0.0191}$ & $1.0$ \\
\bottomrule
\end{tabular}

\end{table}

\noindent Three readings, the third of which changes what this section claims.

\emph{The interior is not flat; it was flat to this instrument, under this schedule.} The
registered read-out returned branch \textbf{CL-2}. The interior spans $11.63$ floors against a
minimum detectable difference of $3.72$, so it is resolved rather than merely wide, and its
shape is not noise but the shape Eq.~\eqref{eq:hatchbound} predicts: ordering the nine non-corner
arms by how blocked the arrangement is (\shuf, then $L{=}1,3,6,9,18,27$, then $L{=}54$), the
allocation is monotone in $\tau_L$ up to $2$ rank inversions out of $36$ pairs. \emph{The averaging wall as a claim about the path is withdrawn.}

\emph{What replaces it is a claim about the schedule, and it is the stronger statement.}
Proposition~\ref{prop:avg} is not refuted; it is located. Its conclusion holds in the limit
where the tail of training carries no weight, and a decaying schedule is what puts a run in that
limit. The cosine is not a neutral background against which the path was measured. \emph{It is
the averaging operator}, integrating the arrangement away, and the constant-rate family is the
same corpus with the operator removed. Read that way the two families are not a result and its
correction but a measurement of one knob at two settings, and this paper's own schedule control
(\S\ref{sec:threats}), which moves the allocation $0.116$ by splitting one cosine into two
while holding the path byte for byte, is the same knob at a third.

\emph{The corner is schedule-free, which is why the paper's central result does not move.}
\blocked\ differs by $1.0$ floor between the two families against $17.0$ at $L{=}54$, which is
Proposition~\ref{prop:sched}'s corner case measured: at $\tau_L=T$ the displacement is set by the
whole profile of $w$ rather than by any part of it a schedule can delete. Every rung
below has a tail the discount can act on, and every one of them moves. The arrangement
switch of \S\ref{sec:conserve}, the antisymmetry of \S\ref{sec:phase} and the key of
\S\ref{sec:key} are all read at or against the corner and are untouched. What the constant-rate
family costs this paper is one sentence about the interior; what it buys is the mechanism that
sentence was standing in for.

\emph{One thing the constant family does less well, stated because it bears on reading the
table.} Capability spans $0.0680$ across its ten arms against $0.0246$ across the cosine ten on
the same three seeds, which is $2.5$ capability floors rather than inside one. The share is a ratio and is not
mechanically driven by that spread, but the constant family is a noisier place to read a
fixed-capability claim, and \S\ref{sec:conserve}'s conservation result is quoted from the cosine
family throughout for that reason.

\paragraph{Flatness is a null, so we test it as one.} A span is a description, not a test.
Three questions have to be answered in order: what can this design see, what does a proper test
of the null say, and can a difference in the interior be attributed to allocation at all. The
answers point the same way and they are worth stating separately, because together they say
something sharper than any one of them.

\emph{First, what the design can see.} Two arms of three seeds give four degrees of freedom, so
at $\alpha=0.05$ two-sided the smallest difference detectable with $80\%$ power is $3.72$
contrast floors at the inherited $\hat\sigma=0.0234$ and $5.01$ at each ladder's own recomputed
$0.0315$; at $50\%$ power the same figures are $2.78$ and $3.74$. The ladder's interior spans
$\mathbf{2.20}$. \emph{The interior span is smaller than anything a pairwise comparison in this
design could have resolved}, at either dispersion and at either power. That is a fact about the
instrument and it disposes of any reading of the interior's shape.

\emph{And it is a fact about the instrument at $k{=}4$, which is a weaker sentence than the one
we first wrote.} Part of the dispersion the paragraph above rests on is generation sampling
rather than training stochasticity, and that part falls as $1/\sqrt{k}$ for inference time and no
training. \emph{How large a part is the whole question, and a borrowed number cannot answer it.} The
recipe-level paper measures $0.0132$ of a
$0.0142$ cross-seed dispersion to be generation sampling, $86.4\%$ of the variance, on a
different corpus, a different skill set and a different scale. Nothing in this paper measured
whether that transfers. It does not.

\begin{center}\small
\begin{tabular}{lcccccc}
\toprule
samples per problem & $k{=}4$ \emph{(as run)} & $k{=}8$ & $k{=}16$ & $k{=}32$ & $k{=}64$ & $k\to\infty$ \\
\midrule
\multicolumn{7}{@{}l}{\emph{evaluation share $0.864$, borrowed from another corpus --- as previously published}} \\
$\mathrm{MDE}_{80}$ in floors & $3.72$ & $2.80$ & $2.21$ & $1.84$ & $1.62$ & $1.37$ \\
resolves the $2.20$-floor interior? & no & no & no & yes & yes & yes \\
\addlinespace[2pt]
\multicolumn{7}{@{}l}{\emph{evaluation share $\mathbf{0.6074}$, measured in place on this paper's own retained family}} \\
$\mathrm{MDE}_{80}$ in floors & $3.72$ & $3.10$ & $2.74$ & $\mathbf{2.55}$ & $\mathbf{2.44}$ & $\mathbf{2.33}$ \\
resolves the $2.20$-floor interior? & no & no & no & \textbf{no} & \textbf{no} & \textbf{no} \\
\bottomrule
\end{tabular}
\end{center}

\noindent \textbf{The conclusion reverses, and the row that reverses it is the one the published
table did not have.} Measured on this paper's own material by the identical closed form, the
evaluation share is $0.6074$ rather than $0.864$, and at that share \emph{no evaluation budget
resolves the interior}: the limit of infinite samples per problem still leaves an
$\mathrm{MDE}_{80}$ of $2.33$ floors against a $2.20$-floor span. So a span of this size does not
become resolvable at $k{=}32$, and this resolution is not something a reader can buy at the
inference rate at all. The remaining gap is
training noise. It can only be bought in seeds, which is the expensive axis, and this paper does
not price it.

\emph{A second instrument says something worse than disagreement.} The retained constant-rate
family was re-scored at $k{=}16$ over all eleven arms and three seeds: if $86\%$ of the variance
were evaluation sampling, the pooled cross-seed dispersion should have fallen from $0.0291$ to
$0.0173$. It reads $0.0314$, a ratio of $\mathbf{1.079}$ where the model predicts $0.594$, and the
additive model $\hat\sigma(k)^2=\sigma_{\text{train}}^2+\sigma_{\text{eval}}^2\cdot(4/k)$ can
represent that only with a \emph{negative} evaluation component, $-0.219$. \emph{The two
instruments agree that the borrowed number does not transfer and disagree about what replaces it},
one saying $0.61$ and the other putting the quantity outside the model's domain. We take the more
conservative reading: even the generous $0.6074$ does not resolve the interior.

\emph{The two instruments stopped disagreeing when the companion audited the harness, and the
resolution is a mechanism} \citep{frl2026}. None of this project's evaluation drivers passes a
seed to the inference engine, so the engine's default takes effect and every run consumes the same
sampling stream on the same prompts. Evaluation sampling is therefore largely \emph{common-mode}
across seeds: present in full in an absolute accuracy, largely cancelled in a cross-seed
dispersion. That predicts what the second instrument measured, since raising $k$ should then leave
the cross-seed dispersion roughly unchanged, and $1.079$ is roughly unchanged. It also says the
first instrument's $0.6074$ is not a share of the measured dispersion but an upper bound on what
an evaluation drawing independently per run would face. Both rows survive with their roles named:
the measured-share row prices a reader's independent re-run and the $k{=}16$ instrument prices
this harness, on which no evaluation budget buys the interior because the term $k$ would buy down
is largely not in the dispersion to begin with.

There is a second way to see that the $2.20$-floor span was an instrument statement rather than a
fact about arrangement, and it is untouched by any of this: the constant-rate ladder spans $11.63$
floors at the same $k{=}4$, comfortably above the $3.72$ this design can resolve, so the same
instrument that could not see the cosine interior sees the constant one without difficulty. The
load-bearing claim, that the cosine interior is unresolvable at the evaluation this project ran,
never depended on the exchange rate and is unaffected. It does not touch the corner, which stands
at $21.99$ floors.

We cannot simply run the published arms at a larger $k$. Their checkpoints were reclaimed after
evaluation (\S\ref{sec:design}), so those rows cannot be re-scored at any $k$; the constant-rate
family is fresh and is what both measurements above are made on. Sources:
\texttt{dpd\_power\_vs\_k}, \texttt{evalvar\_direct}, \texttt{highk\_verdict}.

\emph{Second, the null tested properly.} Paired TOST against the interleaved arm, at an
equivalence bound of $2\hat\sigma=0.0468$ that Definition~\ref{def:channel} froze rather than
this paragraph chose, certifies one interior arm of seven. Given the previous point that is
what it must do: the $90\%$ intervals are two to four times wider than the bound. Six arms are
neither shown equivalent nor shown different.

\emph{Third, whether an interior difference would even be allocation.} The difficulty coupling
of Definition~\ref{def:decomp} is present on all eight interior arms with the same sign, mean
$0.0637$, and it \emph{varies} across them by $0.0817$, which is $4.28$ floors. Its level
cancels in a between-arm difference; its variation does not. So the residual coupling alone
spans nearly twice the interior, and a difference of the interior's size could not be
attributed to a change of convention policy even if the design could resolve it.

\emph{What survives, and it is one statistic rather than any comparison.} Treat the eight
interior arms as groups and their seeds as replicates: the between-arm intraclass correlation of
the allocation share is $-0.076$, with a bootstrap $95\%$ interval over arms of
$[-0.409,+0.166]$ and a permutation $p=0.61$ against the null that the rung label carries no
information. The interval covers zero, so we claim what it supports and no more:
\emph{between-arm dispersion does not exceed within-arm seed noise}. It does not license the
stronger sentence that arm identity explains none of the variance, and an earlier abstract of
ours said exactly that.

The corner is the one part of this figure that no power argument touches. \blocked\ sits
$21.99$ floors above the top of the interior, five times the largest of the three quantities
above and an order of magnitude above the interior span. \emph{The averaging wall as this paper
can defend it is therefore one claim about a pooled statistic and one about a corner, and not a
claim about the shape of the interior.} Source: \texttt{review\_statistics}.

The one substantive difference is where the amplitude starts. On Qwen2.5 $L{=}54$ is the first rung
off the floor, at $0.4922$ against a flat region of $0.407$--$0.449$; on Qwen3 it reads $0.4547$ and
sits \emph{inside} a flat region of $0.424$--$0.479$. The blocked end is undiminished
($0.8907$ against $0.8692$). The stopping phase is therefore present on both families and its onset
is further along the ladder on the second, which is a statement about where the amplitude turns on
rather than about whether it exists.

\paragraph{The mixture ratio is the variable (the deciding test, frozen).} If the share is
the time-average of the path, it must track the mixture weights and nothing else. The
pre-registered form of this test ($3{:}1$) was excluded by its own guard. Below ${\sim}500$
rows the minority convention is not learnable at all ($0.2108$ at $720$ rows, $0.0546$ at
$468$, $0.0000$ at $234$), so its share would measure a learnability floor, not allocation, and was redesigned at $2{:}1$ with the read-out rules re-frozen before the runs (T0
guard; T1/T2/T3 outcomes). Table~\ref{tab:ratio} gives the result: the shuffled arm's share
moves from $0.453$ at parity to $0.312$ at $2{:}1$, against mean-field predictions $0.500$
and $0.333$, both misses within the frozen $0.05$ tolerance, both on the same side, the
side of the pretrained preference for numerals. Read-out T1: \emph{allocation tracks the
mixture weights}. The parity miss ($0.047$) sits $0.003$ inside the tolerance, and is reported
rather than rounding the margin up. Table~\ref{tab:ratiopred} puts this beside the three
interventions that were predicted to do nothing and did nothing, which is the shape of the
claim: the mean field is weighted by source \emph{proportions} and by nothing else the
path can vary, so exactly one of four levers is allowed to move the share, and exactly one
does.

\begin{table}[t]
\centering\small
\caption{\textbf{The share tracks the mixture weights; the blocked arm does not}
(natural-conflict ratio corpora, $936$ rows, three seeds; read-out rules frozen in the job script the scheduler executed: T1). The $3{:}1$ form
was excluded by its own learnability guard and redesigned at $2{:}1$; the redesign is
labelled, not laundered.}
\label{tab:ratio}
\begin{tabular}{lcccc}
\toprule
 & \multicolumn{2}{c}{$1{:}1$ (rr11)} & \multicolumn{2}{c}{$2{:}1$ (rr21)} \\
\cmidrule(lr){2-3}\cmidrule(lr){4-5}
 & share & capability & share & capability \\
\midrule
mean-field prediction & $0.500$ & --- & $0.333$ & --- \\
\shuf\ (measured) & $0.453$ & $0.372$ & $0.312$ & $0.377$ \\
\blocked\ $\Astep$ & $0.731$ & $0.380$ & $0.699$ & $0.381$ \\
$B$ \only\ (guard: $\mathrm{acc}_B=0.263/0.187>$ floor) & $0.713$ & $0.368$ & $0.595$ & $0.315$ \\
\bottomrule
\end{tabular}

\end{table}

\begin{table}[H]
\centering\small
\caption{\textbf{The one lever the averaging limit leaves free.} Proposition~\ref{prop:avg}'s
mean field is weighted by source \emph{proportions}, so changing the mixture ratio must
move the allocation and changing the path must not. Both halves are measured here on the same
corpus family. Tolerance $0.05$ was frozen with the prediction; both misses fall on the side of
the pretrained prior, which the family measures independently. (T1.)}
\label{tab:ratiopred}
\begin{tabular}{lccc c}
\toprule
& \multicolumn{3}{c}{allocation share of $B$} & \\
\cmidrule(lr){2-4}
intervention & mean-field prediction & measured & miss & moves the share? \\
\midrule
mixture ratio $1{:}1 \to 2{:}1$ & $0.500 \to 0.333$ & $0.453 \to 0.312$ & $0.047,\;0.021$ & \textbf{yes, as predicted} \\
\midrule
block length $L{=}1 \to 54$ & \multicolumn{2}{c}{no change predicted} & --- & no ($0.407$--$0.492$) \\
batch purity, fixed blocks & \multicolumn{2}{c}{no change predicted} & --- & no ($3\%$ of the span) \\
palindromic reordering & \multicolumn{2}{c}{no change predicted} & --- & no (and $4\times$ worse) \\
\bottomrule
\end{tabular}

\end{table}

The blocked arm is the counterpoint the mean-field reading needs: its share sits at
$0.731/0.699$ \emph{regardless of the ratio}. A mixture optimum moves with the weights; a
corner state does not. What kind of state it is, the next section measures.

\paragraph{The designed schedule does not escape either (pre-registered, dead, and the sharpest
of the three).} The natural objection to everything above is that we rearranged the path naively.
Training on $A$ then $B$ is, exactly and not by analogy, a \emph{Lie--Trotter splitting}
\citep{trotter1959product} of $\exp(A{+}B)$, which numerical analysis and NMR have studied for
decades \citep{mclachlan2002splitting,hairer2006geometric}, and that literature does not merely
rank schedules but constructs better ones. The canonical construction is palindromic
\citep{strang1968construction}: run $A$ for $T/2$, $B$ for $T$, $A$ for $T/2$. Because the
even-order Magnus terms vanish for a time-symmetric sequence \citep{blanes2009magnus}, a Strang
schedule is \emph{second}-order accurate where blocked is first-order, so a practitioner forced to
train in blocks would recover most of what interleaving buys at identical data and token budget.
The curriculum literature does not propose it, because it is trying to \emph{choose} an order
rather than to \emph{design} a schedule. The optimisation literature has since arrived at the same
construction: \citet{nguyen2026shuffle} prove that a \emph{paired reversal}, symmetrising the
epoch map exactly as a palindrome does, cancels the leading order-dependent second-order term and
takes order sensitivity from quadratic to cubic in the step size. That is the theorem our
registration was betting on, stated more sharply than we stated it.

We registered it in the linear model, where every arm is computable exactly and the joint arm
both schedules approximate is available in closed form. Three checks, thresholds frozen first:
S1, the median $|D_{\text{strang}}|/|D_{\text{blocked}}|$ below $0.5$; S2, log--log slopes in
the step size differing by at least $0.5$, so the gain is an \emph{order} improvement rather
than a constant; S3, the order effect $\mathrm{ASYM}$ identically zero for a palindromic schedule, which is its own
reverse, an implementation guard, not a claim.

\begin{table}[H]
\centering\small
\setlength{\tabcolsep}{4pt}
\caption{\textbf{The construction splitting theory recommends, and what it did.} Registered in
the linear model, where every arm is computable exactly. S3 is an implementation guard, not a
claim: a palindromic schedule is its own reverse, so its order effect must vanish, and it does.
S1 and S2 are the claim and both fail: S1 by a factor of seven against its own bar, and S2
\emph{degenerately}, since refining the step by $16\times$ moves either penalty in the fifth
decimal. Source \texttt{splitting\_verdict}.}
\label{tab:strang}
\begin{tabular}{@{}l>{\raggedright\arraybackslash}p{1.9cm}>{\raggedright\arraybackslash}p{6.5cm}l@{}}
\toprule
check & registered bar & measured & \\
\midrule
\textbf{S3} order effect of a palindrome
 & $\mathrm{ASYM}\equiv0$
 & $0$ to machine precision
 & \textbf{passes} \\
\midrule
\textbf{S1} median $|D_{\text{strang}}|/|D_{\text{blocked}}|$
 & $<0.5$
 & $\mathbf{3.65}$ over $300$ draws (IQR $1.88$--$8.09$); $3.7\%$ of draws below the bar
 & \textbf{fails} \\
\textbf{S2} log--log slopes in $\eta$ differ
 & by $\ge0.5$
 & both slopes $0.0$ to four figures across $\eta=0.4\ldots0.025$
 & \textbf{fails} \\
\midrule
\multicolumn{4}{@{}l@{}}{\footnotesize \emph{the penalty across a $16\times$ refinement of the step, which is what ``no order to improve'' means}}\\
\multicolumn{2}{@{}l}{$\eta = 0.4,\,0.2,\,0.1,\,0.05,\,0.025$} & \multicolumn{2}{l}{} \\
\multicolumn{2}{@{}l}{\quad blocked} & \multicolumn{2}{@{}l}{$0.081646,\ 0.081638,\ 0.081634,\ 0.081632,\ 0.081631$} \\
\multicolumn{2}{@{}l}{\quad palindromic} & \multicolumn{2}{@{}l}{$0.32637,\ 0.32634,\ 0.32631,\ 0.32630,\ 0.32629$} \\
\bottomrule
\end{tabular}

\end{table}

\noindent Table~\ref{tab:strang} is the result. The palindromic schedule is roughly \emph{four
times worse} than the blocked one it was constructed to beat, and the refinement sweep says why no
choice of threshold would have rescued it (Figure~\ref{fig:splitlr}).

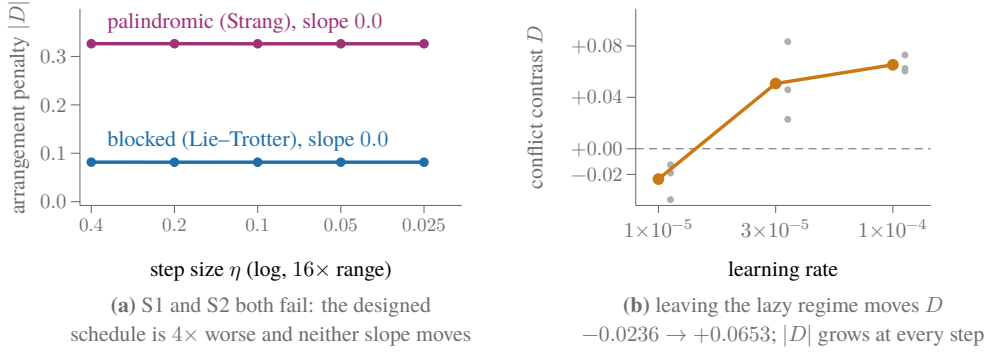
\begin{figure}[t]
\centering
\begin{tikzpicture}[halo/.style={fill=white,inner sep=1pt},font=\small,>=Stealth]
\draw[cGrey,thin] (0.75,0.3) -- (5.8,0.3);
\draw[cGrey,thin] (0.75,0.3) -- (0.75,2.9);
\draw[cGrey,thin] (0.63,0.350) -- (0.75,0.350) node[left,xshift=-2pt,scale=0.78] {$0.0$};
\draw[cGrey,thin] (0.63,0.990) -- (0.75,0.990) node[left,xshift=-2pt,scale=0.78] {$0.1$};
\draw[cGrey,thin] (0.63,1.630) -- (0.75,1.630) node[left,xshift=-2pt,scale=0.78] {$0.2$};
\draw[cGrey,thin] (0.63,2.270) -- (0.75,2.270) node[left,xshift=-2pt,scale=0.78] {$0.3$};
\draw[cD,line width=1.3pt] plot coordinates {(0.900,2.439) (2.000,2.439) (3.100,2.438) (4.200,2.438) (5.300,2.438)};
\fill[cD] (0.900,2.439) circle (1.8pt);
\fill[cD] (2.000,2.439) circle (1.8pt);
\fill[cD] (3.100,2.438) circle (1.8pt);
\fill[cD] (4.200,2.438) circle (1.8pt);
\fill[cD] (5.300,2.438) circle (1.8pt);
\draw[cV,line width=1.3pt] plot coordinates {(0.900,0.873) (2.000,0.872) (3.100,0.872) (4.200,0.872) (5.300,0.872)};
\fill[cV] (0.900,0.873) circle (1.8pt);
\fill[cV] (2.000,0.872) circle (1.8pt);
\fill[cV] (3.100,0.872) circle (1.8pt);
\fill[cV] (4.200,0.872) circle (1.8pt);
\fill[cV] (5.300,0.872) circle (1.8pt);
\node[cD,anchor=west,scale=0.83] at (1.0,2.719) {palindromic (Strang), slope $0.0$};
\node[cV,anchor=west,scale=0.83] at (1.0,1.153) {blocked (Lie--Trotter), slope $0.0$};
\draw[cGrey,thin] (0.900,0.2) -- (0.900,0.3) node[below,yshift=-2pt,scale=0.72] {$0.4$};
\draw[cGrey,thin] (2.000,0.2) -- (2.000,0.3) node[below,yshift=-2pt,scale=0.72] {$0.2$};
\draw[cGrey,thin] (3.100,0.2) -- (3.100,0.3) node[below,yshift=-2pt,scale=0.72] {$0.1$};
\draw[cGrey,thin] (4.200,0.2) -- (4.200,0.3) node[below,yshift=-2pt,scale=0.72] {$0.05$};
\draw[cGrey,thin] (5.300,0.2) -- (5.300,0.3) node[below,yshift=-2pt,scale=0.72] {$0.025$};
\node[anchor=north,scale=0.82] at (3.275,-0.32) {step size $\eta$ (log, $16\times$ range)};
\node[anchor=south,scale=0.82,rotate=90,cGrey] at (0.20,1.6) {arrangement penalty $|D|$};
\node[anchor=north,scale=0.82,cGrey,align=center] at (3.275,-0.80) {\textbf{(a)} S1 and S2 both fail: the designed\\ schedule is $4\times$ worse and neither slope moves};
\begin{scope}[xshift=7.5cm]
\draw[cGrey,densely dashed] (0.6,1.05) -- (4.5,1.05);
\draw[cGrey,thin] (0.6,0.30) -- (4.5,0.30);
\draw[cGrey,thin] (0.6,0.30) -- (0.6,2.9);
\draw[cGrey,thin] (0.48,0.710) -- (0.6,0.710) node[left,xshift=-2pt,scale=0.76] {$-0.02$};
\draw[cGrey,thin] (0.48,1.050) -- (0.6,1.050) node[left,xshift=-2pt,scale=0.76] {$+0.00$};
\draw[cGrey,thin] (0.48,1.730) -- (0.6,1.730) node[left,xshift=-2pt,scale=0.76] {$+0.04$};
\draw[cGrey,thin] (0.48,2.410) -- (0.6,2.410) node[left,xshift=-2pt,scale=0.76] {$+0.08$};
\fill[cGrey!55] (1.060,0.838) circle (1.3pt);
\fill[cGrey!55] (1.060,0.730) circle (1.3pt);
\fill[cGrey!55] (1.060,0.377) circle (1.3pt);
\fill[cGrey!55] (2.610,1.439) circle (1.3pt);
\fill[cGrey!55] (2.610,2.466) circle (1.3pt);
\fill[cGrey!55] (2.610,1.830) circle (1.3pt);
\fill[cGrey!55] (4.160,2.289) circle (1.3pt);
\fill[cGrey!55] (4.160,2.077) circle (1.3pt);
\fill[cGrey!55] (4.160,2.112) circle (1.3pt);
\draw[cA,line width=1.3pt] plot coordinates {(0.900,0.649) (2.450,1.912) (4.000,2.160)};
\fill[cA] (0.900,0.649) circle (2.2pt);
\fill[cA] (2.450,1.912) circle (2.2pt);
\fill[cA] (4.000,2.160) circle (2.2pt);
\draw[cGrey,thin] (0.900,0.2) -- (0.900,0.3) node[below,yshift=-2pt,scale=0.80] {$1{\times}10^{-5}$};
\draw[cGrey,thin] (2.450,0.2) -- (2.450,0.3) node[below,yshift=-2pt,scale=0.80] {$3{\times}10^{-5}$};
\draw[cGrey,thin] (4.000,0.2) -- (4.000,0.3) node[below,yshift=-2pt,scale=0.80] {$1{\times}10^{-4}$};
\node[anchor=north,scale=0.82] at (2.55,-0.32) {learning rate};
\node[anchor=north,scale=0.82,cGrey,align=center] at (2.55,-0.80) {\textbf{(b)} leaving the lazy regime moves $D$\\ $-0.0236\to+0.0653$; $|D|$ grows at every step};
\node[anchor=south,scale=0.82,rotate=90,cGrey] at (-0.52,1.60) {conflict contrast $D$};
\end{scope}
\end{tikzpicture}
\caption{\textbf{Two ways of leaving the averaging wall, one designed and one accidental.}
\textbf{(a)} The palindromic schedule splitting theory recommends, against the blocked one it was
built to beat, in the linear model where the joint arm is available in closed form. It is roughly
four times \emph{worse}, and refining the step by $16\times$ moves either penalty in the fifth
decimal: there is no order to improve because the penalty is not a discretisation error.
\textbf{(b)} The one intervention that does move the contrast is leaving the lazy regime. Raising
the learning rate grows $|D|$ monotonically, $0.0236\to0.0653$, through a change of sign; the
highest rate also costs capability, $0.268$ against $0.340$, so it is the system leaving the
regime the averaging argument is tight in.}
\label{fig:splitlr}
\end{figure}

The second failure is the informative one, and it is this section's claim arriving from the other
side. Splitting theory's guarantees are asymptotic in the step size, but refining $\eta$ by a
factor of $16$ moves the penalty in the fifth decimal. \emph{The penalty is not a discretisation
error at all}: it is set by the coarse structure of the arrangement and is blind to how finely
that structure is resolved, which is what Eq.~\eqref{eq:meanfield} says and why no scheme designed
against the discretisation can help. The failure can therefore be located at an order.
\citet{nguyen2026shuffle}'s guarantee is that symmetrisation deletes the $O(\gamma^2)$
order-dependent term; our sweep says the conflict penalty is not in that term, because a quantity
living at $O(\gamma^2)$ or $O(\gamma^3)$ cannot be invariant under a $16\times$ refinement of
$\gamma$. In the language of Proposition~\ref{prop:sched}, a palindrome is still \emph{one
period}: it symmetrises the block without shortening $\tau_L$, so it lands where
Eq.~\eqref{eq:hatchbound} buys nothing, next to \blocked\ rather than next to the interior. The
block-length sweep found the same wall by measurement; this finds it in a model where the answer
can be computed, and it kills the best-motivated escape we could construct rather than the naive
one.

We report it as a failure because it was registered as a prediction. It would have been the
paper's one piece of practical advice.

One qualification about that registration. The others in this paper are frozen in a job script
committed hours to days before the run it launched, on the public history the registration index
tabulates. This check is a linear-model computation of a few seconds whose thresholds and answer
entered the repository in the same commit, so no earlier artefact carries the prediction. S1--S3
are stated as supported-if conditions rather than as a description of what happened, and the
misses are wide enough ($3.65$ against $0.5$; both slopes zero) that no threshold choice rescues
them. But the ordering the other tests demonstrate, this one can only assert.

\section{The Two Non-Average Terms: the Stopping Phase and the Key}\label{sec:phase}

\begin{table}[t]
\centering
\small
\begin{tabular}{lccc}
\toprule
axis swept & on the path? & share, one end to the other & contrast floors \\
\midrule
stopping phase, \shuf\ against \blocked & yes & $0.4129 \to 0.8692$ & $23.89$ \\
model scale, 3B to 14B & no & $0.1322 \to 0.4732$ & $17.85$ \\
mixture ratio, 1:1 to 2:1 & no & $0.4531 \to 0.3116$ & $7.41$ \\
cosine restart, byte-identical path & no & $0.4071 \to 0.2912$ & $6.07$ \\
block length, $L{=}1$ to $L{=}54$ & yes & $0.4071 \to 0.4922$ & $4.46$ \\
batch purity at fixed block structure & yes & $0.4129 \to 0.4253$ & $0.65$ \\
\bottomrule
\end{tabular}

\caption{\textbf{Every axis this project has swept, on one scale, sorted by how far it moves the
allocation.} All six are measured on the allocation share of a mixed arm at one model and one
budget, so the column is commensurable; the learning-rate sweep is reported on the contrast $D$
rather than on the share and is deliberately absent rather than forced onto this axis. The
stopping phase is the largest term measured anywhere in the project, which is why it gets a
section. The block-length ladder ranks fifth of six, and all three non-path axes move the
allocation further than the whole ladder does. Source: \texttt{axes}.}
\label{tab:axes}
\end{table}

Two of this paper's terms are not averages of the corpus, and Table~\ref{tab:axes} says which
and how large. It puts every axis the project has swept on the one scale where they are
comparable, and it is sorted rather than arranged to flatter the headline. The stopping phase
is the largest term measured anywhere here, $23.89$ contrast floors, which is the case for
spending a section on it. The same table states the scope limit plainly: the block-length
ladder, which is the paper's most refined path instrument, ranks fifth of six at $4.46$ floors,
and all three axes that are \emph{not} properties of the path move the allocation further than
the whole ladder does. The path is a narrow channel and this paper's claims are about the path,
but it is neither the narrowest thing measured here nor the only input to an endpoint, and a
reader who wants to predict an endpoint from the path alone should read this table before the
rest of the section. Figure~\ref{fig:mirror} draws the stopping phase as a position on the
alternation cycle.

\begin{figure}[t]
\centering
\begin{tikzpicture}[halo/.style={fill=white,inner sep=1pt},font=\small,>=Stealth]
\node[anchor=west,scale=0.86,cGrey] at (-0.4,3.95) {\textbf{(a)} the picture, post hoc: a limit cycle};
\draw[cGrey,densely dashed] (0.10,1.35) -- (4.9,1.35);
\node[anchor=west,scale=0.78,cGrey] at (4.15,1.57) {mean-field point};
\draw[cV,line width=1.1pt] (2.3,1.35) ellipse (1.7 and 0.85);
\draw[cV,->,line width=1.1pt] (1.300,2.037) -- (1.232,2.011);
\fill[cA] (2.3,0.50) circle (2.6pt);
\node[cA!80!black,anchor=north,scale=0.8,align=center] at (2.3,0.36) {stop after a $B$ block: share on the $B$ side};
\fill[cD] (2.3,2.2) circle (2.6pt);
\node[cD!85!black,anchor=south,scale=0.8,align=center] at (2.3,2.32) {rebuild to end on $A$:\\ same cycle, other side};
\node[anchor=west,scale=0.8,cGrey,align=left] at (0.10,-0.52) {amplitude grows with $L$; an endpoint measures\\ \emph{where the path stopped on the cycle}};
\begin{scope}[xshift=7.6cm]
\node[anchor=west,scale=0.86,cGrey] at (-0.55,3.95) {\textbf{(b)} the frozen prediction, measured (M1)};
\draw[cGrey,thin] (0.15,0) -- (6.1,0);
\draw[cGrey,thin] (0.15,0) -- (0.15,3.60);
\draw[cGrey,thin] (0.03,0.450) -- (0.15,0.450) node[left,xshift=-2pt,scale=0.76] {$0.2$};
\draw[cGrey,thin] (0.03,1.350) -- (0.15,1.350) node[left,xshift=-2pt,scale=0.76] {$0.3$};
\draw[cGrey,thin] (0.03,2.250) -- (0.15,2.250) node[left,xshift=-2pt,scale=0.76] {$0.4$};
\draw[cGrey,thin] (0.03,3.150) -- (0.15,3.150) node[left,xshift=-2pt,scale=0.76] {$0.5$};
\fill[cA!80] (0.55,0) rectangle (0.89,1.103);
\fill[cV!80] (0.94,0) rectangle (1.28,1.031);
\fill[cA!35] (1.58,0) rectangle (1.92,1.867);
\fill[cV!35] (1.97,0) rectangle (2.31,0.244);
\draw[cD,line width=1.1pt] (0.49,3.484) -- (1.34,3.484);
\draw[cD,line width=1.1pt] (1.52,3.461) -- (2.37,3.461);
\node[anchor=north,scale=0.78,align=center] at (0.915,-0.06) {ends on\\ $B$};
\node[anchor=north,scale=0.78,align=center] at (1.945,-0.06) {ends on\\ $A$};
\node[anchor=north,scale=0.84] at (1.43,-0.78) {$L=54$};
\fill[cA!80] (3.45,0) rectangle (3.79,1.515);
\fill[cV!80] (3.84,0) rectangle (4.18,0.615);
\fill[cA!35] (4.48,0) rectangle (4.82,1.436);
\fill[cV!35] (4.87,0) rectangle (5.21,0.713);
\draw[cD,line width=1.1pt] (3.39,3.480) -- (4.24,3.480);
\draw[cD,line width=1.1pt] (4.42,3.498) -- (5.27,3.498);
\node[anchor=north,scale=0.78,align=center] at (3.815,-0.06) {ends on\\ $B$};
\node[anchor=north,scale=0.78,align=center] at (4.845,-0.06) {ends on\\ $A$};
\node[anchor=north,scale=0.84] at (4.33,-0.78) {$L=27$};
\draw[<->,cGrey,line width=0.8pt] (1.0,1.163) -- (1.98,1.927);
\node[anchor=south,scale=0.86,cGrey,align=center,halo] at (1.6,2.28) {$+0.0850$ / $-0.0875$\\capability $\Delta 0.0025$};
\node[anchor=west,scale=0.80,cD] at (4.30,2.58) {\rule[2.4pt]{6.5pt}{1.2pt}\ capability $C$};
\node[anchor=west,scale=0.80,cA!80!black] at (4.30,2.15) {\rule{6.5pt}{6.5pt}\ $\mathrm{acc}_A$};
\node[anchor=west,scale=0.80,cV!80!black] at (5.25,2.15) {\rule{6.5pt}{6.5pt}\ $\mathrm{acc}_B$};
\end{scope}
\end{tikzpicture}
\caption{\textbf{The stopping phase is a position on a cycle, and the mirror moves it to the other
side.} \textbf{(a)} The reading of Proposition~\ref{prop:phase}, labelled post hoc: a driven
system oscillates about the mean-field point with amplitude growing in $L$, and an endpoint
measures where the path stopped. It makes one frozen prediction. \textbf{(b)} The measurement
(read-out M1): at $L{=}54$ the shift is antisymmetric, $+0.0850$ against $-0.0875$, with
capability moving $0.0025$. At $L{=}27$ both constructions sit at the floor and the mirror
manufactures nothing, as it must.}
\label{fig:mirror}
\end{figure}

Treat $A^LB^L\cdots$ as a periodically driven system: each $A$-block pulls the mixture
weight toward $A$, each $B$-block pulls it back, and at equal weights the steady state is a
\emph{limit cycle around the mean-field point}, with amplitude growing in $L$; what an
endpoint measures is \emph{where the path stopped on the cycle}. This picture was formed
post hoc (we label it so) and it makes two pre-registered predictions that then ran.

First, amplitude: shares $0.407$--$0.449$ for $L\le27$ (amplitude ${\approx}0$, all within
noise of $\shuf$'s $0.413$), $0.492$ at
$L{=}54$, and $0.869$ for \blocked\ at three seeds ($0.856$ at eight), which on this reading is not a different species but
\emph{the lowest-frequency alternation there is}, half a period that never swings back,
stopped at maximum displacement. Second, the mirror: our published sweep always ends on a
$B$-block, so every measured share sits on the $B$-side of the cycle; rebuilding $L{=}54$ to
end on $A$ must relocate the same magnitude to the other side while moving capability not at
all. It does, antisymmetrically: $\mathrm{acc}_A$ shifts $+0.0850$ while $\mathrm{acc}_B$
shifts $-0.0875$, and capability moves $0.0025$, an order of magnitude below either shift
(read-out M1, \texttt{dpd\_mirror\_verdict}; at $L{=}27$ both constructions sit at the floor
and the mirror manufactures nothing, as it must, the published arm anchoring at
$\mathrm{acc}_A=0.3183$ against the mirror's $0.3096$). The residence time
$\tau_{\text{write}}=13.6$ steps that v2 measured as an ``installation time'' is, on this
reading, the moment the mixture weight crosses $50\%$: an allocation quantity, not a
capacity one.

\paragraph{The mirror on the corpus where neither convention is wrong.} Everything above is
\texttt{cf2}, whose second convention is a shifted answer. The construction this paper's claim is
actually about is \texttt{nat}, a numeral against the same number spelled out, and a second
ladder was trained on it at a third of the learning rate. Its $L{=}45$ rung carries a mirror, and
the prediction is the same one: opposite signs on the two accuracies, capability unmoved.

It holds. Ending the alternation on $A$ instead of $B$ moves $\mathrm{acc}_A$ by $+0.0183$ and
$\mathrm{acc}_B$ by $-0.0229$, opposite signs, while capability moves $-0.0046$, which is $0.17$
capability floors. \emph{Read the size before the pattern.} Each shift is about one contrast
floor, against $4.5$ on \texttt{cf2}, so this is an observation consistent with the published
mirror rather than an independent confirmation of it. The smaller size is what
Appendix~\ref{app:budget} requires: $|D|$ grows with the step size, and this ladder runs at
lr $1\times10^{-5}$. What it does establish is that the sign and the symmetry are not artefacts
of a corpus in which one convention is arguably wrong. Source: \texttt{nat\_replication}.

\paragraph{The mirror on the second family manufactures nothing, and that is the prediction.}
Rebuilding the $L{=}54$ alternation to end on $A$ was registered for Qwen3-8B-Base alongside the
ladder, and it returns $\Delta\mathrm{acc}_A=+0.0204$ against $\Delta\mathrm{acc}_B=+0.0012$: not
antisymmetric, read-out \textbf{M2}. Taken alone that reads as the limit-cycle picture failing to
travel. Taken with the rung it was built on, it is what the picture requires. The mirror can only
relocate an amplitude that exists, and on this family $L{=}54$ sits inside the flat region rather
than above it (Table~\ref{tab:family2}), so there is nothing at that block length to move, which
is exactly what we report at $L{=}27$ on Qwen2.5, where both constructions sit at the floor and the
mirror manufactures nothing \emph{as it must}. The test that would decide the picture on this
family is the mirror at a rung where the amplitude has appeared, and \emph{the ladder cannot
supply one at this budget}. We registered the bracketing rungs, the build refused them, and the
registration was withdrawn having spent nothing; the reason is worth more than the experiment
was, so it is given exactly rather than summarised.

\emph{The ladder's rungs are quantised by the epoch, and the onset lies between two of them.}
\texttt{build\_lsweep} writes a \emph{one-epoch} ordering, which the trainer then passes over
three times, so a block must divide the $54$ chunks each source contributes per epoch: the
divisors are $1,2,3,6,9,18,27,54$ and are exactly the rungs already run. \blocked\ is not on
that ladder at all. It is a different construction, all of $A$ for three epochs and then all of
$B$, so its effective block is $L{=}162$, and the interval $54<L<162$ is not a finer division of
an epoch but a \emph{coarser} one that spans several. Such a rung is constructible: a
three-epoch path file with $L{=}81$ is perfectly writable, so what closes this region is not an
implementation limit. What forbids it is one section further on. Writing
the path across epochs means training one pass over a tripled file, and that is precisely the
construction the single-cosine control of \S\ref{sec:threats} was built to test, which came
back \textbf{void} on its own validity check because one epoch over a tripled file is not
equivalent in what it learns to three epochs over the original: total accuracy $0.4442$ against
$0.5125$, a gap of $2.5$ capability floors.

So the two negatives this paper reports separately are one negative. \emph{At a fixed multiset
and a fixed budget, the path space cannot be refined in the region where the stopping phase
turns on}, because every route into that region changes what is learned, and a rung that is not
learning-matched is not a rung on this ladder. That is a property of the instrument in the same
sense that a diffraction limit is a property of a lens: it is read off the construction rather
than discovered by failing. We label the onset reading as we labelled the picture: post hoc, and
now also untestable at this budget for a stated reason. \textbf{M2} stands as the registered
outcome.

\paragraph{The mirror of \blocked, run, and what its control permits.} The stopping-phase
reading rested on one informative mirror at $L{=}54$ plus \blocked\ as an endpoint. The most
direct test is the mirror of \blocked\ itself, training all of $B$ then all of $A$ against the
published $A$ then $B$. It was registered with its read-out and run: three seeds, the same
corpus, budget and evaluation.

\emph{The result, and then what it is worth.} The share relocates from $0.8692$ to $0.0343$,
crossing the mean-field point, with capability moving $0.0020$, well inside the $0.0270$
capability floor. The registered branch is \textbf{BM-2}: the mirror goes to the other side but
not by the same magnitude, $|{-}0.4657|$ against $|{+}0.3692|$, a gap of $5.05$ contrast floors
where the branch allowed two.

\emph{The comparability control fires, and it separates the two halves of that result.} The
published checkpoints for this corpus were reclaimed after evaluation, so the mirror had to be
scored on a node whose kernel cache was cold, which compiles the sampling kernels afresh. Before
reading the mirror we re-evaluated a \emph{published} checkpoint that does survive, the
\texttt{B\_only} arm the mirror resumes from. It reproduces the majority accuracy, $0.4250$
against a published $0.4238$, $0.06$ floors, and \emph{does not} reproduce the minority one,
$0.0900$ against $0.0575$, $1.70$ floors. Small accuracies moved and large ones did not.

That is exactly the regime the mirror's share lives in: at $S=0.0343$ the numerator is
$\mathrm{acc}_B \approx 0.0175$. Propagating the control's drift through gives $S$ anywhere in
$[0.000,0.092]$. So the two halves of BM-2 are not equally safe, and we separate them rather than
report the branch label alone.

\emph{\textbf{BM-3 is excluded robustly.}} Across the control's entire drift the share stays far
below the mean-field point, so the mirror does relocate to the other side and \blocked\ is a
position on a cycle rather than a distinct mechanism. That is the question this cell was built to
settle, it is settled, and Table~\ref{tab:axes}'s largest row keeps its explanation.

\emph{\textbf{The BM-1 against BM-2 distinction is not readable.}} It turns on a magnitude gap of
$0.0964$, and the control shows the small-accuracy regime can move $0.0325$, which is $34\%$ of
it. Whether the mirror is antisymmetric or carries a source-order asymmetry on top is therefore
open, and the cell that would close it is the same three arms scored on a node with a warm cache,
or a re-scoring of the published comparator on the same node. We report BM-2 as the mechanical
branch label and decline to interpret it. Source: \texttt{blocked\_mirror\_verdict}.

\paragraph{The mirror at a constant rate, and a guard that could not have passed.} Every mirror
above trains under a cosine, so the branch of Proposition~\ref{prop:sched} that speaks most
directly to a mirror, the terminal term carrying $\eta(T)$ with no contraction discount, is the
one branch with no schedule contrast. The cell was registered on a sign and a symmetry and
explicitly not on a magnitude, with its read-out frozen alongside it
(\texttt{prereg/constant\_rate\_mirror.md}). Six arms were trained, the comparator beside the
mirror rather than reused, for a reason the registration records as an amendment: the base model
the published arms trained from sits in a cache under \verb|$HOME|, which is node-local here and
is no longer on any node the scheduler will accept a job for.

\emph{The registered outcome is \textbf{CM-4}, void.} Capability moves $+0.0358$ between the two
arms, $1.33$ capability floors, past the one-floor bar the registration set as the condition for
reading a share. The shares are not read.

\emph{The guard could not have passed, which is a defect in the registration rather than a
property of the result.} \emph{Post hoc}: across the three seeds the capability difference is
$-0.0238$, $+0.0612$ and $+0.0700$, a sign change with a standard deviation of $0.0518$ and
$t=1.20$ against $4.303$. The mean is smaller than its own dispersion, and that dispersion is
$1.92$ capability floors by itself. \S\ref{sec:average} had already measured why: capability
spans $2.5$ capability floors across the constant family and under one across the cosine one,
which is why \S\ref{sec:conserve}'s conservation result is quoted from the cosine family
throughout. The bar was set at one capability floor for a family already reported at two and a
half, and a threshold no outcome of the design could have met is not stringency.

\emph{What the run printed, and what we decline to read.} $\Delta\mathrm{acc}_A=+0.4150$ against
$\Delta\mathrm{acc}_B=-0.3792$, which is $+21.73$ and $-19.85$ contrast floors, the former at
$+0.4087$, $+0.4187$ and $+0.4175$ on the three seeds. We print the numbers so that no reader need
wonder what we saw, and we do not read them, because the arms whose difference they are are not
matched in what they can do. Closing the cell means resolving the capability difference rather
than assuming it away: at the observed dispersion that is $17$ seeds for a $95\%$ interval inside
one capability floor. This paper does not buy them, and \S\ref{sec:family} lists the cell at that
size. Sources: \texttt{constant\_mirror\_verdict}, \texttt{constant\_mirror\_diagnosis}.

\emph{One thing the co-trained comparator does settle.} Trained from the shared copy of the base
model, it can be held against the published arm trained from the copy that has since disappeared.
The two agree to $0.0120$ in mean absolute accuracy, inside the $0.0191$ contrast floor, with the
sign of the difference changing across seeds, so the substitution the amendment made is not a
change this paper's numbers can see.

Two consequences deserve flat statement. \emph{Recency is not a confound of this system; it
is the system's one non-average degree of freedom}: the stopping phase is what every
recency effect in sequential fine-tuning is made of. And the exact-match reader should hold
this sentence against every blocked-beats-shuffled result, ours included:
\begin{quote}
\emph{\blocked\ wins the exact-match score precisely because it fails to optimise its own
objective (it never reaches the mixture that maximises the likelihood of its own corpus) and exact match rewards the failure.}
\end{quote}

\subsection{The Key: Training's Index Channel}\label{sec:key}

Everything above concerns a learner with only the compressive channel: no signal in the
prompt says which convention applies, so the state can carry only a mixture weight. CCH's
access-complete move is to add the index: keep bindings addressable by key, pay the price,
answer the query that names its target. The training-side analogue is disambiguation:
rebuild the natural-conflict corpus with a convention marker in the prompt, so the
convention becomes \emph{queryable} rather than contested, and re-run the arms.

\emph{The marker is not new, and the contribution is not the marker.} Conditioning training
text on a prepended tag and then steering with it at inference is an established technique:
control codes over source domains \citep{keskar2019ctrl}, conditioning on human-preference
scores \citep{korbak2023pretraining}, metadata conditioning with a cooldown so the model still
runs unconditioned \citep{gao2025meco}, and document identifiers injected to make knowledge
attributable \citep{khalifa2024sourceaware}. What that literature has not had is a setting where
the tag disambiguates \emph{mutually contradictory} supervision, and therefore no measurement of
where the technique stops working. That is what this section supplies: the marker collapses the
arrangement effect completely, and then recovers only $87.5\%$ of the union ceiling, with the
shortfall traced to a specific and predictable failure, coverage of the minority convention at
write time ($0.3767$ against $0.0025$). The negative direction has independent support:
\citet{higuchi2025whenmetadata} find on controlled grammars that metadata conditioning
\emph{hurts} when the context does not determine the latent the tag names, which is the same
boundary reached from the other side.

The deciding read-out was frozen with two components. \textbf{C: does the switch collapse?}
It does: the arrangement effect falls from $D=-0.0985$ ($5.16$ floors) unmarked to $-0.0006$
($0.03$ floors, eight seeds) marked, read-out C1. The switch is convention-\emph{selection}, and when
selection is moved from the path to the query, arrangement stops mattering: the commitment
reading of \S\ref{sec:conserve} survives its sharpest test. \textbf{L: does capability reach
the union?} Partially: the marked mean reaches $0.3259$ against the additive ceiling
$0.3727$, $87.5\%$ of the ceiling, $74.9\%$ of the gap recovered (read-out L4, against
the L1 target $0.3354$, short by $0.50$ floors). So conservation is an approximation (``to
within an eighth''), not a law,
and the deficit has a mechanism rather than an excuse:

(Every number here is computed by an analyzer that also recovers the three
constants the criterion was frozen against (now $0.1863$, $0.3727$, $0.3354$) from the
result base rather than restating them, which is what lets the section survive a change in its
inputs without a hand edit: the unkeyed corpus went from three
seeds to eight (job $2654$), which moved the baseline from $0.1925$ to $0.1863$, the ceiling
from $0.3850$ to $0.3727$ and the fraction of that ceiling from $84.7\%$ to $87.5\%$, with no hand
edit and no change of verdict.)

\begin{quote}
Marked $B$-\only, trained on nothing but the spelled convention, scores $0.3767$ on the
numeral evaluation when the key asks for a numeral, against $0.1283$ unmarked. Marked
$A$-\only, trained on nothing but numerals and asked for a word, scores $0.0025$. \emph{The
key selects among conventions the model already holds; it cannot recover one that was never
written.} Numerals are the pretrained default; the spelled form must be acquired.
\end{quote}

\noindent Figure~\ref{fig:key} separates the two obstructions and Table~\ref{tab:unlock} gives the
read-out.

\begin{figure}[t]
\centering
%
%
%
\begin{tikzpicture}[font=\small,>=Stealth,
 ax/.style={cGrey!75,line width=0.5pt},
 tk/.style={font=\scriptsize,cGrey},
 lb/.style={font=\scriptsize,cGrey!90!black},
 hd/.style={font=\scriptsize\bfseries}]

\begin{scope}[shift={(0,0)}]
 \def\S{11.6}                      
 \draw[ax,->] (0,0) -- (0,4.9);
 \foreach \v in {0.1,0.2,0.3}{
   \draw[ax] (0,\v*\S) -- (-0.09,\v*\S); \node[tk,anchor=east] at (-0.11,\v*\S) {$\v$};}
 \node[tk,rotate=90,anchor=south] at (-0.62,2.2) {mean accuracy};

 \fill[cV!12] (0.30,0) rectangle (3.30,0.3727*\S);
 \draw[cV!75,line width=0.9pt] (0.30,0.3727*\S) -- (3.30,0.3727*\S);
 \node[lb,cV!70!black,anchor=south west] at (0.30,0.3727*\S+0.04)
  {additive ceiling $0.3727$};

 \fill[cGrey!45] (0.55,0) rectangle (1.55,0.1863*\S);
 \node[lb,anchor=north,align=center,text width=1.3cm] at (1.05,-0.06) {unkeyed\\ $0.1863$};
 \fill[cA!75] (2.05,0) rectangle (3.05,0.3259*\S);
 \node[lb,anchor=north,align=center,text width=1.3cm] at (2.55,-0.06) {keyed\\ $0.3259$};

 \draw[cD,line width=0.8pt,<->] (3.18,0.3259*\S) -- (3.18,0.3727*\S);
 \node[lb,cD,anchor=west,align=left,text width=1.55cm] at (3.30,0.300*\S)
  {residue $\mathbf{0.0468}$: $12.5\%$ of the ceiling};
 \draw[cA!80,line width=0.8pt,<->] (1.80,0.1863*\S) -- (1.80,0.3259*\S);
 \node[lb,cA!85!black,anchor=east,align=right,text width=1.5cm] at (1.72,0.256*\S)
  {$74.9\%$ of the gap};

 \node[hd,anchor=north west,align=left,text width=5.0cm] at (-0.30,-0.72)
  {\textbf{(a)} the key opens most of the union};
\end{scope}

\begin{scope}[shift={(6.05,0)}]
 \def\S{11.6}
 \draw[ax,->] (0,0) -- (0,4.9);
 \foreach \v in {0.1,0.2,0.3}{
   \draw[ax] (0,\v*\S) -- (-0.09,\v*\S); \node[tk,anchor=east] at (-0.11,\v*\S) {$\v$};}

 %
 \fill[cGrey!40] (0.30,0) rectangle (1.00,0.1283*\S);
 \node[lb,anchor=north,align=center,text width=1.16cm] at (0.65,-0.06) {unmarked\\ $0.1283$};
 \fill[cA!75] (1.45,0) rectangle (2.15,0.3767*\S);
 \node[lb,anchor=north,align=center,text width=1.16cm] at (1.80,-0.06) {marked\\ $\mathbf{0.3767}$};
 \node[lb,cA!85!black,anchor=south,align=center,text width=2.3cm] at (1.22,0.3767*\S+0.10)
  {$B$-\only\ asked for a \emph{numeral}: written, so served};

 \fill[cD!70] (3.25,0) rectangle (3.95,0.0025*\S);
 \draw[cD,line width=0.6pt] (3.25,0.0025*\S) -- (3.95,0.0025*\S);
 \node[lb,cD,anchor=north,align=center,text width=1.16cm] at (3.60,-0.06)
  {marked\\ $\mathbf{0.0025}$};
 \node[lb,cD,anchor=south,align=center,text width=2.15cm] at (3.60,0.42)
  {$A$-\only\ asked for a \emph{word}: never written, so nothing to select};

 \node[hd,anchor=north west,align=left,text width=4.15cm] at (-0.30,-0.72)
  {\textbf{(b)} the residue is coverage};
\end{scope}

\begin{scope}[shift={(10.80,0)}]
 \def\S{11.6}
 \draw[ax,->] (0,0) -- (0,4.9);

 \fill[cD!35] (0.35,0) rectangle (1.05,0.0468*\S);
 \node[lb,anchor=north,align=center,text width=1.15cm] at (0.70,-0.06)
  {Qwen2.5\\ $0.0468$};
 \fill[cD] (1.55,0) rectangle (2.25,0.1025*\S);
 \node[lb,anchor=north,align=center,text width=1.15cm] at (1.90,-0.06)
  {Qwen3\\ $\mathbf{0.1025}$};

 \draw[cGrey!70,densely dashed] (0.20,0.0468*\S+2*0.0191*\S) -- (2.60,0.0468*\S+2*0.0191*\S);
 \draw[cGrey!55,line width=0.4pt] (0.75,0.0468*\S+2*0.0191*\S) -- (0.75,2.05);
 \node[lb,cGrey!85!black,anchor=south west,align=left,text width=2.3cm] at (0.05,2.08)
  {the $2$-floor bar, frozen first};
 \draw[cD,line width=0.8pt,<->] (2.42,0.0468*\S) -- (2.42,0.1025*\S);
 \node[lb,cD,anchor=south,align=center,text width=1.7cm] at (1.90,0.1025*\S+0.08)
  {$\mathbf{+2.92}$ floors: \textbf{V1}};

 \node[hd,anchor=north west,align=left,text width=3.6cm] at (-0.30,-0.72)
  {\textbf{(c)} and grows where coverage thins};
\end{scope}
\end{tikzpicture}
\caption{\textbf{What the key opens, and why what it leaves is a different obstruction.}
\textbf{(a)} The unkeyed model sits at exactly half the additive ceiling, which is what a coin
flip between two mutually exclusive golds scores; the key recovers $74.9\%$ of the gap and leaves
$0.0468$. \textbf{(b)} That residue is \emph{coverage}, and this panel is the whole of the
evidence. Marked $B$-\only, trained on nothing but the spelled convention, answers in numerals at
$0.3767$ because numerals are the pretrained default and are therefore written; marked
$A$-\only\ asked for a word scores $0.0025$, because there is nothing to select. \textbf{(c)} And
the residue grows on a family that writes four times less of the minority convention, $+2.92$
floors against a bar of two frozen before the run (V1). It is a necessary-condition test, which V2
or V3 would have refuted and which V1 cannot establish. Sources \texttt{unlock\_verdict},
\texttt{q3unlock\_verdict}.}
\label{fig:key}
\end{figure}
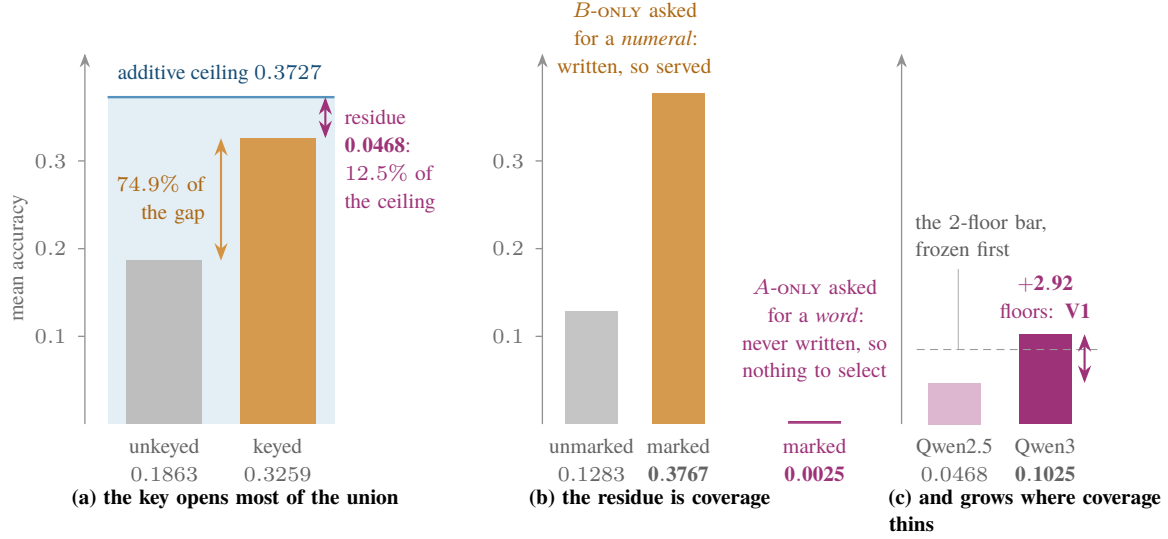

\begin{table}[H]
\centering\small
\caption{\textbf{What the key opens, and what it cannot.} The read-out was frozen in the job
script before the runs (C on the switch, L on the level), and the analyzer recovers the three constants it was frozen against from the result base. The residue has a mechanism
rather than an excuse: a marked $B$-\only\ model, trained on nothing but the spelled convention,
answers in numerals when the key asks for numerals, but a marked $A$-\only\ model asked for a
word cannot produce one it never saw. The key \emph{selects} among conventions already written.}
\label{tab:unlock}
\setlength{\tabcolsep}{3pt}\footnotesize
\begin{tabular}{llccl}
\toprule
read-out & quantity & unkeyed & keyed & verdict \\
\midrule
C & arrangement switch $D$ & $-0.0985$ ($5.16$ floors) & $\mathbf{-0.0006}$ ($0.03$ floors) & \textbf{C1}: it collapses \\
L & mean accuracy & $0.1863$ & $\mathbf{0.3259}$ & \textbf{L4}: partial unlock \\
 & \quad against ceiling $0.3727$ & $50.0\%$ & $87.5\%$ & $74.9\%$ of the gap \\
 & \quad against L1 target $0.3354$ & --- & short by $0.50$ floors & \\
\midrule
\multicolumn{5}{l}{\emph{the residue, and why it is coverage rather than difficulty}}\\
& marked $B$-\only\ on the numeral evaluation & $0.1283$ & $\mathbf{0.3767}$ & the key selects \\
& marked $A$-\only\ on the spelled evaluation & --- & $\mathbf{0.0025}$ & it cannot create \\
\bottomrule
\end{tabular}

\end{table}

\textbf{The $87.5/12.5$ split is two walls, not one.} Read through
Remark~\ref{rem:wallkind} the result decomposes exactly: the key removes the
underdetermination term \emph{in full}: that is the collapse of $D$ from $-0.0985$ to
$-0.0006$, and it costs no capacity because it changes which policy is optimal rather than
what can be stored. What survives is a \emph{coverage} term of a different species: the key
cannot serve a convention the corpus never wrote, and the $0.3767$ against $0.0025$ asymmetry
above measures precisely that residue. Two obstacles, two mechanisms, separated by one
experiment; the $12.5\%$ is not a weaker version of the $87.5\%$ but a different quantity, and a
difficulty-matched control (\S\ref{sec:threats}) is what would settle whether the residue is
coverage or merely difficulty.

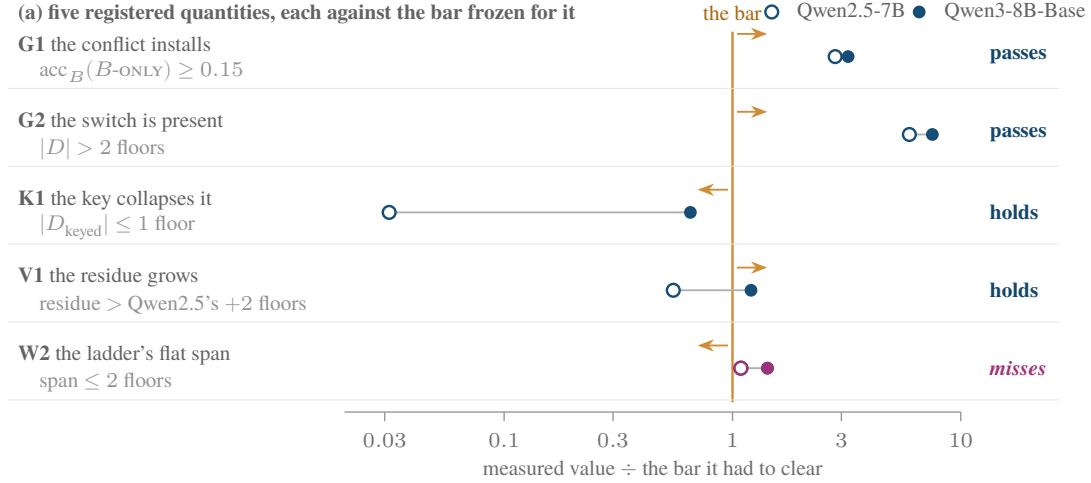
\begin{figure}[t]
\centering
\begin{tikzpicture}[font=\footnotesize,>=Stealth,
 ax/.style={cGrey!75,line width=0.5pt},
 tk/.style={font=\scriptsize,cGrey},
 lb/.style={font=\scriptsize,cGrey!90!black},
 ru/.style={font=\scriptsize,cGrey!75},
 ok/.style={font=\scriptsize\bfseries,cV!65!black},
 no/.style={font=\scriptsize\bfseries,cD},
 hd/.style={font=\scriptsize\bfseries,cGrey!90!black}]

\def\ax{4.45}
\def\aw{8.15}
\def\bx{9.581}                      

\node[hd,anchor=west] at (0,5.72) {\textbf{(a)} five registered quantities, each against the bar frozen for it};
\draw[cV!70!black,line width=0.9pt] (10.10,5.74) circle (2.4pt);
\node[lb,anchor=west] at (10.30,5.74) {Qwen2.5-7B};
\fill[cV!70!black] (12.05,5.74) circle (2.4pt);
\node[lb,anchor=west] at (12.25,5.74) {Qwen3-8B-Base};

\draw[cA!85,line width=0.9pt] (\bx,0.58) -- (\bx,5.48);
\node[lb,cA!85!black,anchor=south] at (\bx,5.50) {the bar};
\draw[ax] (\ax,0.45) -- (12.60,0.45);
\foreach \v/\x in {0.03/4.982, 0.1/6.561, 0.3/8.001, 1/9.581, 3/11.021, 10/12.600}{
  \draw[ax] (\x,0.45) -- (\x,0.33); \node[tk,anchor=north] at (\x,0.31) {$\v$};}
\node[tk,anchor=north] at (8.525,-0.06) {measured value $\div$ the bar it had to clear};

\foreach \y/\nm/\rl/\p/\q/\dir/\vd/\st in {%
 5.15/{\textbf{G1} the conflict installs}/{$\mathrm{acc}_B(B\text{-}\only)\ge0.15$}/10.939/11.110/1/{passes}/{ok},
 4.12/{\textbf{G2} the switch is present}/{$|D|>2$ floors}/11.920/12.224/1/{passes}/{ok},
 3.09/{\textbf{K1} the key collapses it}/{$|D_{\text{keyed}}|\le1$ floor}/5.041/9.024/-1/{holds}/{ok},
 2.06/{\textbf{V1} the residue grows}/{residue $>$ Qwen2.5's $+2$ floors}/8.798/9.826/1/{holds}/{ok},
 1.03/{\textbf{W2} the ladder's flat span}/{span $\le2$ floors}/9.692/10.044/-1/{\emph{misses}}/{no}}
{
 \node[lb,anchor=west] at (0,\y+0.17) {\nm};
 \node[ru,anchor=west] at (0.28,\y-0.19) {\rl};
 \draw[cGrey!18,line width=0.3pt] (0,\y-0.42) -- (13.90,\y-0.42);
 \draw[cGrey!55,line width=0.7pt] (\p,\y) -- (\q,\y);
 \draw[cV!70!black,line width=0.9pt,fill=white] (\p,\y) circle (2.4pt);
 \fill[cV!70!black] (\q,\y) circle (2.4pt);
 \draw[cA!85,->,line width=0.7pt] (\bx+\dir*0.06,\y+0.30) -- (\bx+\dir*0.46,\y+0.30);
 \node[\st,anchor=west] at (12.85,\y) {\vd};
}
\draw[cD,line width=0.9pt,fill=white] (9.692,1.03) circle (2.4pt);
\fill[cD] (10.044,1.03) circle (2.4pt);

\draw[cGrey!55,densely dashed,line width=0.5pt] (0,-0.62) -- (13.90,-0.62);
\node[ru,anchor=north west,align=left,text width=13.9cm] at (0,-0.72)
 {\emph{post hoc}, chosen after \textbf{W2} came back: resolvable endpoint states
  (Def.~\ref{def:channel}) are $2$ on both families, which does not convert a miss into a hold};

\begin{scope}[shift={(1.00,-6.55)}]
 \node[hd,anchor=west] at (-1.00,4.52) {\textbf{(b)} \textbf{M2} the mirror, whose rule is a shape and not a number};
 \fill[cV!8] (0.5,0) rectangle (3.5,3.0);
 \draw[cGrey!35,line width=0.4pt] (0,0) rectangle (3.5,4.0);
 \draw[ax] (0,3.0) -- (3.5,3.0);
 \draw[ax] (0.5,0) -- (0.5,4.0);
 \foreach \v/\x in {0/0.5, 2/1.5, 4/2.5, 6/3.5}{
   \draw[ax] (\x,0) -- (\x,-0.12); \node[tk,anchor=north] at (\x,-0.14) {$\v$};}
 \foreach \v/\y in {-6/0.0, -4/1.0, -2/2.0, 0/3.0, 2/4.0}{
   \draw[ax] (0,\y) -- (-0.12,\y); \node[tk,anchor=east] at (-0.14,\y) {$\v$};}
 \draw[cGrey!70,densely dashed,line width=0.7pt] (0,3.5) -- (3.5,0.0);
 \node[ru,anchor=south west] at (0.62,3.60) {$\Delta_A{+}\Delta_B{=}0$};
 \node[tk,anchor=north] at (1.75,-0.52) {$\Delta\mathrm{acc}_A$ (contrast floors)};
 \node[tk,rotate=90,anchor=south] at (-0.58,2.0) {$\Delta\mathrm{acc}_B$};
 \node[ru,anchor=south west] at (0.68,0.18) {relocation quadrant};
 \draw[cV!70!black,line width=0.9pt,fill=white] (2.725,0.7095) circle (2.6pt);
 \node[lb,cV!70!black,anchor=east] at (2.56,0.7095) {Qwen2.5-7B};
 \fill[cD] (1.034,3.0315) circle (2.6pt);
 \node[lb,cD,anchor=south west] at (1.13,3.10) {Qwen3-8B-Base};
\end{scope}

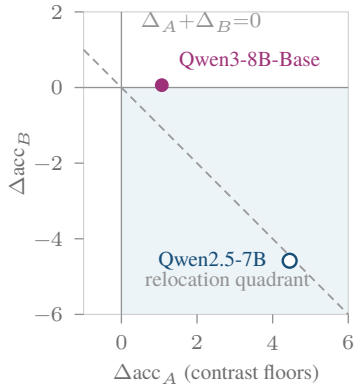
\node[lb,anchor=north west,align=left,text width=8.5cm] at (5.35,-2.20)
 {The mirror rebuilds $L{=}54$ to end on $A$ instead of $B$ and asks the endpoint to move the same
  magnitude to the other side. On Qwen2.5 it does: $+4.45$ against $-4.58$ floors, on the dashed
  line of pure relocation, with capability moving $0.09$ of a capability floor. On Qwen3 both skills move \emph{up},
  $+1.07$ and $+0.06$, which is a capability change of $0.80$ capability floors and not a
  relocation at all. The two families disagree about the shape rather than about a threshold,
  which is why this row is drawn and not scored; $L{=}54$ is inside the flat region on Qwen3
  (panel (a), \textbf{W2}), so there is no amplitude there for a mirror to move.};
\end{tikzpicture}
\caption{\textbf{Three results asked to travel, and the answers are not uniform.}
\textbf{(a)} Every quantity whose threshold was frozen before the second family ran, divided by
that threshold, so the bar sits at $1$ for all of them and the arrow says which side had to be
cleared. Hollow is Qwen2.5-7B, filled is Qwen3-8B-Base, and the line between them is the travel.
The guards pass well: the conflict installs \emph{better} on the new family, so nothing below is
an instrument at its floor. The key collapses the switch to $3\%$ of its bar on one family and
$66\%$ on the other, and the residue grows as the coverage reading requires. The ladder's span
misses \emph{and so does the published row at $1.09$}, which is a fault in the registration
rather than in the second family. Rows are ordered by what they decide, not by outcome, and the
post-hoc statistic is ruled off. \textbf{(b)} The mirror is the one registration with no scalar
bar: its rule is a shape, so it is drawn as one. Sources \texttt{q3unlock\_verdict},
\texttt{q3ladder\_verdict}, \texttt{unlock\_verdict}, \texttt{rate\_verdict}.}
\label{fig:travel}
\end{figure}

\paragraph{The same key on a second pretraining family, and what the residue is made of.}
The collapse above is one model, and a remedy can be family-specific where a phenomenon is not:
the key works by making the convention conditionable, which is a property of pretraining rather
than of the conflict. We registered the replication before running it, on Qwen3-8B-Base, at the
unmarked family's own budget and seeds, and it holds: $D$ falls from $-0.0875$ ($4.58$ floors) to
$-0.0125$ ($0.66$ floors), read-out \textbf{K1} (Table~\ref{tab:family2}).

The registration also carried a second branch, and it is the one worth the twenty trainings. This
family is not merely a second model: it is a second point on the axis this section's own
explanation names. Trained on nothing but the spelled convention, Qwen3 reaches
$\mathrm{acc}_B=0.0520$ while still emitting $0.2983$ in numerals, against Qwen2.5's $0.2091$:
\emph{four times less of convention $B$ is written}, at the same corpus, budget, seeds and $k$. If
the residue is coverage, a family that writes less of the minority convention must leave a larger
one. The threshold was fixed at two floors on the residue before the run, and the residue grows
from $0.0468$ to $0.1025$, a move of $\mathbf{+2.92}$ floors: read-out \textbf{V1}. The unlock
falls with it, from $87.5\%$ of the ceiling to $75.5\%$.

\emph{What that does and does not buy.} It turns the $87.5/12.5$ split from a number into a
relationship, which is more than a single model can support and less than a proof. The two
families differ in everything a pretraining corpus can differ in, so V was registered as a
\emph{necessary-condition} test: V2 or V3 would have refuted the coverage reading, while V1 is
consistent with it and cannot establish it, because some third property of Qwen3 could move the
residue the same way. The difficulty-matched control of \S\ref{sec:threats} therefore stays a
missing cell rather than being quietly retired.

\emph{One instrument correction, because it fired first.} The read-out audits the instrument
before printing any branch, and on first execution it refused: eighteen of twenty keyed arms
contained problems scored correct under \emph{both} conventions, which the audit called
impossible. The audit was wrong. Exclusivity holds for the unmarked corpus, where one prompt
carries two mutually exclusive golds; the marked evaluations are two \emph{different} prompts, so
a model that serves both scores on both, which is what this section claims the key buys. The check
was corrected against the published run: \texttt{dis\_q25\_7b} carries $17$--$19$ such problems per
arm while the unmarked families carry exactly $0$. The count is now recorded rather than treated
as a failure, and is a coverage signal of its own: $2$ here against $17$--$19$ there.
Figure~\ref{fig:travel} collects the three results asked to travel.

\begin{table}[t]
\centering\small
\setlength{\tabcolsep}{4pt}
\caption{\textbf{The key, and the ladder, asked to travel.} Both were registered before their runs
with the branches and thresholds frozen, and both gates were read first. The key replicates and
its residue moves as the coverage reading requires. The ladder's registered branch is \textbf{W2},
and the row below it is why that verdict is reported together with the threshold it was read
against: the published Qwen2.5 ladder does not meet the same bar, which was checkable before the
run and was not checked. Sources: \texttt{q3unlock\_verdict}, \texttt{q3ladder\_verdict}.}
\label{tab:family2}
\begin{tabular}{@{}l>{\raggedright\arraybackslash}p{4.1cm}ccl@{}}
\toprule
& quantity & Qwen2.5-7B & Qwen3-8B-Base & read-out \\
\midrule
\textbf{guards} & minority convention written, $\mathrm{acc}_B(B\ \only)$ & $0.2091$ & $0.0520$ & --- \\
 & \quad on the \emph{synthetic} conflict & $0.4225$ & $\mathbf{0.4813}$ & G1 passes \\
 & switch present, $|D|$ on that corpus & $11.9$ fl & $\mathbf{15.0}$ fl & G2 passes \\
\midrule
\textbf{the key} & $D$ unkeyed & $-0.0985$ ($5.16$ fl) & $-0.0875$ ($4.58$ fl) & \\
 & $D$ keyed & $-0.0006$ ($0.03$ fl) & $-0.0125$ ($0.66$ fl) & \textbf{K1} collapses \\
 & marked mean against its ceiling & $87.5\%$ & $75.5\%$ & \textbf{L4} partial \\
 & residue $=$ ceiling $-$ marked & $0.0468$ ($2.45$ fl) & $\mathbf{0.1025}$ ($5.37$ fl) & \textbf{V1} $+2.92$ fl \\
\midrule
\textbf{the ladder} & flat region span, $L{=}1$--$27$, \purerand, \shuf & $0.0416$ ($2.18$ fl) & $0.0544$ ($2.85$ fl) & \textbf{W2} at a $2$ fl bar \\
 & \quad \emph{which the published row also misses} & \multicolumn{2}{c}{differ by $0.67$ floors} & \emph{post hoc} \\
 & \quad span as a fraction of the blocked excursion & $9\%$ & $12\%$ & \emph{post hoc} \\
 & resolvable endpoint states (Def.~\ref{def:channel}) & $\mathbf{2}$ & $\mathbf{2}$ & $R_{\mathrm{real}}{=}1$ bit \\
 & \quad within-cluster / between-cluster gap & $8.7\times$ & $20.7\times$ & \emph{post hoc} \\
 & $L{=}54$ & $0.4922$, off the floor & $0.4547$, \emph{inside} it & \\
 & \blocked & $0.8692$ & $0.8907$ & \\
 & mirror at $L{=}54$: $\Delta\mathrm{acc}_A$, $\Delta\mathrm{acc}_B$ & $+0.0850$, $-0.0875$ & $+0.0204$, $+0.0012$ & \textbf{M2} \\
\bottomrule
\end{tabular}

\end{table}

This is also the training-time face of the inference-time theory's load-bearing assumption.
CCH's hybrid crosses its walls only under \emph{write-time code separability}: the anchor
must be distinguishable when the binding is written, not merely when it is queried
\citep{chen2026cch}, which in the vocabulary of Remark~\ref{rem:wallkind} is the assumption
that $H(\text{convention}\mid\text{query})=0$. CCH assumes the stream is well posed; this
paper measures what a learner does when it is not. Our marker is exactly such an anchor, and
its failure mode is exactly the assumption's: present at write time, it opens the union up to
what was written; absent (or the content never learned), no query-time cleverness recovers the
binding. And the
family's super-additivity has a measured analogue: the reachable behaviour set with the key
strictly contains the set without it: calibrated mixture \emph{and} per-query convention
control, against mixture-or-corner alone, at the price the walls demand (the unlearned
remainder), not for free.

\section{Threats to Validity}\label{sec:threats}

\paragraph{Every registered claim, and how it came back.} The abstract says three mechanisms
were pre-registered and falsified. Table~\ref{tab:ledger} is the whole ledger rather than those
three, because a scorecard that lists only the informative failures is a selection, and because
two further entries cost us something a reader should not have to reconstruct: the positive
control inverted on its first build, and the control registered to discharge the learning-rate
confound came back void. Rules are quoted from the artefact that froze them; the registration
index of Appendix~\ref{app:registry} gives the commit and the interval for each.

%
\begingroup
\footnotesize
\setlength{\tabcolsep}{4pt}
\begin{longtable}{@{}>{\raggedright\arraybackslash}p{3.15cm}>{\raggedright\arraybackslash}p{4.35cm}>{\raggedright\arraybackslash}p{4.65cm}l@{}}
\caption{\textbf{The ledger: every claim this paper registered, and what came back.} Read the
verdict column as a distribution rather than a score: three mechanisms died, four predictions
held, one control was void and one instrument build inverted. The three deaths were derivable in
advance from Proposition~\ref{prop:avg}, which is why they are kept rather than deleted. Rules
are quoted from the artefact that froze them; Appendix~\ref{app:registry} gives the commit and
the interval for each.}
\label{tab:ledger} \\
\toprule
registered claim & the rule, frozen first & what came back & verdict \\
\midrule
\endfirsthead
\toprule
registered claim & the rule, frozen first & what came back & verdict \\
\midrule
\endhead
\bottomrule
\endfoot
\bottomrule
\endlastfoot
\textbf{v1} batch composition carries the switch
 & \purerand\ near \blocked\ $\Rightarrow$ purity; near \shuf\ $\Rightarrow$ order
 & \purerand\ $0.2242$ against \shuf\ $0.2179$, \blocked\ $0.4454$: purity is $3\%$ of the span
 & \textbf{dead} \\
\textbf{v2} optimiser memory cancels contested directions
 & a response turning on at $L^\star\!\approx\!10$, or at $\tau_{\text{write}}\!=\!13.6$
 & $M(L)$ flat through both onsets, $\sigma(M)\!\approx\!0.06$; only $L{=}54$ rises
 & \textbf{dead} \\
\textbf{Strang} a palindromic schedule beats a blocked one
 & S1 median $|D_{\text{str}}|/|D_{\text{blk}}|<0.5$; S2 slopes differ by $\ge0.5$
 & S1 $3.65$; S2 both slopes $0$; S3 passes
 & \textbf{dead} \\
\midrule
\textbf{T1} allocation tracks the mixture weights
 & share within $0.05$ of mean-field $0.500$ and $0.333$
 & $0.453$ and $0.312$: misses $0.047$ and $0.021$, both toward the pretrained prior
 & \textbf{held} \\
\textbf{Ladder} every block length lands at the interleaved allocation
 & $L{=}1\ldots54$ within noise of \shuf's own share
 & $0.407$--$0.449$ against \shuf's $0.413$
 & \textbf{held} \\
\textbf{M1} the mirror relocates antisymmetrically at fixed $C$
 & ending on $A$ moves the same magnitude to the other side
 & $+0.0850$ against $-0.0875$, capability moves $0.0025$
 & \textbf{held} \\
\textbf{C} the key collapses the switch
 & $|D_{\text{keyed}}|$ within one floor of zero
 & $-0.0985 \to -0.0006$ ($5.16 \to 0.03$ floors)
 & \textbf{held} \\
\midrule
\textbf{L} the key lifts capability to the union
 & marked mean $\ge 0.90\times$ the additive ceiling $0.3727$
 & $0.3259$: $87.5\%$ of the ceiling, short of the bar by $0.50$ floors
 & \textbf{partial} \\
\textbf{S-C} a single-cosine arm discharges the schedule confound
 & validity check S-C4 read before any allocation comparison
 & total accuracy $0.4442$ against $0.5125$, a gap of $2.5$ capability floors
 & \textbf{void} \\
\textbf{Control} conflict fires and its twin stays inert
 & control within noise, conflict beyond it
 & v1 \emph{inverted} ($0.10\sigma$ against $2.41\sigma$); v3 reads $0.15\sigma$ against $12.29\sigma$
 & \textbf{v1 failed} \\
\midrule
\textbf{Q3-key} the key collapses on a second pretraining family
 & branches C, L and V frozen before the first result file
 & $D$ $-0.0875\!\to\!-0.0125$; unlock $75.5\%$ against $87.5\%$; residue grows $2.92$ floors
 & \textbf{K1/L4/V1} \\
\textbf{Q3-wall} the averaging wall replicates
 & span of the flat region $\le2$ floors, gate read first
 & gate passes at $15.0$ floors; span $2.85$ floors, and the \emph{published} row is $2.18$, so
   the bar was mis-specified
 & \textbf{W2} \\
\textbf{Q3-mirror} the stopping phase replicates
 & antisymmetric shift at $L{=}54$, capability within a floor
 & $+0.0204$ against $+0.0012$; $L{=}54$ is inside the flat region on this family
 & \textbf{M2} \\
\textbf{high-$k$} the borrowed evaluation share transfers, so
   $\sigma$ falls as $\sqrt{0.136+0.864\cdot4/k}$
 & $\mathrm{ratio}(16)\in[0.475,0.713]$ around a predicted $0.594$; \emph{all} eleven arms
   required at each of $k{=}16$ and $k{=}32$
 & $k{=}16$ complete at $33/33$; $\mathrm{ratio}(16)=1.079$, outside the band and above
   \emph{one}, which the additive model can carry only as a negative component. \S\ref{sec:average}'s
   shared engine stream predicts one
 & \textbf{HK-4} \\
\end{longtable}
\endgroup

\noindent The last row's bookkeeping, since the verdict turns on it rather than on a
measurement. \textbf{G-HK1} fails because the job ran nine arms at $k{=}32$ and the guard asks
for eleven, so the cell is void on bookkeeping and not on measurement; the six missing
evaluations were submitted and then \emph{withdrawn unstarted} when the project's compute
closed. They would only have moved the label to \textbf{HK-2}, because \S\ref{sec:average}'s
correction comes from a direct measurement that decides no branch.

\paragraph{Three byproducts, flagged not developed.} \emph{Ordering poisoning:} \purerand\
and \blocked\ are the same multiset in different order and differ by $0.22$ exact-match: an attacker controlling only data-loader order, touching no byte, decides which convention a
model commits to; the existing arms are the demonstration's skeleton. \emph{The $\kappa$
tracer:} a few hundred conflicting pairs planted in any large run measure that pipeline's
effective averaging in situ, without touching the main data. \emph{Evaluation methodology:}
any benchmark whose answers admit multiple correct conventions is silently scoring
commitment; the decomposition of Eq.~\eqref{eq:decomp} separates the two at zero cost.

\paragraph{Limitations.} The ladder, mirror, ratio and key experiments are one model
(Qwen2.5-7B) at one budget on two conflict constructions. The switch alone is broader:
\thecompanion\ paper measures it at $3$B, $7$B and $14$B within one pretraining family and
replicates it on a second, pre-registered before the run (Qwen3-8B-Base: $-0.0875$ at $4.58$
floors, five of five seeds negative, control inert at $0.93$) \citep{frl2026}. That replication
carries the caveat this paper is in the worst position to wave through: its learnability guard
reads $0.052$ against $0.209$ on Qwen2.5-7B, so the second family fires the switch from a rung
close to the floor at which \S\ref{sec:design}'s v1 build failed outright.

Three of this paper's own results have since been asked to travel to that family, each registered
before its run, and the answers are not uniform (Table~\ref{tab:family2}). The \textbf{key
replicates} and its residue moves as the coverage reading requires (K1, V1). The \textbf{ladder's}
registered branch is \textbf{W2}, and the honest report of it is two sentences rather than one: the
bar was two floors, the second family spans $2.85$, and the \emph{published} first family spans
$2.18$ and misses the same bar. We set that threshold without checking it against the row we
already had. The \textbf{mirror} returns M2 at a block length that, on this family, sits inside the
flat region and therefore has no amplitude to relocate. What travels, then, is the key and the
switch; what the ladder shows is that the two families are flat to within two thirds of a floor of
each other under a criterion we mis-specified, which is weaker than the replication we registered
and stronger than the failure the branch label alone suggests. Under the criterion
Proposition~\ref{prop:avg} actually implies, the number of resolvable endpoint states, both
families read $2$ with an order-of-magnitude margin (\S\ref{sec:average}); that comparison is
post hoc and does not convert W2 into W1, but it does say which of the two statistics was
measuring the wall.

\emph{What the flatness claim is powered to say.} \S\ref{sec:average} states it in full and it
is short: the interior span is smaller than this design's detectable difference at any power we
would quote, so the averaging wall is defended by the pooled between-arm statistic and by the
corner, not by the interior's shape.
Conservation is approximate: $9.7\%$ arm-to-arm on the synthetic corpus, $15\%$ on the
natural one, range $0.1475$ along a trajectory, $74.9\%$ gap recovery under the key, and only \emph{mutually exclusive} conflicts are tested; style conflicts, or cases where both
answers can be wrong, are open. The learning-rate schedule was entered as a residual confound and is not one: \blocked\
is a second cosine, the ladder a single one, and the gap it could
account for is $0.8559$ (the blocked arm at eight seeds) against the ladder's top rung at $0.4922$, nineteen floors, which is not a
size a reader should be asked to dismiss on a plausibility argument. We therefore registered the
single-cosine $A^{162}B^{162}$ control before building it, froze its four branches, ran it, and
\textbf{it came back void}. The validity check S-C4 is read first by construction and it fires:
total accuracy is $0.4442$ against the blocked arm's $0.5125$, a gap of $2.5$ capability floors,
so one epoch over a tripled file is not equivalent in learning to three epochs per stage and the
two arms are not comparable. The registered consequence is taken in full. The allocation
comparison is void, is reported as an instrument result, and is \emph{not} repaired by adjusting
the budget.

Two things follow and neither is comfortable. First, the confound is not discharged. The
plausibility argument that the quantity explained is a mixture weight rather than an amount
learned is not available here, and the failure of the experiment we sent to settle it does not
reinstate that argument. Second, the direction of the void is not neutral. The share it
forbids us to read is $0.5089$, which sits $0.87$ floors above the ladder's top rung, and the
branch it would have landed in is S-C3, the one that amends this paper's abstract. We print the
number so that no reader need wonder what we saw, and we do not read it, because the arms whose
difference it is are not matched in what they learned.

The components show why the design cannot be repaired cheaply. Under a single cosine the terminal
B block trains in the decayed half of the schedule, so B is learned less and A is forgotten less:
at seed $42$, $\mathrm{acc}_A$ rises from $0.078$ to $0.215$ while $\mathrm{acc}_B$ falls from
$0.450$ to $0.224$. That is the recency-by-learning-rate interaction of
Appendix~\ref{app:budget}, and it is exactly why the registration said in advance that this
control measures two-stage against one-stage rather than the cosine alone. That same
non-equivalence has a second consequence we did not anticipate when we ran this control, and
\S\ref{sec:phase} draws it: the tripled file is also the only route to a ladder rung between
$L{=}54$ and \blocked, so the $2.5$ floors below close the onset region as well as voiding this
comparison. One measurement, two negatives. Separating them needs
an arm that equalises what is learned, which is the move that registration forbids us to reach
for after seeing its result.

\paragraph{So we registered a different arm, and it found something else.} The forbidden move was
repairing S-C; a new design frozen in advance is not that move, and
\texttt{prereg/\allowbreak schedule\_control.md} says so in its first section. S-C changed the \emph{path} and
held the schedule. This changes the \emph{schedule} and holds the path \textbf{byte for byte},
which is possible only because the ladder arms disable trainer shuffling: three epochs over the
file is exactly the file's order three times, so writing that order out and splitting it in half
gives two stages whose concatenation is the control's own stream. Same rows, same order, same
$324$ steps, same epochs' worth of gradient; two cosines of $162$ steps where the control has one
of $324$, which is \blocked's structure. The builder asserts the reconstruction rather than
claiming it.

The guard S-C failed is read first and passes: capability moves $0.59$ capability floors at
$L{=}27$ and $0.26$ at $L{=}54$, against a bar of two. The arms are matched in what they learned.

\emph{And the allocation moves a great deal.} On the registered primary rung $L{=}27$, chosen
before either arm ran because its control seed dispersion is $0.0159$ against $L{=}54$'s
$0.0595$, the share falls from $0.4071$ to $0.2912$: $\Delta=-0.116$, $t=-10.1$ on four
degrees of freedom, $6.07$ contrast floors. Table~\ref{tab:schedctl} carries both rungs.

Three consequences follow, the first of which costs us a sentence.

\emph{The stopping phase is not the only non-average term.} The two arms have the identical path
and therefore the identical time-average, which is all Proposition~\ref{prop:avg} sees, and their
endpoints differ by six floors. The abstract's claim is amended to say what is true: the only
non-average term \emph{of the path} is the stopping phase, and the path is not the only input to
the endpoint. Proposition~\ref{prop:avg} is untouched, being a statement about $\eta\to0$ that
says nothing about where a cosine restarts. What was too strong is the reading we hung on it:
that arrangement is the interesting variable because the endpoint is otherwise fixed.

\emph{The confound is closed in the direction it threatened.} A schedule term that explained
\blocked\ would have to push allocation \emph{toward} it. This pushes away, on both rungs, so the
second cosine is not what puts \blocked\ at $0.869$. That is the specific alternative
\S\ref{sec:phase}'s attribution was exposed to, and it is now excluded by measurement rather than
by the plausibility argument this section withdrew.

\emph{And the flat ladder means something different than we said.} The whole block-length family,
$L{=}1$ to $54$, moves the share $0.085$. One cosine restart on a fixed path moves it $0.116$.
Allocation is not rigid; the path is simply not the channel that carries it, which is the
sharpest form of \S\ref{sec:limit}'s two-state reading and was not available to us before this
arm ran.

What we cannot say is \emph{why} the schedule pushes toward $A$. The registered branch is
\textbf{SC-3}, whose definition is that the movement is reported as unexplained and not folded
into the branch that would have closed the question, and we take that in full. Two accounts we
had ready are excluded by the data rather than by argument: the registration predicted that if a
mid-run stopping phase were the mechanism the two rungs would move in \emph{opposite} directions,
because the boundary falls on an $A$ block at $L{=}54$ and a $B$ block at $L{=}27$; they move the
same way. And a bias from which source occupies the restarted cosine's peak fails for the same
reason, since $L{=}54$'s second stage begins on $B$. This is one open cell rather than a
qualification of the three above. It sits alongside
X1 and X2 of \S\ref{sec:family} and the limit-cycle picture's post-hoc origin (its two
deciding tests were frozen and passed after the picture was formed, and are labelled so).

\begin{table}[t]
\centering\small
\setlength{\tabcolsep}{4pt}
\caption{\textbf{The schedule control: identical path, one cosine against two.} The treatment
consumes the control's own stream in the same order and differs only in that the cosine restarts
and the optimiser state resets at the midpoint, which is \blocked's structure. The guard that
voided the single-cosine control is read first and passes on both rungs. $L{=}27$ was fixed as
primary before either arm ran, because its control dispersion is four times smaller. Its
movement is $6.07$ contrast floors at $t=-10.1$. $L{=}54$ moves further and does not clear its
own noise: its treatment seeds span $0.40$, wider than the effect, and the registration recorded
in advance that at this dispersion its silence is not evidence of absence. That is why the
mechanical branch label \textbf{SC-2} is reported with that sentence beside it rather than as
``no effect''. Both rungs move the same way, which the registration fixed in advance as the
signature of the schedule rather than of a mid-run stopping phase. Source
\texttt{schedctl\_verdict}.}
\label{tab:schedctl}
\begin{tabular}{@{}lccccc@{}}
\toprule
 & \multicolumn{2}{c}{share} & & & \\
\cmidrule(lr){2-3}
rung & control (one cosine) & treatment (two) & $\Delta$ & $t$ (crit $2.776$) & guard $|\Delta C|$ \\
\midrule
$L{=}27$ \ \emph{primary} & $0.4071$ & $\mathbf{0.2912}$ & $\mathbf{-0.116}$ & $\mathbf{-10.10}$ & $0.59$ fl \\
 & \footnotesize $0.413,0.419,0.389$ & \footnotesize $0.294,0.302,0.278$ & \footnotesize $6.07$ fl & \textbf{SC-3} & passes \\
\midrule
$L{=}54$ \ \emph{secondary} & $0.4922$ & $0.2053$ & $-0.287$ & $-2.35$ & $0.26$ fl \\
 & \footnotesize $0.424,0.526,0.527$ & \footnotesize $0.024,0.424,0.168$ & \footnotesize $15.02$ fl & SC-2 & passes \\
\midrule
\multicolumn{6}{@{}l@{}}{\footnotesize \emph{for scale: the entire block-length ladder, $L{=}1$ through $54$, moves the share $0.085$}} \\
\bottomrule
\end{tabular}

\end{table}

Two further cells are named by Remark~\ref{rem:wallkind} and are the ones we would run first.
\emph{A difficulty-matched control for the $12.5\%$ residue}: the coverage reading of
\S\ref{sec:key} requires that the unrecovered part is a convention never written rather than
one merely harder, and until that control exists the two-wall decomposition is an
interpretation with a plausible alternative. \emph{A predictable-conflict corpus}: if the
mixture is the optimum of an underdetermined query rather than a limit on what can be stored,
then a conflict of the same size whose convention \emph{is} a function of the input should be
resolved with no key at all, at the same $k$, volume, budget and architecture. That is the
sharpest test of Remark~\ref{rem:wallkind} available, it costs one corpus and four arms, and a
mixture there would refute the reading rather than qualify it.

\emph{That cell has since been run, by the family's evaluation-side member, and the disclosure
belongs here rather than in its pages alone} \citep{mark2026}: a corpus keying the convention on a
single-character feature of the problem statement ran twice at this scale and decided nothing.
The pooled read-out was voided by its own symmetry check; the paired within-problem contrast,
registered afterwards, returned a miss ($\Delta=+0.0018$, $t=0.23$); and the audit of the void
found the keying feature nearly collinear with per-item scoring asymmetry, all $65$ problems of
the discriminating stratum on one side of the feature. The reading of
Remark~\ref{rem:wallkind} is therefore \emph{pending on that corpus, not supported and not
refuted}, and any rebuild must first pass the orthogonality certificate that episode produced
(worst stratum's minority share $\ge0.20$, checked on the two pure arms before any mixed
training).

Two things about any rebuild have to be said now rather than rediscovered, because both would
invalidate it; the first is this paper's own design point and the second is that episode's lesson,
not our foresight. First, the read-out
cannot be Definition~\ref{def:decomp}'s. Setting $H(\mathrm{conv}\mid I)=0$ is exactly the
statement that each held-out problem has one applicable convention, so there is no allocation
left to split: the measurement becomes accuracy against the applicable gold against the rate of
answering under the inapplicable one, and the capability--allocation coordinates this paper is
built on do not survive the construction that tests them. Second, the keying feature must be
uncorrelated with difficulty. Assigning the convention \emph{by skill}, which is the obvious
construction and the one we would have reached for, fails that: skills differ in how often they
are solved at all, so the two golds would sit on problem subsets of unequal difficulty and the
comparison against the unpredictable corpus would not be matched in the one quantity
(\S\ref{sec:key}) it exists to separate from coverage. The feature has to ride on the problem
rather than partition the problems.

\section{Discussion: One Stream, Two Channels}\label{sec:family}

Table~\ref{tab:family} states the correspondence this paper has been using, one row per
load-bearing object. It is a structural identification, not an analogy hunt: in both columns
the object is a bounded state fed by an unbounded stream, the two channels are compression
and indexed access, and disagreement ($\kappa$) is the switch that makes the choice of
channel visible in behaviour.

\begin{table}[t]
\centering\small
\caption{\textbf{The CCH family: the same objects at inference time and training time.}
Left column: \citet{chen2026cch}. Right column: this paper and, for the premise rows, \thecompanion\ recipe-level paper \citep{frl2026}.}
\label{tab:family}
\begin{tabular}{>{\raggedright\arraybackslash}p{6.3cm}>{\raggedright\arraybackslash}p{7.6cm}}
\toprule
inference time (CCH) & training time (this paper) \\
\midrule
bounded state $\Sigma$, budget $B$ & the parameters, reached only through the update path \\
the stream: bindings $+$ distractors, kernel $\kappa$ & the data path; conflict
$\kappa_A\neq\kappa_B$ between sources \\
compressive $O(1)$ channel (state mixing) & the averaging limit: only the path's
time-average is written (\S\ref{sec:average}) \\
horizon wall: a fixed window; deferral, not escape & the stopping phase: recency as the one
non-average term (\S\ref{sec:phase}) \\
verbatim index channel, cost $\Theta(L)$ & a query-time key in the data; selection moves
from path to query (\S\ref{sec:key}) \\
write-time code separability (Assumption 1) & the key selects only among conventions already
written: $0.3767$ vs $0.0025$ (\S\ref{sec:key}) \\
super-additive capability: $\mathcal C_\infty(\mathcal T\sqcup\mathcal S)\supsetneq
\mathcal C_\infty(\mathcal T)\cup\mathcal C_\infty(\mathcal S)$ & keyed corpus reaches
mixture control $+$ per-query selection; unkeyed reaches mixture or corner, strictly less \\
structure is free until $\kappa\neq0$ & recipes are free until the data disagrees with
itself \citep{frl2026} \\
\bottomrule
\end{tabular}

\end{table}

The recipe-level member of the family \citep{frl2026} is the premise map: within a coherent
domain ($\kappa\approx0$) the compressive channel's output does not depend on the path at all, so
composition, order and arrangement sit inside single-run noise. The free-recipe limit is the
averaging limit's special case at zero conflict, and its budget-transient order effect is the
low-frequency end of the axis this paper's ladder climbs. The three papers measure one theory at
three places: what a bounded system can \emph{serve} at query time (CCH), what a recipe can
\emph{move} when premises hold or break (FRL), and what the path \emph{writes} through each
channel (this paper).

\paragraph{The flagship runs on the corpus where one convention is arguably wrong, and the
corpus where neither is runs weaker.} Two constructions are used throughout: \emph{cf2}, in which one
half's boxed answers are shifted by one and are therefore contradictory by construction, and
\emph{nat}, in which one half writes a numeral and the other spells the same number out, so both
are genuinely correct. Every headline figure in this paper is measured on cf2. On the standard
install guard, the minority convention's accuracy in the arm trained on nothing else, cf2 reads
$0.4225$ and nat reads $0.2091$, \emph{less than half}; capability at the interleaved arm is
$0.5202$ against $0.3727$. The philosophical claim of the paper, that a conflict need not be an
error, is a claim about nat, and nat is the weaker instrument.

This neither voids the results nor can be waved through, so here is what it costs. It does not
affect conservation or the allocation coordinate, both measured on each corpus separately and
agreeing in sign and in order of magnitude. It does mean the effect sizes quoted are the cf2 ones,
and the discount a reader rebuilding this on a both-correct conflict should apply is the one
measured in the next paragraph, because the two published numbers differ in step size as well as
in corpus. And it exposes one asymmetry the construction creates: on nat the $A$-\only\ arm
scores \emph{exactly} $0.0000$ on the spelled convention across eight seeds, so the spelled form
is never produced unless it is trained, whereas cf2's shifted form leaks at $0.0169$. A conflict
between two forms the pretrained model already writes and a conflict where one form must be
installed from scratch are not the same experiment, and the second is the one our philosophical
framing is about. Source: \texttt{review\_statistics}.

\paragraph{One premise of that accounting, registered and tested: the size asymmetry is mostly a
budget asymmetry.} The pair above is \texttt{cf2} at lr $3\times10^{-5}$ against \texttt{nat} at lr
$1\times10^{-5}$, so it confounds the corpus with the step size. We registered the cheap half of
the separation before running it
(\texttt{prereg/\allowbreak nat\_\allowbreak budget\_\allowbreak gate.md} with
\texttt{readout\_\allowbreak nat\_\allowbreak budget.py}, frozen $1$h\,08m ahead of the verdict and
untouched since, branches read \textbf{NI-4} to \textbf{NI-1} so that the branch closing the route
is read before the branch opening it) and reran the \texttt{nat} $B$-\only\ arm at \texttt{cf2}'s
learning rate with nothing else changed: same $720$ rows, same three epochs, same $135$ steps,
same base model, a separate tag so no published file is written to. Minority install rises from
$0.2091$ (sd $0.0089$, eight seeds) to $0.3738$ (sd $0.0181$, seeds $42$, $43$, $44$), a move of
$8.62$ contrast floors at $t=15.06$, and capability on that arm rises by $5.09$ capability floors
to $0.4688$; the guard was one-sided against a run destabilised by the larger step, and a rise is
what more budget is supposed to do. That is branch \textbf{NI-1}: $77.2\%$ of the distance to
\texttt{cf2}'s $0.4225$ is closed by the step size alone. A corroborating number was already on
the result base and unread: at the \emph{published} lr $1\times10^{-5}$, \texttt{cf2}'s own install
reads $0.1988$ against \texttt{nat}'s $0.2091$, $0.54$ floors at $t=1.27$ and therefore no
separation, so \emph{the} ``\emph{less than half}'' \emph{above is a statement about two learning
rates before it is a statement about two corpora.} The remainder is stated as narrowly as it is
measured. This is one arm of four and says nothing about the arrangement effect; the
interleaved-arm capability comparison, $0.5202$ against $0.3727$, is still at unmatched budgets and
stands as printed. What is left on the arm we did run, $0.0488$ or $2.55$ contrast floors, sits
within a hair of the family's two-degree-of-freedom critical value ($t=4.55$ against $4.303$) and
three seeds do not resolve it. The accounting above therefore changes in size and not in
direction: on the one arm now matched, the discount is about seven eighths rather than about half.
Putting the switch itself at parity needs \texttt{nat} and \texttt{natc} at four arms and eight
seeds, roughly $56$ GPU-hours, listed as future work item (vii) rather than bought here. Source:
\texttt{nat\_budget\_verdict}, \texttt{nat\_budget\_context}.

\paragraph{Three claims this paper leans on and cannot check.} Every number reported here is
measured here, but three \emph{context} claims are borrowed, and each weakens a different
generalisation. \emph{First, scale.} The conflict switch at $3$B and $14$B is measured by
\thecompanion\ paper \citep{frl2026} and not here, so every arm in this manuscript is $7$B or
$8$B and \emph{this paper on its own establishes nothing about how the switch scales}. That is the
single largest gap in its external validity. \emph{Second, prevalence.} The rate at which public
corpora actually disagree about form, $12.7\%$ of shared problems rising to $27.1\%$ against
externally authored keys, is the survey's \citep{gold2026} and licenses the choice of construction
rather than any number below; if that rate is wrong, this paper measures a real mechanism on a
rare population. \emph{Third, the evaluation harness.} That the sampling stream is common-mode
across seeds is diagnosed by the recipe paper \citep{frl2026} and is why the dispersions here are
harness-specific (\S\ref{sec:average}); an external replication should expect a \emph{higher}
noise floor than $\hat\sigma=0.0234$, which would widen every interval we print and shrink no
effect. None of the three is peer-reviewed at the time of writing, and they are stated as
borrowings rather than cited as support.

\paragraph{Future work, listed at the size we think each one is.} None is run to a readable
answer. Two have been run: the ladder at a constant learning rate (Table~\ref{tab:constlr}) and
item (v), which returned a void and is listed at the size that would close it. A third, item
(vii), has had its cheap premise tested and confirmed. \textbf{(i) Separating the two things a constant rate leaves open.}
Proposition~\ref{prop:sched} attributes the constant-rate family's $11.63$ floors to the terminal
weight $w(T)$; \citet{sweeney2026shuffle} attributes order sensitivity instead to AdamW's buffers
advancing on step count rather than on $\tau=\eta k$. Both predict a resolved ladder at a constant
rate and the present design cannot tell them apart. One arm decides it: the same ladder without
bias correction, or under SGD with momentum, where the fixed clock is absent and only the terminal
weight remains. \textbf{(ii) Eight seeds on the main rung.} Every pairwise statement in
\S\ref{sec:average} is limited by three seeds rather than by the effect, which is cheap to fix.
\textbf{(iii) A conflict outside mathematics.} Every corpus here is competition mathematics with
integer answers, so the convention is surface form; whether a style or a factual conflict
conserves capability the same way is untested. \textbf{(iv) A run that keeps generations.} The
$0.1475$ capability excursion along a trajectory cannot presently be separated into knowledge
briefly lost and scorer briefly unable to parse; retaining generations costs storage and no
compute. \textbf{(v) The mirror at a constant rate, at enough seeds to read it.} This cell has run
and returned \textbf{CM-4}, void: capability moved past the registered bar and the shares are not
comparable (\S\ref{sec:phase}). It needs seeds and not arms: at the capability dispersion the
constant family has, $17$ seeds put a $95\%$ interval on the capability difference inside one
capability floor, against the three this design ran. \textbf{(vi) A difficulty-unweighted
allocation.} $S$ is a $p$-weighted share and $\mathrm{Cov}(p,s)\neq0$ is measured rather than
assumed away (Definition~\ref{def:decomp}); the unweighted $\hat{\mathbb{E}}[s]$ over a fixed
problem set separates policy from population, needs no new training, and should move the arms
together rather than differentially, which is why the coupling is reported as a bound.
\textbf{(vii) The switch itself, at matched budget.} The budget gate above returned
\textbf{NI-1}: matching the learning rate closes $77.2\%$ of the install gap between the two
corpora on the $B$-\only\ arm. That is one arm of four and says nothing about the arrangement
effect. Putting the switch at parity needs \texttt{nat} and \texttt{natc} at four arms and eight
seeds, about $56$ GPU-hours; the gate is what makes the cell worth buying, and this manuscript
does not buy it.

\paragraph{The mean-field predictions were tested as point values, and a directional test
would have been the better instrument.}
The ratio test compares an installed share against $0.500$ and $0.333$ with a frozen tolerance
of $0.05$, and both cells miss on the same side, toward the pretrained prior. Two misses in the
same direction are evidence for a model this test cannot express: mean field \emph{plus} a
prior-bias term of unknown size. A point test with a tolerance cannot separate that from a
failure of mean field, whereas a directional test can, and the design is a third ratio: a
$1{:}1$, $2{:}1$, $4{:}1$ sweep testing monotone decrease and the slope rather than three point
predictions. We did not run it, and we flag the current test as the weaker instrument rather
than reporting its two misses as though the alternative model were not on the table.

\paragraph{Cross-predictions, stated as missing cells.} A unification earns its keep by
predicting across members, and we state the two sharpest as pre-registered future tests
rather than claims. \textbf{X1, the mixture-capacity wall:} CCH's capacity floor
($B/N$ per binding) should have a training-time image, with $N$ mutually exclusive
conventions rather than two, the installed mixture's calibration against the mixture weights
should degrade with $N$ at fixed data per convention; a clean break would be training's
Shannon wall. \textbf{X2, replay as deferral:} CCH's index defers the horizon by $1/\rho$;
mixing a fraction $\rho$ of first-stage data into the final stage is the recipe-level index,
so the commitment flip of \S\ref{sec:phase} should be deferred by the same coefficient, connecting this family to replay in continual learning \citep{lopezpaz2017gem} with a
quantitative prediction rather than a metaphor. Neither is run; both are cheap; either
failing would cut the family at the joint we have named.

\paragraph{Where the correspondence stops.}
Table~\ref{tab:family} identifies objects at inference time and training time and goes no
further. A third position exists and is deliberately not claimed here: a system that writes its
own conclusions into a store it later retrieves from. There the bounded state is neither the
parameters nor the context window but an accumulating record; capability can be conserved by the
store's construction rather than by an averaging limit; and what fixes the allocation is a
retrieval rule rather than a stopping phase. The decomposition of
Eq.~\eqref{eq:decomp} is written for parameters trained once over a fixed corpus, and its
conservation is approximate and measured. Whether the same two coordinates describe an
accumulating store is a question about a different object, and nothing in this paper is evidence
either way. Establishing that two such decompositions are the same object would itself be an
experiment, not a remark.

\section{Conclusion}\label{sec:concl}

An averaging theorem organises everything above, and one factorisation says when it applies.
The ordering-dependent part of an endpoint is bounded by a product in which the arrangement enters
only as a period and the schedule only as a reach (Proposition~\ref{prop:sched}), so the two ends of
the divided record on ordering are one knob at two settings rather than two findings. Below corpus scale, and under the decaying
schedule that is the field's default, the
path is compressed to its time-average, so batch composition, block length and optimiser buffers
move nothing that this design can resolve; the time-average installs the Bayes-optimal mixture,
so conflicting supervision trains a stochastic policy at capability conserved to within an
eighth rather than damaging what the model can do; the one term of the path that is not an
average is where it stops, so blocked training is a stopping phase, a commitment device whose
exact-match victory is its own objective's defeat. A query-time key is a second channel rather
than more bandwidth on the first, which is why it collapses the arrangement effect that no
schedule could.

The claim we would most like a reader to take is the one that costs them nothing to adopt.
Conflicting supervision moves \emph{commitment}, and exact match reports commitment as though it
were capability. Any benchmark whose answers admit more than one correct form is affected, the
correction is the two coordinates of Definition~\ref{def:decomp}, and it needs no experiment from
this paper: only the per-problem scores an evaluation already produces.

\appendix

\section{Supporting Measurements}\label{app:support}

\subsection*{The stopping term in-model: real, and not where the theory says}\label{app:asym}

Proposition~\ref{prop:phase} makes the stopping phase the only $\pi$-dependent term. If that is
right, a linear model in the lazy regime should reproduce the sign reversal we measure along the
budget axis, and it does, without reproducing its \emph{location}, which is a limitation of
the theory rather than of the measurement and is reported here rather than in the main line.

Sampling random non-commuting task pairs inside the lazy regime and sweeping the budget $T$:
$\mathrm{ASYM}(T)$ changes sign in $\mathbf{72\%}$ of draws and $D(T)$ in $\mathbf{70\%}$, so
budget-dependent reversal is generic and not an artefact of our corpora. The rate is stable
across step sizes ($\eta/\eta_{\text{stab}}=0.1$ gives $0.867$ and $0.725$; the sweep is flat to
within a few points across the range). But the leading Magnus/commutator proxy $T^\star$, which
is what a practitioner would compute once at initialisation to \emph{predict} where the
crossing sits, correlates with the measured crossing budget at Spearman $\mathbf{0.031}$ over
$278$ pairs with a crossing. That is, not at all.

\begin{table}[t]
\centering\small
\caption{\textbf{The mechanism survives every test; the one-shot proxy for \emph{where} it fires
survives none.} Upper panel: budget-dependent sign reversal is generic in the lazy regime, so the
stopping term of Proposition~\ref{prop:phase} is not an artefact of our corpora. Lower panel: the
leading Magnus/commutator score, the quantity a practitioner would compute once at
initialisation to predict the crossing: is asked to rank three different things in three
different places and fails all three, twice with the \emph{wrong sign}. Reported together because
they are one settled negative about a single object, not three separate disappointments.}
\label{tab:proxy}
\begin{tabular}{llcc}
\toprule
\multicolumn{4}{l}{\emph{the mechanism, in-model (random non-commuting pairs, lazy regime)}} \\
\midrule
sign reversal of $\mathrm{ASYM}(T)$ & across the budget sweep & & $\mathbf{72\%}$ of draws \\
sign reversal of $D(T)$ & across the budget sweep & & $\mathbf{70\%}$ \\
the same, at $\eta/\eta_{\text{stab}}=0.1$ & step size is not the cause & & $0.867$ / $0.725$ \\
median rotation at the crossing & sum / difference & & $11.2^\circ$ / $11.7^\circ$ \\
\midrule
\multicolumn{4}{l}{\emph{the one-shot proxy, asked to predict where} (Spearman)} \\
\midrule
in-model, $T^\star$ from Magnus & against the measured crossing budget & $278$ & $+0.031$ \\
Qwen2.5-7B activations, $1$ epoch & against $|\mathrm{ASYM}|$, matched pairs & $6$ & $\mathbf{-0.543}$ \\
Qwen2.5-7B activations, $3$ epochs & against $|\mathrm{ASYM}|$, matched pairs & $6$ & $\mathbf{-0.600}$ \\
\bottomrule
\end{tabular}

\end{table}

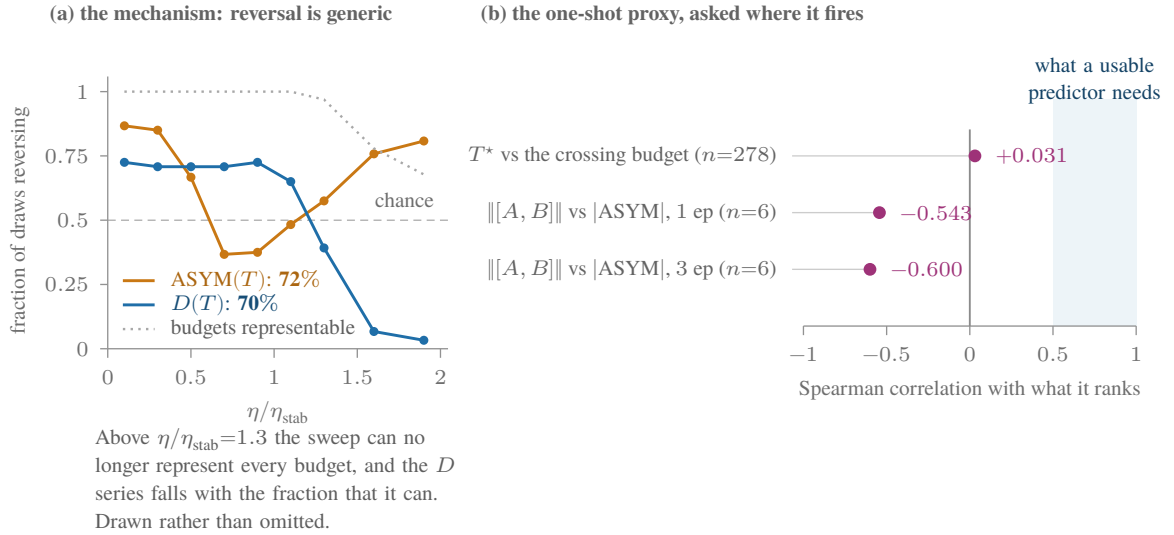
\begin{figure}[t]
\centering
%
%
\begin{tikzpicture}[font=\small,>=Stealth,
 ax/.style={cGrey!75,line width=0.5pt},
 tk/.style={font=\scriptsize,cGrey},
 lb/.style={font=\scriptsize,cGrey!90!black},
 hd/.style={font=\scriptsize\bfseries,cGrey!90!black}]

\begin{scope}
 \node[hd,anchor=west] at (-0.90,4.42) {\textbf{(a)} the mechanism: reversal is generic};
 \draw[ax] (0,0) -- (4.50,0);
 \draw[ax] (0,0) -- (0,3.60);
 \foreach \v/\x in {0/0.0, 0.5/1.1, 1/2.2, 1.5/3.3, 2/4.4}{
   \draw[ax] (\x,0) -- (\x,-0.12); \node[tk,anchor=north] at (\x,-0.14) {$\v$};}
 \foreach \v/\y in {0/0.0, 0.25/0.85, 0.5/1.70, 0.75/2.55, 1/3.40}{
   \draw[ax] (0,\y) -- (-0.10,\y); \node[tk,anchor=east] at (-0.12,\y) {$\v$};}
 \node[tk,anchor=north] at (2.25,-0.56) {$\eta/\eta_{\text{stab}}$};
 \node[tk,rotate=90,anchor=south] at (-0.90,1.70) {fraction of draws reversing};
 \draw[cGrey!70,densely dashed] (0,1.70) -- (4.50,1.70);
 \node[tk,anchor=south east] at (4.45,1.74) {chance};
 \draw[cA,line width=1.1pt] plot coordinates {(0.22,2.948) (0.66,2.890) (1.10,2.268) (1.54,1.248) (1.98,1.275) (2.42,1.642) (2.86,1.955) (3.52,2.577) (4.18,2.747)};
 \foreach \x/\y in {0.22/2.948, 0.66/2.890, 1.10/2.268, 1.54/1.248, 1.98/1.275, 2.42/1.642, 2.86/1.955, 3.52/2.577, 4.18/2.747}
   \fill[cA] (\x,\y) circle (1.7pt);
 \draw[cV,line width=1.1pt] plot coordinates {(0.22,2.465) (0.66,2.407) (1.10,2.407) (1.54,2.407) (1.98,2.465) (2.42,2.210) (2.86,1.333) (3.52,0.228) (4.18,0.112)};
 \foreach \x/\y in {0.22/2.465, 0.66/2.407, 1.10/2.407, 1.54/2.407, 1.98/2.465, 2.42/2.210, 2.86/1.333, 3.52/0.228, 4.18/0.112}
   \fill[cV] (\x,\y) circle (1.7pt);
 \draw[cGrey!60,dotted,line width=0.9pt] plot coordinates {(0.22,3.400) (0.66,3.400) (1.10,3.400) (1.54,3.400) (1.98,3.400) (2.42,3.400) (2.86,3.295) (3.52,2.649) (4.18,2.305)};
 \draw[cA,line width=1.1pt] (0.20,0.90) -- (0.62,0.90);
 \node[lb,cA!85!black,anchor=west] at (0.70,0.90) {$\mathrm{ASYM}(T)$: $\mathbf{72\%}$};
 \draw[cV,line width=1.1pt] (0.20,0.58) -- (0.62,0.58);
 \node[lb,cV!70!black,anchor=west] at (0.70,0.58) {$D(T)$: $\mathbf{70\%}$};
 \draw[cGrey!60,dotted,line width=0.9pt] (0.20,0.26) -- (0.62,0.26);
 \node[lb,cGrey!85!black,anchor=west] at (0.70,0.26) {budgets representable};
 \node[lb,cGrey!85!black,anchor=north west,align=left,text width=4.8cm] at (-0.30,-0.92)
  {Above $\eta/\eta_{\text{stab}}{=}1.3$ the sweep can no longer represent every budget, and the
   $D$ series falls with the fraction that it can. Drawn rather than omitted.};
\end{scope}

\begin{scope}[shift={(9.20,0)}]
 \node[hd,anchor=west] at (-4.40,4.42) {\textbf{(b)} the one-shot proxy, asked where it fires};
 \fill[cV!8] (3.30,0.30) rectangle (4.40,3.30);
 \node[lb,cV!60!black,anchor=south,align=center,text width=1.9cm] at (3.85,3.10) {what a usable predictor needs};
 \draw[ax] (0,0.30) -- (4.40,0.30);
 \foreach \v/\x in {-1/0.0, -0.5/1.1, 0/2.2, 0.5/3.3, 1/4.4}{
   \draw[ax] (\x,0.30) -- (\x,0.18); \node[tk,anchor=north] at (\x,0.16) {$\v$};}
 \node[tk,anchor=north] at (2.20,-0.28) {Spearman correlation with what it ranks};
 \draw[cGrey!80,line width=0.7pt] (2.2,0.30) -- (2.2,2.90);
 \foreach \y/\nm/\x/\vl in {%
  2.55/{$T^\star$ vs the crossing budget ($n{=}278$)}/2.268/{$+0.031$},
  1.80/{$\|[A,B]\|$ vs $|\mathrm{ASYM}|$, $1$ ep ($n{=}6$)}/1.005/{$-0.543$},
  1.05/{$\|[A,B]\|$ vs $|\mathrm{ASYM}|$, $3$ ep ($n{=}6$)}/0.880/{$-0.600$}}
 {
  \draw[cGrey!35,line width=0.5pt] (-0.15,\y) -- (\x,\y);
  \fill[cD] (\x,\y) circle (2.4pt);
  \node[lb,anchor=east] at (-0.22,\y) {\nm};
  \node[lb,cD,anchor=west] at (\x+0.14,\y) {\vl};
 }
\end{scope}
\end{tikzpicture}
\caption{\textbf{One object, confirmed as a mechanism and failed as a predictor.} The same two
halves as Table~\ref{tab:proxy}, drawn, because both halves are claims about a distribution.
\textbf{(a)} Over random non-commuting pairs in the lazy regime, $\mathrm{ASYM}(T)$ reverses sign
in $72\%$ of draws and $D(T)$ in $70\%$, and the rate stays above chance across a $19\times$
sweep of the step size, so the reversal is something the finite-budget theory produces rather
than a property of our corpora. The dotted series is the fraction of budgets the sweep can still
represent, which falls above $\eta/\eta_{\text{stab}}=1.3$ and is why the $D$ series falls with
it; it is drawn rather than left out. \textbf{(b)} Every ranking the one-shot commutator score is
asked for, on one axis. It is uncorrelated with the in-model crossing budget over $278$ pairs and
\emph{inverted} against the measured order effect at both budgets. Sources
\texttt{asym\_budget\_verdict}, \texttt{commutator\_verdict}.}
\label{fig:asym}
\end{figure}

We take the honest reading: the mechanism is confirmed in-model, the schedule of the mechanism
is not. A single evaluation at $\theta_0$ is the first term of a series whose error accumulates
with the budget, which is also what \citet{pairwisefragile2026} observe empirically as
pairwise-order accuracy decaying with block length. Median rotations at the crossing are small
($11.2^\circ$ on the sum, $11.7^\circ$ on the difference), so the reversal is not a large
geometric event; it is a near-cancellation, which is exactly why its location is hard.

\subsection*{The same effect at three budgets, and at three learning rates}\label{app:budget}

\paragraph{The published order effect does not survive its own budget.} On
algebra--combinatorics, $\mathrm{ASYM}$ reads $+0.0779$ at the published budget and
$\mathbf{-0.0619}$ when the budget is raised, three seeds, $\mathrm{s.e.}=0.0191$: a
reversal, not an attenuation. The probe explains why the two budgets are different regimes at
all: at one epoch and lr $3{\times}10^{-5}$, fine-tuning \emph{damages} three of the five
skills relative to zero-shot (calculus $0.304\to0.281$, and similarly for two others), so the
published regime is one in which training is net-negative on part of the corpus. An order effect
measured there is a statement about that regime. Reported as a limitation of the published
number, and it is the sharpest single reason this paper reports budgets rather than recipes.

\paragraph{Learning rate moves the conflict contrast monotonically.} Table~\ref{tab:lrsweep}
carries the sweep.

\begin{table}[H]
\centering\small
\caption{\textbf{The switch grows with the learning rate, monotonically, and the arm that is
supposed to be inert does not move with it.} Three learning rates at three epochs on the natural
conflict, three seeds each; $\bar D$ is the mean allocation contrast and
$\overline{\mathrm{acc}}_{\only}$ the mean single-convention accuracy, which is what a capacity
explanation would have to move and does not. Raising the rate is the one intervention that takes
the system further from the lazy regime the theory is derived in, and $|D|$ grows with it: which is the direction the suppression argument predicts.}
\label{tab:lrsweep}
\begin{tabular}{lccc}
\toprule
learning rate & $\bar D$ & per-seed $D$ & $\overline{\mathrm{acc}}_{\only}$ \\
\midrule
$1\times10^{-5}$ & $-0.0236$ & $-0.0125,\;-0.0188,\;-0.0396$ & $0.3014$ \\
$3\times10^{-5}$ & $+0.0507$ & $+0.0229,\;+0.0833,\;+0.0459$ & $0.3396$ \\
$1\times10^{-4}$ & $+0.0653$ & $+0.0729,\;+0.0604,\;+0.0625$ & $0.2680$ \\
\bottomrule
\end{tabular}

\end{table}

\noindent $|D|$ grows monotonically, $0.0236\to0.0653$, as the step size grows: the direction
the suppression argument requires, since raising the learning rate is the one intervention that
moves the system furthest from the lazy regime in which arrangement is doubly suppressed. Note
that the highest rate also costs capability ($0.268$ against $0.340$), so the growth in $|D|$ is
not a free improvement: it is the system leaving the regime in which the averaging argument is
tight.

\subsection*{The commutator proxy, measured and rejected}\label{app:commutator}

The natural predictor of which order wins is the commutator of the two tasks' update operators,
evaluated once at $\theta_0$. We measured it on Qwen2.5-7B activations at layer fraction $0.75$
across the matched pairs. It ranks $|\mathrm{ASYM}|$ at Spearman $\mathbf{-0.543}$ at one epoch
and $\mathbf{-0.600}$ at three: the wrong sign, consistently, on both budgets. Four of the
six pairs flip sign between budgets, so there is no fixed ordering for a static proxy to
predict (Figure~\ref{fig:proxy}, Table~\ref{tab:pairs}).

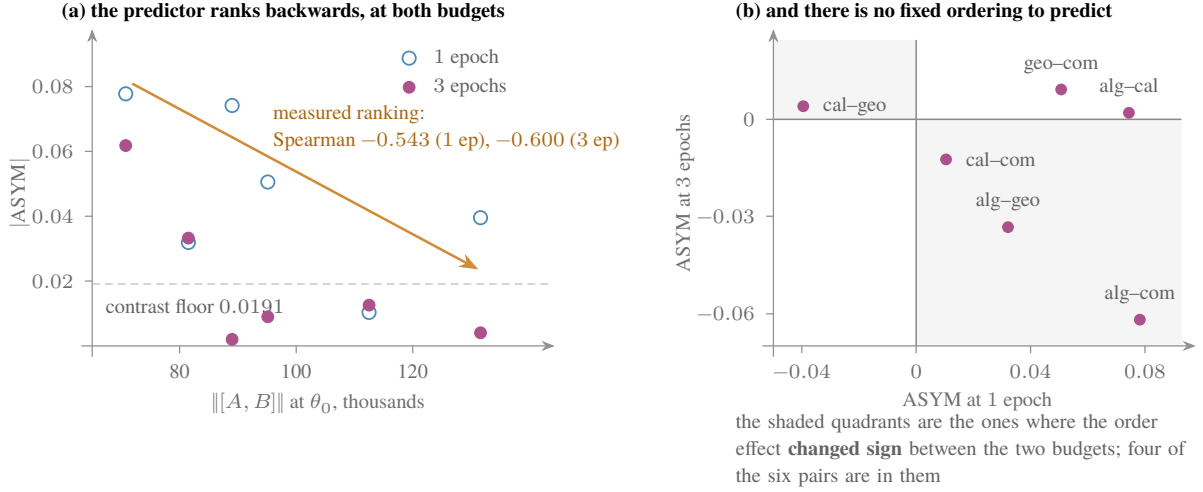
\begin{figure}[t]
\centering
\resizebox{\linewidth}{!}{
%
%
%
\begin{tikzpicture}[font=\small,>=Stealth,
 ax/.style={cGrey!75,line width=0.5pt},
 tk/.style={font=\scriptsize,cGrey},
 e1/.style={draw=cV!80,fill=white,line width=0.7pt},
 e3/.style={draw=cD!85,fill=cD!85},
 lb/.style={font=\scriptsize,cGrey!90!black}]

\begin{scope}
 \draw[ax,->] (-0.15,0) -- (6.35,0);
 \draw[ax,->] (0,-0.15) -- (0,4.35);
 \foreach \v/\x in {80/1.20, 100/2.80, 120/4.40}{
   \draw[ax] (\x,0) -- (\x,-0.10); \node[tk,anchor=north] at (\x,-0.12) {\v};}
 \foreach \v/\y in {0.02/0.89, 0.04/1.78, 0.06/2.67, 0.08/3.56}{
   \draw[ax] (0,\y) -- (-0.10,\y); \node[tk,anchor=east] at (-0.12,\y) {$\v$};}
 \node[tk,anchor=north] at (3.1,-0.50) {$\|[A,B]\|$ at $\theta_0$, thousands};
 \node[tk,rotate=90,anchor=south] at (-0.72,2.1) {$|\mathrm{ASYM}|$};

 \draw[cGrey!55,densely dashed] (0,0.85) -- (6.3,0.85);
 \node[lb,anchor=north west] at (0.05,0.78) {contrast floor $0.0191$};

 \foreach \x/\y in {1.92/3.30, 1.32/1.42, 0.46/3.46, 5.33/1.76, 3.80/0.46, 2.41/2.25}
   \draw[e1] (\x,\y) circle (2.6pt);
 \foreach \x/\y in {1.92/0.09, 1.32/1.48, 0.46/2.75, 5.33/0.18, 3.80/0.56, 2.41/0.40}
   \fill[e3] (\x,\y) circle (2.3pt);

 \draw[cA!85,line width=1pt,->] (0.55,3.60) -- (5.30,1.05);
 \node[lb,cA!85!black,anchor=south west,align=left] at (2.35,2.55)
  {measured ranking:\\ Spearman $-0.543$ ($1$ ep), $-0.600$ ($3$ ep)};

 \draw[e1] (4.35,3.95) circle (2.6pt);
 \node[lb,anchor=west] at (4.55,3.95) {$1$ epoch};
 \fill[e3] (4.35,3.55) circle (2.3pt);
 \node[lb,anchor=west] at (4.55,3.55) {$3$ epochs};
 \node[font=\scriptsize\bfseries,anchor=north west] at (-0.55,4.85)
  {\textbf{(a)} the predictor ranks backwards, at both budgets};
\end{scope}

\begin{scope}[shift={(9.35,0)}]
 \fill[cGrey!7] (0,3.11) rectangle (1.96,4.20);      
 \fill[cGrey!7] (1.96,0) rectangle (5.60,3.11);      
 \draw[ax,->] (-0.15,0) -- (5.80,0);
 \draw[ax,->] (0,-0.15) -- (0,4.35);
 \draw[cGrey!80,line width=0.6pt] (1.96,0) -- (1.96,4.20);
 \draw[cGrey!80,line width=0.6pt] (0,3.11) -- (5.60,3.11);
 \foreach \v/\x in {-0.04/0.39, 0/1.96, 0.04/3.53, 0.08/5.10}{
   \draw[ax] (\x,0) -- (\x,-0.10); \node[tk,anchor=north] at (\x,-0.12) {$\v$};}
 \foreach \v/\y in {-0.06/0.44, -0.03/1.78, 0/3.11}{
   \draw[ax] (0,\y) -- (-0.10,\y); \node[tk,anchor=east] at (-0.12,\y) {$\v$};}
 \node[tk,anchor=north] at (2.8,-0.50) {$\mathrm{ASYM}$ at $1$ epoch};
 \node[tk,rotate=90,anchor=south] at (-0.95,2.1) {$\mathrm{ASYM}$ at $3$ epochs};

 \foreach \x/\y/\n/\a/\dx/\dy in {4.88/3.20/{alg--cal}/south/0/0.09,
                                    3.22/1.63/{alg--geo}/south/0/0.09,
                                    5.03/0.36/{alg--com}/south/0/0.09,
                                    0.41/3.29/{cal--geo}/west/0.14/0,
                                    2.37/2.56/{cal--com}/west/0.14/0,
                                    3.95/3.52/{geo--com}/south/0/0.09}{
   \fill[cD!85] (\x,\y) circle (2.3pt);
   \node[lb,anchor=\a] at (\x+\dx,\y+\dy) {\n};}

 \node[lb,anchor=north west,align=left,text width=6.20cm] at (-0.65,-0.80)
  {the shaded quadrants are the ones where the order effect \textbf{changed sign} between the
   two budgets; four of the six pairs are in them};
 \node[font=\scriptsize\bfseries,anchor=north west] at (-0.65,4.85)
  {\textbf{(b)} and there is no fixed ordering to predict};
\end{scope}
\end{tikzpicture}}
\caption{\textbf{The commutator proxy fails twice, and the two failures are different.}
\textbf{(a)} The score a practitioner would compute once at $\theta_0$ ranks the pairs
\emph{backwards}, at both budgets (Spearman $-0.543$ at one epoch, $-0.600$ at three). A weak
predictor scatters about zero; this one is consistently inverted, which is stronger than
uninformative. The dashed line is the contrast floor. \textbf{(b)} The deeper failure: four of
six pairs change the \emph{sign} of their order effect between budgets, so there is no fixed
ordering for any once-computed quantity to predict. Source \texttt{commutator\_verdict}.}
\label{fig:proxy}
\end{figure}

\begin{table}[t]
\centering\small
\setlength{\tabcolsep}{5pt}
\caption{\textbf{The six pairs the proxy is asked to rank, and the ranking it produces.} The
commutator norm is evaluated once at $\theta_0$ on Qwen2.5-7B activations at layer fraction
$0.75$; $T^\star$ is the leading Magnus budget built from it; $\mathrm{ASYM}$ is the measured
order effect, three seeds, contrast floor $0.0191$. Two things are fatal to a static predictor
and both are visible by eye: the commutator ranks $|\mathrm{ASYM}|$ with the \emph{wrong sign} at
both budgets, and four of six pairs change that sign between budgets, so the quantity being
ranked is not a property of the pair alone. Source \texttt{commutator\_verdict}.}
\label{tab:pairs}
\begin{tabular}{lrrrrc}
\toprule
pair & $\|[A,B]\|$ & $T^\star$ & $\mathrm{ASYM}$ ($1$ ep) & $\mathrm{ASYM}$ ($3$ ep) & sign flips? \\
\midrule
algebra--calculus       & $89{,}002$  & $0.381$ & $+0.0743$ & $+0.0020$ & \\
algebra--geometry       & $81{,}524$  & $0.459$ & $+0.0320$ & $-0.0333$ & yes \\
algebra--combinatorics  & $70{,}810$  & $0.281$ & $+0.0779$ & $-0.0619$ & yes \\
calculus--geometry      & $131{,}642$ & $0.319$ & $-0.0396$ & $+0.0041$ & yes \\
calculus--combinatorics & $112{,}524$ & $0.355$ & $+0.0104$ & $-0.0125$ & yes \\
geometry--combinatorics & $95{,}175$  & $0.361$ & $+0.0506$ & $+0.0091$ & \\
\midrule
\multicolumn{3}{@{}l}{Spearman, $\|[A,B]\|$ against $|\mathrm{ASYM}|$}
 & $\mathbf{-0.543}$ & $\mathbf{-0.600}$ & $4/6$ flip \\
\multicolumn{3}{@{}l}{Spearman, $T^\star$ against $|\mathrm{ASYM}|$} & --- & $-0.200$ & \\
\bottomrule
\end{tabular}

\end{table}

Taken with Appendix~\ref{app:asym}'s in-model result (Figure~\ref{fig:asym}) (the same proxy at Spearman $0.031$ against the crossing budget, and all three rankings collected in Table~\ref{tab:proxy}) the
conclusion is that a one-shot geometric score is not a usable
predictor of order effects at these budgets, in the model or in the measurement. We report this
because the proxy is the obvious thing to try and because our own framework motivates it.

\emph{The scope of that sentence is ``at these budgets'', and it is doing work.} We measure
$\|[A,B]\|$ once at $\theta_0$ over blocks of $162$--$324$ steps; \citet{sweeney2026liebracket}
scores a different statistic and already reports its accuracy falling from $98\%$ to $73\%$ between
$k{=}1$ and $k{=}20$, and \citet{pairwisefragile2026} reports the same decay independently. Nothing
here contradicts either. What the two rankings above add is that the decay does not stop at
indifference: past some block length the score is \emph{inverted}, which is the regime a
practitioner choosing a fine-tuning order is actually in. Whether the tournament construction
inverts where our scalar proxy does is a question about that construction and we have not measured
it.

\subsection*{The positive control's first build, in full}\label{app:v1}

\S\ref{sec:design} summarises the inverted first build. The full read-out, for a reader checking
whether the redesign was principled or opportunistic: v1 drew $238$ rows per condition from
algebra alone at one epoch, and the conflict arms scored $0.02$ against the shifted golds while
the same checkpoints scored $0.39$--$0.50$ against unshifted ones. That is the diagnostic: the
$+1$ convention was never installed, so $D$ was measured on an instrument whose treatment was
absent, and the control's $2.41\sigma$ was the only thing in the experiment with any variance to
report. The redesign changed two things, both stated before the rerun: pool the integer-answer
rows of every skill (raising $238$ to $864$) and raise the epochs, together giving the
conflicting convention roughly an order of magnitude more gradient steps. Nothing about the
read-out rule changed. Table~\ref{tab:v1} shows why the first build read at the floor, and
Figure~\ref{fig:v1} draws the two halves of it.

\begin{table}[H]
\centering\small
\setlength{\tabcolsep}{5pt}
\caption{\textbf{Why v1 read at the floor: the treatment was never installed.} Every conflict arm
is scored twice, against the shifted golds the $+1$ convention defines and against the unshifted
ones the pretrained model already produces. Had the conflicting convention been learned, the
first column would carry the mass. It does not, so $D$ in v1 was a difference between two arms
neither of which had received the treatment, and the control's $2.41\sigma$ was the only variance
in the experiment. Source the \texttt{conflict\_switch} family, three seeds, $238$ rows per
condition at one epoch.}
\label{tab:v1}
\begin{tabular}{llcc}
\toprule
build & arm & against the \emph{shifted} gold & against the \emph{unshifted} gold \\
\midrule
v1 (algebra pool, $1$ epoch) & conflict arms & $\approx0.02$ & $0.39$--$0.50$ \\
 & \multicolumn{3}{l}{\footnotesize $\Rightarrow$ the $+1$ convention lost to the pretrained prior; $D$ measured at the instrument's floor} \\
\midrule
\multicolumn{4}{@{}l}{\emph{the read-out that followed, unchanged in rule}}\\
v1 & control $D$ & \multicolumn{2}{l}{$+0.0461$ \ ($2.41\sigma$) \quad --- the arm that should have been inert moved} \\
v1 & conflict $D$ & \multicolumn{2}{l}{$+0.0020$ \ ($0.10\sigma$) \quad --- the arm that should have moved did not} \\
\bottomrule
\end{tabular}

\end{table}

\begin{figure}[t]
\centering
%
%
\begin{tikzpicture}[font=\small,>=Stealth,
 ax/.style={cGrey!75,line width=0.5pt},
 tk/.style={font=\scriptsize,cGrey},
 lb/.style={font=\scriptsize,cGrey!90!black},
 hd/.style={font=\scriptsize\bfseries,cGrey!90!black}]

\begin{scope}
 \def\S{6.545}                      
 \node[hd,anchor=west] at (-0.75,4.30) {\textbf{(a)} v1: the conflict arms, scored twice};
 \draw[ax,->] (0,0) -- (0,3.95);
 \foreach \v in {0.1,0.2,0.3,0.4,0.5}{
   \draw[ax] (0,\v*\S) -- (-0.09,\v*\S); \node[tk,anchor=east] at (-0.11,\v*\S) {$\v$};}
 \node[tk,rotate=90,anchor=south] at (-0.72,1.8) {exact match};
 \fill[cD!70] (0.55,0) rectangle (1.75,0.02*\S);
 \draw[cD,line width=0.7pt] (0.55,0.02*\S) -- (1.75,0.02*\S);
 \node[lb,cD,anchor=south,align=center] at (1.15,0.10) {$\approx0.02$};
 \node[lb,anchor=north,align=center,text width=1.9cm] at (1.15,-0.10)
  {against the \emph{shifted} gold the treatment defines};
 \fill[cA!45] (2.55,0.39*\S) rectangle (3.75,0.50*\S);
 \fill[cA!75] (2.55,0) rectangle (3.75,0.39*\S);
 \node[lb,cA!85!black,anchor=south,align=center] at (3.15,0.50*\S+0.06) {$0.39$--$0.50$};
 \node[lb,anchor=north,align=center,text width=1.9cm] at (3.15,-0.10)
  {against the \emph{unshifted} gold the base already writes};
 \node[lb,cGrey!85!black,anchor=west,align=left,text width=2.5cm] at (4.05,1.90)
  {the $+1$ convention lost to the pretrained prior, so $D$ was read on an instrument whose treatment was absent};
\end{scope}

\begin{scope}[shift={(8.28,0)}]
 \def\T{0.2769}                     
 \node[hd,anchor=west] at (-0.85,4.30) {\textbf{(b)} the read-out, unchanged in rule};
 \draw[ax,->] (0,0) -- (0,3.95);
 \foreach \v in {0,4,8,12}{
   \draw[ax] (0,\v*\T) -- (-0.09,\v*\T); \node[tk,anchor=east] at (-0.11,\v*\T) {$\v$};}
 \node[tk,rotate=90,anchor=south] at (-0.62,1.8) {$|D|$ in seed s.e.};
 \draw[cGrey!70,densely dashed] (0,2*\T) -- (5.55,2*\T);
 \node[tk,anchor=south west] at (2.72,0.60) {noise, $2\,\mathrm{s.e.}$};
 \fill[cGrey!45] (0.45,0) rectangle (1.15,2.41*\T);
 \node[lb,anchor=south,align=center] at (0.80,2.41*\T+0.06) {$2.41$};
 \fill[cA!75] (1.25,0) rectangle (1.95,0.10*\T);
 \draw[cA,line width=0.7pt] (1.25,0.10*\T) -- (1.95,0.10*\T);
 \node[lb,cA!85!black,anchor=south,align=center] at (1.60,0.16) {$0.10$};
 \node[lb,anchor=north,align=center,text width=2.0cm] at (1.20,-0.10) {\textbf{v1}: the inert arm moved and the treated one did not};
 \fill[cGrey!45] (3.35,0) rectangle (4.05,0.19*\T);
 \node[lb,anchor=south,align=center] at (3.70,0.19*\T+0.06) {$0.19$};
 \fill[cV!75] (4.15,0) rectangle (4.85,11.91*\T);
 \node[lb,cV!70!black,anchor=south,align=center] at (4.50,11.91*\T+0.06) {$\mathbf{11.91}$};
 \node[lb,anchor=north,align=center,text width=2.0cm] at (4.10,-0.10) {\textbf{v3}: the rebuild, three seeds, same rule};
 \draw[cGrey!60,line width=0.7pt] (2.65,0) -- (2.65,3.60);
 \fill[cGrey!45] (5.05,3.62) rectangle (5.25,3.78);
 \node[lb,anchor=west] at (5.32,3.70) {control};
 \fill[cV!75] (5.05,3.30) rectangle (5.25,3.46);
 \node[lb,anchor=west] at (5.32,3.38) {conflict};
\end{scope}
\end{tikzpicture}
\caption{\textbf{An instrument reporting a null because its treatment was absent.}
\textbf{(a)} v1's conflict arms scored against both golds. The $+1$ convention the arms were
trained on is worth ${\approx}0.02$; the unshifted convention the base already writes is worth
$0.39$--$0.50$. Had the treatment installed, the first bar would carry the mass.
\textbf{(b)} What the read-out rule then returned, in units of the seed standard error, and what
the same rule returned after the redesign. In v1 the arm registered to stay inert moved $2.41$
and the arm registered to move returned $0.10$: the control is the only thing in the experiment
with variance to report. The v3 pair is the three-seed read-out, matched to v1's three seeds
rather than the eight-seed pair Table~\ref{tab:ledger} prints. Nothing about the rule changed
between them. Sources \texttt{conflict\_switch\_verdict},
\texttt{conflict\_switch2b3lr3e5\_verdict}.}
\label{fig:v1}
\end{figure}
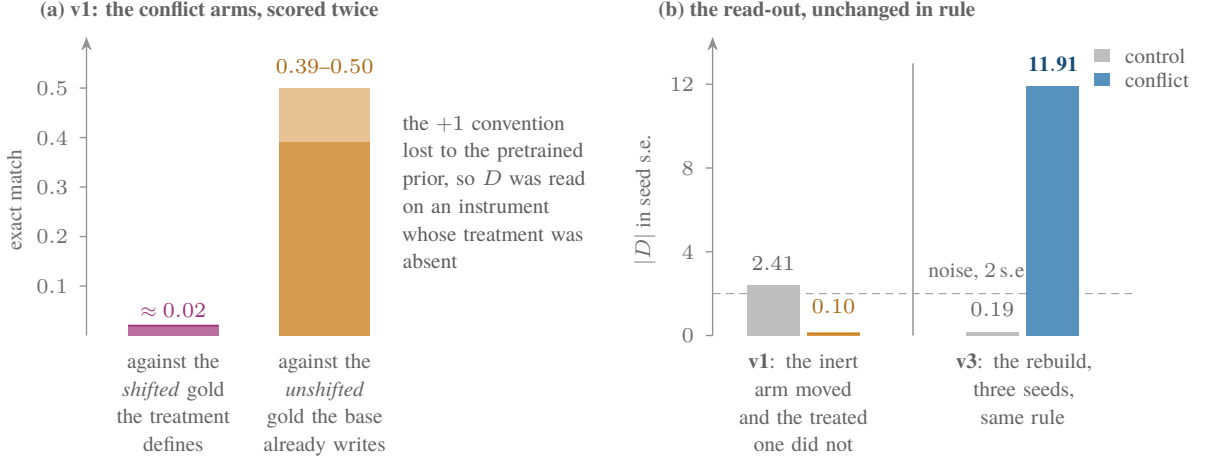

\section{The Clustering Rule, and Why It Is Not a Channel Rate}\label{app:clustering}

This appendix carries the endpoint-clustering rule \S\ref{sec:limit} uses, the sweep of the one constant it depends on, and the reason the counts are reported as counts.

\begin{definition}[The write channel, its capacity, and its realised rate]\label{def:channel}
The path carries $\log_2\binom{n_A+n_B}{n_A}$ bits of source-order information. Let
$\Phi(\pi)\in\mathbb{R}^d$ be a behavioural read-out of the endpoint and $\hat\sigma$ the seed
dispersion. Two endpoints count as \emph{distinguishable} when they are separated by
$2\hat\sigma$. The factor of two is a convention and the headline depends on it, so
Table~\ref{tab:ratesens} sweeps it and Figure~\ref{fig:clusters} draws the sweep. The two rates below respond to that sweep differently:
halving the resolution adds exactly one bit to $R_{\mathrm{cap}}$, which is a logarithm of it,
and does not to $R_{\mathrm{real}}$, which counts clusters, so at $1\hat\sigma$ the first family
resolves a third state and $R_{\mathrm{real}}$ reads $1.58$ bits rather than $2$. The
\emph{write channel} is the induced map $\pi\mapsto\Phi(\pi)$, and it has two rates that must
be kept apart.
\begin{align*}
\text{capacity}\quad R_{\mathrm{cap}}
 &= \log_2\bigl(\operatorname{range}\Phi/2\hat\sigma\bigr),\\
\text{realised rate}\quad R_{\mathrm{real}}
 &= \log_2\bigl|\Phi(\Pi)/{\sim}_{2\hat\sigma}\bigr|,
\end{align*}
where $\Pi$ is the family of paths actually run and ${\sim}_{2\hat\sigma}$ is single-linkage at
the same resolution. $R_{\mathrm{cap}}$ counts the distinguishable values the \emph{range}
supports; it is the rate an encoder would achieve if it could place an endpoint anywhere in
that range. $R_{\mathrm{real}}$ counts the distinguishable values the path family
\emph{occupies}, and it is an upper bound on the information any encoder built from this
family can transmit, however many paths it uses.
\end{definition}

\noindent\textbf{Three things this definition is not, stated before it is used.} First,
$\log_2\binom{n_A+n_B}{n_A}$ is the entropy of the design space under a uniform measure over
\emph{all} orderings. We ran ten. An ensemble of ten codewords carries at most
$\log_2 10=3.32$ bits, so no measurement in this paper can exhibit a compression larger than
$3.32{:}1$, and the ratio of the design-space entropy to a measured cluster count is not a
compression measurement, and the design-space entropy appears in this paper only as the size of
the space the ladder samples. Second, neither rate
is a mutual information. $I(\pi;\Phi)$ would require an explicit prior on paths and a noise
model for seed dispersion, and we estimate neither; $R_{\mathrm{real}}$ is a count of
resolvable clusters wearing a logarithm, and the logarithm is a naming convention rather than a
result. Third, $\Phi$ here is one scalar. A degenerate image in one coordinate is not a
degenerate image, and \S\ref{sec:threats} lists the coordinates we did not read. The
statement this paper defends needs none of this vocabulary: \emph{ten paths spanning $0.4621$ in
allocation, at a resolution of $2\hat\sigma=0.0468$ that their own span would divide into about
ten distinguishable values, occupy two}. Source: \texttt{review\_statistics}.

\noindent The two coincide only when the image is spread. Proposition~\ref{prop:avg} says it
is not, so the paper's central measurement is the \emph{gap} between them
(\S\ref{sec:average}), and quoting $R_{\mathrm{cap}}$ as though it were $R_{\mathrm{real}}$
was the error the gap corrects.

\begin{table}[t]
\centering\small
\caption{\textbf{The headline pair, swept over the resolution constant it is defined at.}
$R_{\mathrm{real}}$ counts the endpoint states the ten arms occupy under single linkage;
$R_{\mathrm{cap}}$ is the logarithm of the range in units of the resolution and therefore gains
exactly one bit per halving, which $R_{\mathrm{real}}$ does not. \emph{Two clusters is what both
families give from $2\hat\sigma$ upward}, and the second family gives it from $1\hat\sigma$
upward. Below $2\hat\sigma$ the first family resolves a third state, so the paper's constant sits
at the \emph{lower edge} of the interval where the two families agree rather than in its middle,
rather than presenting $1$ bit as resolution-free. The interval is wide above the
choice and the between-cluster gap exceeds the largest within-cluster gap by $8.7\times$ on the
first family and $20.7\times$ on the second, which is why the count is stable there. Source:
\texttt{rate\_sensitivity}.}
\label{tab:ratesens}
\setlength{\tabcolsep}{5pt}
\begin{tabular}{@{}lcccccc@{}}
\toprule
resolution & $0.5\hat\sigma$ & $1\hat\sigma$ & $1.5\hat\sigma$ & $\mathbf{2}\hat\sigma$ &
$3\hat\sigma$ & $4\hat\sigma$ \\
\midrule
Qwen2.5-7B, clusters & $4$ & $3$ & $3$ & $\mathbf{2}$ & $2$ & $2$ \\
\quad $R_{\mathrm{real}}$ (bits) & $2.00$ & $1.58$ & $1.58$ & $\mathbf{1.00}$ & $1.00$ & $1.00$ \\
\quad $R_{\mathrm{cap}}$ (bits) & $5.30$ & $4.30$ & $3.72$ & $\mathbf{3.30}$ & $2.72$ & $2.30$ \\
\addlinespace[2pt]
Qwen3-8B-Base, clusters & $4$ & $2$ & $2$ & $\mathbf{2}$ & $2$ & $2$ \\
\quad $R_{\mathrm{real}}$ (bits) & $2.00$ & $1.00$ & $1.00$ & $\mathbf{1.00}$ & $1.00$ & $1.00$ \\
\bottomrule
\end{tabular}

\end{table}

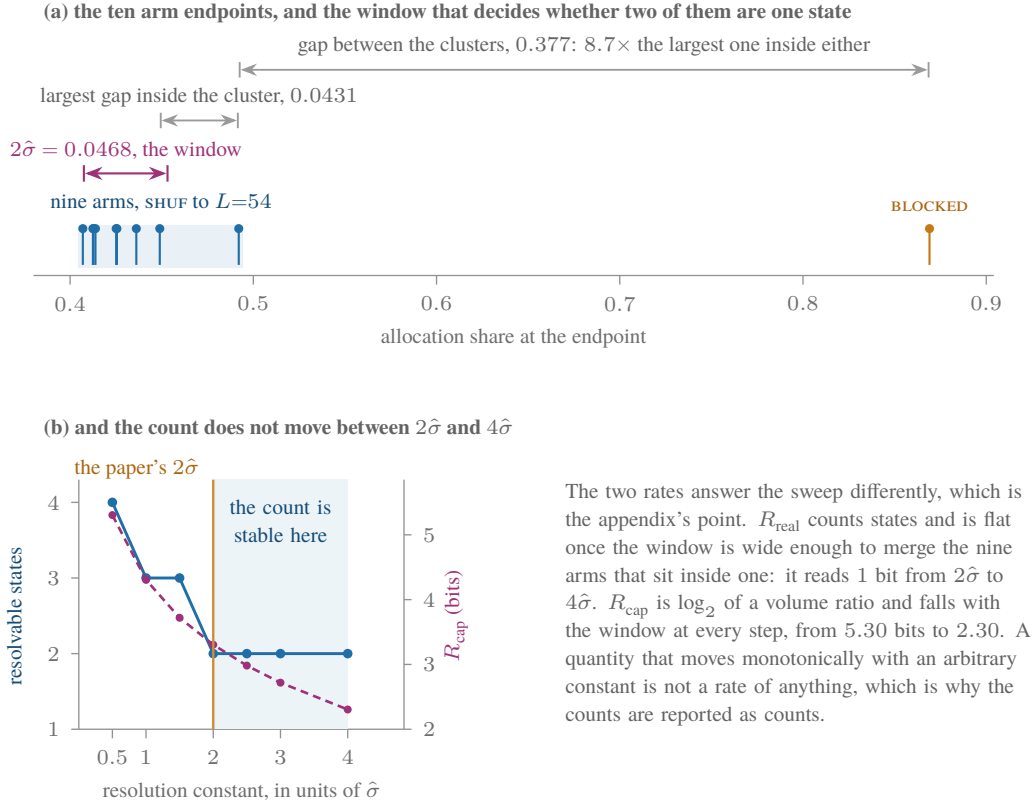
\begin{figure}[t]
\centering
%
%
\begin{tikzpicture}[font=\small,>=Stealth,
 ax/.style={cGrey!75,line width=0.5pt},
 tk/.style={font=\scriptsize,cGrey},
 lb/.style={font=\scriptsize,cGrey!90!black},
 hd/.style={font=\scriptsize\bfseries,cGrey!90!black}]

\begin{scope}
 \node[hd,anchor=west] at (0,3.48) {\textbf{(a)} the ten arm endpoints, and the window that decides whether two of them are one state};
 \draw[ax] (0,0) -- (12.70,0);
 \foreach \v/\x in {0.4/0.4846, 0.5/2.9077, 0.6/5.3308, 0.7/7.7538, 0.8/10.1769, 0.9/12.6000}{
   \draw[ax] (\x,0) -- (\x,-0.12); \node[tk,anchor=north] at (\x,-0.14) {$\v$};}
 \node[tk,anchor=north] at (6.35,-0.56) {allocation share at the endpoint};
 \fill[cV!10] (0.60,0.10) rectangle (2.77,0.66);
 \foreach \x in {0.6531,0.7863,0.7936,0.8203,1.0941,1.1038,1.3606,1.6708,2.7150}{
   \draw[cV,line width=0.8pt] (\x,0.14) -- (\x,0.62); \fill[cV] (\x,0.62) circle (1.7pt);}
 \draw[cA,line width=0.8pt] (11.8497,0.14) -- (11.8497,0.62); \fill[cA] (11.8497,0.62) circle (1.7pt);
 \node[lb,cV!70!black,anchor=south] at (1.69,0.72) {nine arms, \shuf\ to $L{=}54$};
 \node[lb,cA!85!black,anchor=south] at (11.85,0.72) {\blocked};
 \draw[cD,line width=0.8pt,|<->|] (0.6531,1.35) -- (1.7871,1.35);
 \node[lb,cD,anchor=south] at (1.22,1.40) {$2\hat\sigma=0.0468$, the window};
 \draw[cGrey!70,line width=0.7pt,|<->|] (1.6708,2.05) -- (2.7150,2.05);
 \node[lb,anchor=south] at (2.19,2.10) {largest gap inside the cluster, $0.0431$};
 \draw[cGrey!70,line width=0.7pt,|<->|] (2.7150,2.72) -- (11.8497,2.72);
 \node[lb,anchor=south] at (7.28,2.77) {gap between the clusters, $0.377$: $8.7\times$ the largest one inside either};
\end{scope}

\begin{scope}[shift={(0.60,-6.00)}]
 \node[hd,anchor=west] at (-0.60,3.98) {\textbf{(b)} and the count does not move between $2\hat\sigma$ and $4\hat\sigma$};
 \draw[ax] (0,0) -- (4.30,0);
 \draw[ax] (0,0) -- (0,3.30);
 \foreach \v/\x in {0.5/0.444, 1/0.889, 2/1.778, 3/2.667, 4/3.556}{
   \draw[ax] (\x,0) -- (\x,-0.12); \node[tk,anchor=north] at (\x,-0.14) {$\v$};}
 \node[tk,anchor=north] at (2.15,-0.56) {resolution constant, in units of $\hat\sigma$};
 \foreach \v/\y in {1/0.0, 2/1.0, 3/2.0, 4/3.0}{
   \draw[ax] (0,\y) -- (-0.10,\y); \node[tk,anchor=east] at (-0.12,\y) {$\v$};}
 \node[tk,rotate=90,anchor=south,cV!70!black] at (-0.62,1.5) {resolvable states};
 \fill[cV!8] (1.778,0) rectangle (3.556,3.30);
 \node[lb,cV!70!black,anchor=south,align=center,text width=2.0cm] at (2.67,2.34) {the count is stable here};
 \draw[cV,line width=1.1pt] plot coordinates {(0.444,3.0) (0.889,2.0) (1.333,2.0) (1.778,1.0) (2.222,1.0) (2.667,1.0) (3.556,1.0)};
 \foreach \x/\y in {0.444/3.0, 0.889/2.0, 1.333/2.0, 1.778/1.0, 2.222/1.0, 2.667/1.0, 3.556/1.0}
   \fill[cV] (\x,\y) circle (1.8pt);
 \draw[ax] (4.30,0) -- (4.30,3.30);
 \foreach \v/\y in {2/0.0, 3/0.857, 4/1.714, 5/2.571}{
   \draw[ax] (4.30,\y) -- (4.40,\y); \node[tk,anchor=west] at (4.42,\y) {$\v$};}
 \node[tk,rotate=90,anchor=south,cD] at (5.28,1.5) {$R_{\mathrm{cap}}$ (bits)};
 \draw[cD,densely dashed,line width=1.0pt] plot coordinates {(0.444,2.832) (0.889,1.975) (1.333,1.474) (1.778,1.118) (2.222,0.842) (2.667,0.616) (3.556,0.261)};
 \foreach \x/\y in {0.444/2.832, 0.889/1.975, 1.333/1.474, 1.778/1.118, 2.222/0.842, 2.667/0.616, 3.556/0.261}
   \fill[cD] (\x,\y) circle (1.5pt);
 \draw[cA!85,line width=0.9pt] (1.778,0) -- (1.778,3.30);
 \node[lb,cA!85!black,anchor=south east] at (1.72,3.18) {the paper's $2\hat\sigma$};
\end{scope}

\node[lb,anchor=north west,align=left,text width=6.3cm] at (6.90,-2.62)
 {The two rates answer the sweep differently, which is the appendix's point.
  $R_{\mathrm{real}}$ counts states and is flat once the window is wide enough to merge the nine
  arms that sit inside one: it reads $1$ bit from $2\hat\sigma$ to $4\hat\sigma$.
  $R_{\mathrm{cap}}$ is $\log_2$ of a volume ratio and falls with the window at every step, from
  $5.30$ bits to $2.30$. A quantity that moves monotonically with an arbitrary constant is not a
  rate of anything, which is why the counts are reported as counts.};
\end{tikzpicture}
\caption{\textbf{The count, and the two reasons it is a count rather than a rate.}
\textbf{(a)} The ten arm endpoints of the cosine family on the share axis, with the $2\hat\sigma$
window that decides whether two of them are one state drawn at the size it actually is. Nine arms
fall inside a span whose largest internal gap is $0.0431$; the tenth sits $0.377$ away, $8.7$
times that gap. The separation is not a knife-edge and does not need the constant to be exactly
two. \textbf{(b)} The constant swept. The number of resolvable states is flat from $2\hat\sigma$
to $4\hat\sigma$, while $R_{\mathrm{cap}}$, which is $\log_2$ of a volume ratio rather than a
count, falls at every step of the sweep. Sources \texttt{rate\_verdict},
\texttt{rate\_sensitivity}.}
\label{fig:clusters}
\end{figure}

\noindent What the sweep does and does not rescue. It does not make $1$ bit resolution-free: at
$1\hat\sigma$ the first family reads $1.58$. It does establish that the count is $2$ on both
families over a factor of two in resolution above the chosen constant, that the choice is the
\emph{conservative} end of that interval rather than a value picked to produce a round number, and
that what changes below it is one arm separating from a cluster rather than the cluster
dissolving. The claim the paper defends is therefore the ordering, that the path family occupies
a small number of states against a range that would support ten, and not the specific integer.

\emph{The design space offers more than anything we ran carried.} The $1722$ bits are the
entropy of all $\binom{1728}{864}$ orderings under a uniform measure. Ten were executed, and
ten codewords carry at most $\log_2 10=3.32$ bits, so $3.32{:}1$ is the largest compression this
experiment could exhibit and it is the one it does exhibit: ten paths, two states. The finding
is the \emph{gap} between the ten values the range supports and the two it occupies. The channel
is not narrow because it cannot resolve; it is narrow because \textbf{the encoder is
degenerate}, which is what Proposition~\ref{prop:avg} predicts. Sources:
\texttt{rate\_verdict}, \texttt{review\_statistics}.

\emph{Why the body reports counts and not bits.} Writing $\log_2$ of a cluster count invites a channel reading that this measurement does not support. There is no prior over paths and no noise model, so neither rate is a mutual information; ten paths were run, so no ensemble here carries more than $\log_2 10=3.32$ bits whatever the design space offers; and $\Phi$ is one scalar, so a degenerate image in one coordinate is not a degenerate image. The counts and the gap ratios are the measurement. The logarithms were a naming convention and the body no longer uses them.

\section{Corrections to Earlier Versions of This Paper}\label{app:changes}

Several statements in this paper replaced earlier ones that were wrong. The four that moved a
number a reader would quote, or withdrew one, are below with their direction; the rest are listed
at the end in one sentence each.

\begin{enumerate}[leftmargin=1.4em,itemsep=3pt,topsep=3pt]
\item \textbf{The flat interior was the schedule's, and the deciding experiment was named before
it was run.} An earlier version reported the block-length ladder's interior flat and read that as
a property of the data path, while naming the account it could not exclude: that the interior is
flat because a cosine leaves no step size in the second half. The constant-rate ladder has now
been run and decides against the reading we preferred, spanning $0.2221$ in the interior, $11.63$
contrast floors, against $0.0420$ and $2.20$ under the cosine, with an intraclass correlation of
$0.836$ whose interval excludes zero and an allocation monotone in how blocked the arrangement is
(Table~\ref{tab:constlr}, branch \textbf{CL-2}). \emph{The averaging wall as a claim about the
path is withdrawn} from the abstract, \S\ref{sec:limit} and \S\ref{sec:average}. What replaces it
is not weaker: the schedule \emph{is} the averaging operator, and the corner where every result
this paper defends is read moves by one contrast floor between the two schedules against seventeen
at the rung below it. The two-occupied-states count goes with it: it replicates across pretraining
families and does not survive a change of schedule, the constant-rate ladder occupying four states
at the same $2\hat\sigma$.

\item \textbf{The headline was a ratio of incommensurable things, and it is withdrawn.} An
earlier version led with ``$1722$ bits in, $1$ out, a compression of ${\sim}1700\times$''. The
numerator is the entropy of all $\binom{1728}{864}$ orderings under a uniform measure; ten
orderings were run, and ten codewords carry at most $\log_2 10=3.32$ bits. The ratio divided a
prior over a space against a measurement on a sample of it. Withdrawn from the abstract,
\S\ref{sec:limit}, Definition~\ref{def:channel} and Figure~\ref{fig:roadmap}.

\item \textbf{Two power statements were made without the design that produced them.} An earlier
Limitations paragraph quoted a detectable-effect figure of $5.79$ contrast floors, attributed it
to the inherited $\hat\sigma=0.0234$, and named no power level; it was computed at the recomputed
$0.0315$, on two degrees of freedom rather than four, and at $50\%$ power. And \S\ref{sec:average}
said the interior span is below anything this design could resolve, when the minimum detectable
difference is computed at $k{=}4$ and most of the dispersion under it is generation sampling, so
at $k{=}32$ a span of this size is resolvable. Both are now scoped to the evaluation that was
run.

\item \textbf{Two dispersion figures were reported without what they need to be read.} The
between-arm intraclass correlation was given as ``arm identity explains none of the variance'';
the point estimate is $-0.076$ with a bootstrap $95\%$ interval of $[-0.409,+0.166]$, which
supports only the weaker statement now made. And the commitment fingerprint mixed two binnings,
reporting $24$--$30\%$ against $2.5$--$6.5\%$ by counting each problem's best convention on one
side and each (problem, convention) cell on the other; on one binning throughout the factor is
about three and a half rather than seven.
\end{enumerate}

\noindent The remainder, one sentence each. ``Required and not merely observed'' over-read
Remark~\ref{prop:noescape}, which constrains the optimum and not a finite run's endpoint,
corrected at both sites in \S\ref{sec:method}; that remark was a proposition and proves nothing
that needs proving. The resolution constant's effect was stated for $R_{\mathrm{cap}}$ and
asserted of $R_{\mathrm{real}}$, which counts clusters and reads $1.58$ rather than $2$ at
$1\hat\sigma$ on the first family. Definition~\ref{def:decomp}'s antecedent said ``$s$ independent
of $p$'' where the decomposition needs only $\mathrm{Cov}(p,s)=0$ across problems, and an
intermediate version read the measured failure as a refutation of per-problem conditional
independence, which it is not. The onset region was called unconstructible; it is constructible,
and what closes it is the learning non-equivalence of \S\ref{sec:threats}. A hand-derived marked
mean of $0.3204$ appeared in an early draft of \S\ref{sec:key} and is reproduced by no arm or seed
subset, so every number in that section is now recovered from the result base by a named analyzer.
The commutator result was described as a disagreement with a published construction; it is a range
boundary, and \S\ref{sec:related} and Appendix~\ref{app:commutator} now say which of the two the
measurement supports. The claim that the schedule sets the width of the escape hatch was made in
prose and is now Proposition~\ref{prop:sched}. A future-work item still listed the constant-rate
ladder as unrun after \S\ref{sec:average} had reported it, and the abstract stated the schedule
result without the provenance \S\ref{sec:related} gives it. The registration index corrected half
a row, which Appendix~\ref{app:registry} carries because it is evidence about the index rather
than about a result.

\section{The Registration Index: Every Threshold, and When It Was Frozen}\label{app:registry}

Registering a read-out inside the job script the scheduler executes is stronger than a markdown
note in one way and weaker in another. Stronger, because the file carrying the threshold is the
file that ran. Weaker, because a job script is committed when it is written, which is usually but
not always before the run. Table~\ref{tab:registry} carries the gap, measured, rather than the
assurance, and Figure~\ref{fig:registry} puts every interval on one axis.

\emph{It also carries what each registration returned, because an index that lists only thresholds
invites the reader to assume they were met.} Of the seventeen: seven confirmed at their primary
branch, two decided at a later one, two came back partial with the unreadable half named, three
voided on a guard, and three returned no branch at all. The branch numbers are per registration and not a
grade: \textbf{CL-2} is this paper's largest positive result and
\textbf{CM-4} is a void, and nothing in the labels distinguishes them.

\emph{Registered} is the first commit in which the artefact carries the threshold, found with
\texttt{git log -S} on a distinctive fragment of the rule, not the first commit of the file.
\emph{Result} is the modification time of the verdict file on the result base. Both are
machine-checkable and the commands are at the end of this appendix.

\emph{Two rows are evidenced differently and are marked so.} The \textbf{CM} and \textbf{NI}
registrations were frozen on the shared volume the scheduler reads from, and timestamped there,
before the jobs that used them were submitted; they reached a commit only afterwards. A file
timestamp on a volume we can write to is weaker evidence than a commit, and the rows say which
they have rather than borrowing the strength of the rows above them. The weakness is not
theoretical: \textbf{CM}'s registration was later given an outcome section, and that write
replaced the only timestamp its interval rested on, so its $+4$h\,52m is now an assertion this
tree cannot recheck. \textbf{NI}'s file was left untouched after its verdict for exactly that
reason, and \texttt{verify\_dpd\_registry.py} recomputes its interval rather than trusting it.

%
\begingroup
\footnotesize
\setlength{\tabcolsep}{4pt}
\begin{longtable}{@{}>{\raggedright\arraybackslash}p{2.1cm}>{\raggedright\arraybackslash}p{2.25cm}>{\raggedright\arraybackslash}p{2.9cm}l>{\raggedright\arraybackslash}p{1.85cm}r@{}}
\caption{\textbf{Seventeen registrations, what each returned, and the interval each one bought.}
\emph{Seven confirmed at their primary branch, two decided at another, two came back partial,
three void on a guard, and three returned no branch at all.} Fifteen were frozen from about an
hour to nearly three days ahead of the result they decided. Two are not, and both are stated
rather than averaged away: the palindromic check is a linear-model computation of a few seconds whose
thresholds and whose answer entered the repository in the same commit, so it can assert an ordering
it cannot demonstrate; and the amplitude-onset registration was withdrawn before any training
because the rungs it named cannot be built at this budget, which is a fact about the ladder rather
than about either family. The second-family replications carry positive intervals and are
nonetheless printed here together with their answers rather than announced in an earlier version,
which is the weaker of the two guarantees a registration can offer. The last two rows' intervals are
measured against a file timestamp on the shared volume rather than against a commit, which is
weaker again and is why those cells say \emph{no commit} instead of a hash. Of the two only
\textbf{NI} is still recomputable from this tree: \textbf{CM}'s registration gained an outcome
section after its verdict, and writing that section overwrote the timestamp its interval was
measured from. \textbf{NI}'s was left byte-frozen for that reason and its outcome is recorded
here and in its verdict file instead.}
\label{tab:registry} \\
\toprule
claim & outcome & artefact carrying the threshold & registered & result (verdict file, mtime) & gap \\
\midrule
\endfirsthead
\toprule
claim & outcome & artefact carrying the threshold & registered & result (verdict file, mtime) & gap \\
\midrule
\endhead
\bottomrule
\endfoot
\bottomrule
\endlastfoot
instrument, the positive control & \textbf{confirmed} & \texttt{slurm\_\allowbreak natconflict.\allowbreak sbatch} & \texttt{65acb7a} 08-01 12:52 & \texttt{natconflict} 08-03 08:23 & $+1$d\,19h \\
\textbf{v1} batch composition & \textbf{confirmed} & \texttt{slurm\_\allowbreak batchcomp.\allowbreak sbatch}, \textsc{decision rule} & \texttt{2e5f5a9} 08-02 12:42 & \texttt{batchcomp} 08-02 15:32 & $+2$h\,50m \\
\textbf{v2} and the ladder & \emph{no branch fired} & \texttt{slurm\_\allowbreak pathspec.\allowbreak sbatch} and \texttt{PLAN.md}: \emph{predicts} $L^\star\!\approx\!10$ & \texttt{cd53c09} 08-02 16:02 & \texttt{pathspec} 08-03 02:33 & $+10$h\,30m \\
\textbf{M1} the mirror & \textbf{M1} confirmed & \texttt{slurm\_\allowbreak dpd\_\allowbreak mirror.\allowbreak sbatch}: \emph{an asymmetry falsifies the limit-cycle picture} & \texttt{a85c247} 08-03 15:32 & \texttt{dpd\_mirror} 08-06 07:44 & $+2$d\,16h \\
\textbf{C}, \textbf{L} the key & \textbf{C1}, \textbf{L4} confirmed & \texttt{slurm\_\allowbreak disambig.\allowbreak sbatch}, C on the switch and L on the level & \texttt{4830e37} 08-04 06:46 & \texttt{disambig} 08-04 10:49 & $+4$h\,03m \\
\textbf{T1} the mixture ratio & \textbf{T1} confirmed & \texttt{slurm\_\allowbreak ratio2.\allowbreak sbatch}, tolerance $0.05$ frozen with the prediction & \texttt{ff94a95} 08-05 06:22 & \texttt{ratio2} 08-05 16:20 & $+9$h\,58m \\
\textbf{S-C} the single-cosine control & \textbf{S-C4} void & \texttt{prereg/\allowbreak single\_\allowbreak cosine\_\allowbreak control.\allowbreak md}, four branches, S-C4 first & \texttt{39e106d} 08-08 08:41 & \texttt{singlecos} 08-08 16:53 & $+8$h\,12m \\
\midrule
\textbf{S1--S3} the palindromic schedule & \emph{asserted, not demonstrated} & \texttt{verify\_\allowbreak splitting.\allowbreak py} docstring & \texttt{e1eadef} 08-02 11:07 & same commit & \textbf{none} \\
\textbf{Q3-onset} where the amplitude turns on & \emph{withdrawn} & \texttt{prereg/\allowbreak q3\_\allowbreak amplitude\_\allowbreak onset.\allowbreak md}, \emph{withdrawn before any training}: the rungs it names lie across the epoch boundary, and the only route there is the tripled file the S-C control measured as not learning-matched (\S\ref{sec:phase}) & \texttt{79b05eb} 08-11 01:56 & \emph{none possible} & --- \\
\textbf{Q3-key} the key on a second family & \textbf{K1}, \textbf{L4}, \textbf{V1} confirmed & \texttt{prereg/\allowbreak q3\_\allowbreak key\_\allowbreak replication.\allowbreak md} with \texttt{readout\_\allowbreak q3unlock.\allowbreak py} & \texttt{66eab04} 08-09 14:25 & \texttt{q3unlock} 08-11 01:37 & $+1$d\,11h \\
\textbf{Q3-wall} the ladder on a second family & \textbf{W2}, \textbf{M2} partial & \texttt{prereg/\allowbreak q3\_\allowbreak ladder\_\allowbreak replication.\allowbreak md} with \texttt{readout\_\allowbreak q3ladder.\allowbreak py} & \texttt{a679a1b} 08-09 15:41 & \texttt{q3ladder} 08-10 05:04 & $+13$h\,23m \\
\textbf{SC} the schedule control & \textbf{SC-3} decided & \texttt{prereg/\allowbreak schedule\_\allowbreak control.\allowbreak md} with \texttt{readout\_\allowbreak schedctl.\allowbreak py}, guard G-SC first, primary rung fixed at $L{=}27$ & \texttt{a8bcd80} 08-13 00:19 & \texttt{schedctl} 08-14 02:03 & $+1$d\,1h \\
\textbf{BM} the mirror of \blocked & \textbf{BM-2} partial & \texttt{prereg/\allowbreak blocked\_\allowbreak mirror.\allowbreak md} with \texttt{readout\_\allowbreak blocked\_\allowbreak mirror.\allowbreak py}, comparability control read before the branch & \texttt{3fcc1a0} 08-16 12:16 & \texttt{blocked\_\allowbreak mirror} 08-16 15:06 & $+2$h\,50m \\
\textbf{CL} the ladder at a constant rate & \textbf{CL-2} decided & \texttt{prereg/\allowbreak constant\_\allowbreak lr\_\allowbreak ladder.\allowbreak md}, four branches on the interior span against the cosine ladder's & \texttt{a67e3a1} 08-17 06:05 & \texttt{constlr\_\allowbreak ladder} 08-17 14:25 & $+8$h\,20m \\
\textbf{HK} resolution at higher $k$ & \textbf{HK-4} void & \texttt{prereg/\allowbreak highk\_\allowbreak resolution.\allowbreak md} with \texttt{readout\_\allowbreak highk.\allowbreak py}, guards G-HK1--3 read before HK-1--3 & \texttt{1e1daa0} 08-17 14:37 & \texttt{highk} 08-18 13:06 & $+22$h\,29m \\
\textbf{CM} the mirror at a constant rate & \textbf{CM-4} void & \texttt{prereg/\allowbreak constant\_\allowbreak rate\_\allowbreak mirror.\allowbreak md} with \texttt{readout\_\allowbreak constant\_\allowbreak mirror.\allowbreak py}, guard \textbf{CM-4} read first, magnitude registered as exploratory and turning no branch & \emph{no commit} 08-26 05:43 & \texttt{constant\_\allowbreak mirror} 08-26 10:35 & $+4$h\,52m \\
\textbf{NI} whether \texttt{nat}'s weakness is a budget artefact & \textbf{NI-1} confirmed & \texttt{prereg/\allowbreak nat\_\allowbreak budget\_\allowbreak gate.md} with \texttt{readout\_\allowbreak nat\_\allowbreak budget.py}, guards \textbf{G-1}--\textbf{G-3} and a capability bar computed from the family's own dispersion before the run & \emph{no commit} 08-26 15:28 & \texttt{nat\_\allowbreak budget} 08-26 16:37 & $+1$h\,08m \\
\end{longtable}
\endgroup

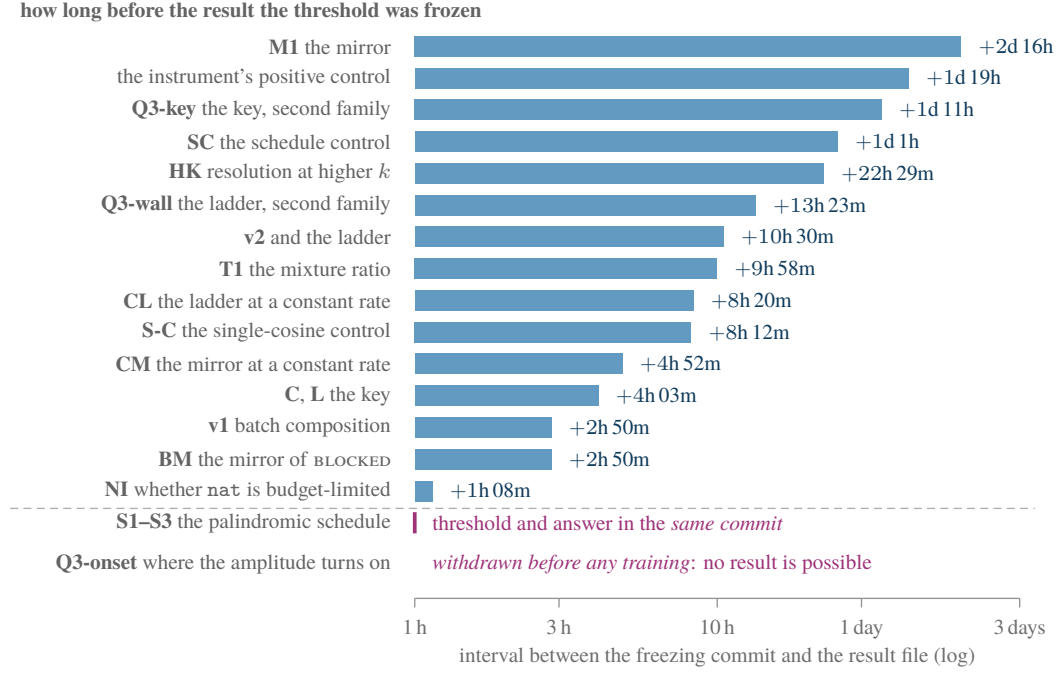
\begin{figure}[t]
\centering
%
%
\begin{tikzpicture}[font=\small,>=Stealth,
 ax/.style={cGrey!75,line width=0.5pt},
 tk/.style={font=\scriptsize,cGrey},
 lb/.style={font=\scriptsize,cGrey!90!black},
 hd/.style={font=\scriptsize\bfseries,cGrey!90!black}]

\def\B{5.35}                        

\node[hd,anchor=west] at (0,7.65) {how long before the result the threshold was frozen};
\draw[ax] (\B,-0.10) -- (\B+8.00,-0.10);
\foreach \v/\x in {{1\,h}/0.000, {3\,h}/1.908, {10\,h}/4.000, {1\,day}/5.908, {3\,days}/8.000}{
  \draw[ax] (\B+\x,-0.10) -- (\B+\x,-0.22); \node[tk,anchor=north] at (\B+\x,-0.24) {\v};}
\node[tk,anchor=north] at (\B+4.00,-0.62) {interval between the freezing commit and the result file (log)};

\foreach \y/\nm/\w/\g in {%
 7.20/{\textbf{M1} the mirror}/1.806/{$+2$d\,16h},
 6.78/{the instrument's positive control}/1.633/{$+1$d\,19h},
 6.36/{\textbf{Q3-key} the key, second family}/1.544/{$+1$d\,11h},
 5.94/{\textbf{SC} the schedule control}/1.398/{$+1$d\,1h},
 5.52/{\textbf{HK} resolution at higher $k$}/1.352/{$+22$h\,29m},
 5.10/{\textbf{Q3-wall} the ladder, second family}/1.127/{$+13$h\,23m},
 4.68/{\textbf{v2} and the ladder}/1.021/{$+10$h\,30m},
 4.26/{\textbf{T1} the mixture ratio}/0.999/{$+9$h\,58m},
 3.84/{\textbf{CL} the ladder at a constant rate}/0.921/{$+8$h\,20m},
 3.42/{\textbf{S-C} the single-cosine control}/0.914/{$+8$h\,12m},
 3.00/{\textbf{CM} the mirror at a constant rate}/0.687/{$+4$h\,52m},
 2.58/{\textbf{C}, \textbf{L} the key}/0.607/{$+4$h\,03m},
 2.16/{\textbf{v1} batch composition}/0.452/{$+2$h\,50m},
 1.74/{\textbf{BM} the mirror of \blocked}/0.452/{$+2$h\,50m},
 1.32/{\textbf{NI} whether \texttt{nat} is budget-limited}/0.058/{$+1$h\,08m}}
{
 \node[lb,anchor=east] at (\B-0.18,\y) {\nm};
 \fill[cV!70] (\B,\y-0.135) rectangle (\B+\w*4.0,\y+0.135);
 \node[lb,cV!55!black,anchor=west] at (\B+\w*4.0+0.10,\y) {\g};
}
\draw[cGrey!55,densely dashed,line width=0.5pt] (0,1.09) -- (13.55,1.09);
\node[lb,anchor=east] at (\B-0.18,0.90) {\textbf{S1--S3} the palindromic schedule};
\draw[cD,line width=1.4pt] (\B,0.76) -- (\B,1.04);
\node[lb,cD,anchor=west] at (\B+0.10,0.90) {threshold and answer in the \emph{same commit}};
\node[lb,anchor=east] at (\B-0.18,0.36) {\textbf{Q3-onset} where the amplitude turns on};
\node[lb,cD,anchor=west] at (\B+0.10,0.36) {\emph{withdrawn before any training}: no result is possible};
\end{tikzpicture}
\caption{\textbf{Seventeen registrations, and the interval each one bought.} Table~\ref{tab:registry}
in one axis: the distance between the commit that froze a threshold and the file that answered
it. Fifteen run from about an hour to nearly three days ahead of their result. The two that
are not intervals are drawn on the same axis rather than dropped from the count: the palindromic check
put its thresholds and its answer in one commit, so it can assert an ordering it cannot
demonstrate, and the amplitude-onset registration was withdrawn before any training because the
rungs it named cannot be built at this budget. Source \texttt{tab\_registry}, whose gap column is
computed from the commit and the verdict file's mtime.}
\label{fig:registry}
\end{figure}

\paragraph{One correction this index made to our own record.} An earlier version of it named
the half-mix job as the ladder's registering artefact. That is wrong: the half-mix verdict is the
volume-matched hierarchy result, a different experiment. The ladder's thresholds live in the
path-spectroscopy job beside v2's, which is correct, because the block-length sweep \emph{is} v2's
deciding test.

\emph{And the correction was applied to half the row.} The artefact column was changed and the
result column was not, so the row named the path-spectroscopy job beside the half-mix era's
timestamp, \texttt{08-03 15:43}, which is \texttt{mixture\_trajectory\_verdict}'s mtime and not
\texttt{pathspec\_verdict}'s. The true interval is $+10$h\,30m rather than the $+23$h\,41m we
printed. Both are positive and neither changes a verdict, but the number was wrong and it was
wrong in the direction that flattered us. It was found by a script that rebuilds this whole table
from \texttt{git} and the result base without reading it (\texttt{verify\_dpd\_registry.py},
released with the paper), which is also why the \emph{result} column now names its verdict file:
the old column printed a bare timestamp, so a row could carry a corrected artefact and an
uncorrected result and look consistent. We report this because an index nobody audits is a claim
rather than a check, and this index was not audited until it was.

\paragraph{How to reproduce any row.} Two commands, one per column:

\begin{center}\footnotesize
\begin{tabular}{@{}l@{}}
\texttt{git log -S'<fragment>' -{}-format='\%h \%ad' -{}- <artefact> | tail -1} \\[2pt]
\texttt{stat -c '\%y \%n' results/<verdict>.json} \\
\end{tabular}
\end{center}

\bibliographystyle{unsrtnat}
{\sloppy\setlength{\emergencystretch}{4em}%
\ifanon\bibliography{references_anon}\else\bibliography{references}\fi\par}

@inproceedings{lopezpaz2017gem,
  title={Gradient Episodic Memory for Continual Learning},
  author={Lopez-Paz, David and Ranzato, Marc'Aurelio},
  booktitle={Advances in Neural Information Processing Systems (NeurIPS)},
  year={2017}, note={arXiv:1706.08840}
}

@article{diazrodriguez2018metrics,
  title={Don't forget, there is more than forgetting: new metrics for Continual Learning},
  author={D{\'i}az-Rodr{\'i}guez, Natalia and Lomonaco, Vincenzo and Filliat, David and Maltoni, Davide},
  journal={arXiv preprint}, year={2018}, note={arXiv:1810.13166}
}

@article{mccloskey1989catastrophic,
  title={Catastrophic Interference in Connectionist Networks: The Sequential Learning Problem},
  author={McCloskey, Michael and Cohen, Neal J.},
  journal={Psychology of Learning and Motivation}, volume={24}, pages={109--165}, year={1989}
}

@article{french1999catastrophic,
  title={Catastrophic Forgetting in Connectionist Networks},
  author={French, Robert M.},
  journal={Trends in Cognitive Sciences}, volume={3}, number={4}, pages={128--135}, year={1999}
}

@article{kirkpatrick2017ewc,
  title={Overcoming Catastrophic Forgetting in Neural Networks},
  author={Kirkpatrick, James and Pascanu, Razvan and Rabinowitz, Neil and others},
  journal={Proceedings of the National Academy of Sciences}, volume={114}, number={13}, year={2017}
}

@article{wu2024cllmsurvey,
  title={Continual Learning for Large Language Models: A Survey},
  author={Wu, Tongtong and Luo, Linhao and Li, Yuan-Fang and Pan, Shirui and Vu, Thuy-Trang and Haffari, Gholamreza},
  journal={arXiv preprint}, year={2024}, note={arXiv:2402.01364}
}

@article{interleaved2025,
  author={Say, Hanne and Ada, Suzan Ece and Ugur, Emre and Asada, Minoru and Oztop, Erhan},
  title={Interleaved Multitask Learning with Energy Modulated Learning Progress},
  journal={arXiv preprint}, year={2025}, note={arXiv:2504.00707}
}

@article{rohrer2007interleaving,
  title={The Shuffling of Mathematics Problems Improves Learning},
  author={Rohrer, Doug and Taylor, Kelli},
  journal={Instructional Science}, volume={35}, number={6}, pages={481--498}, year={2007}
}

@article{kornell2008interleaving,
  title={Learning Concepts and Categories: Is Spacing the ``Enemy of Induction''?},
  author={Kornell, Nate and Bjork, Robert A.},
  journal={Psychological Science}, volume={19}, number={6}, pages={585--592}, year={2008}
}

@article{curriculum2026dynamics,
  author={Elgaar, Mohamed and Amiri, Hadi},
  title={Curriculum Learning for {LLM} Pretraining: An Analysis of Learning Dynamics},
  journal={arXiv preprint}, year={2026}, note={arXiv:2601.21698}
}

@article{lrdecay2025curriculum,
  author={Luo, Kairong and Sun, Zhenbo and Wen, Haodong and Shi, Xinyu and Cui, Jiarui and Dang, Chenyi and Lyu, Kaifeng and Chen, Wenguang},
  title={How Learning Rate Decay Wastes Your Best Data in Curriculum-Based {LLM} Pretraining},
  journal={arXiv preprint}, year={2025}, note={arXiv:2511.18903}
}

@article{beyondrandom2025,
  author={Zhang, Yang and Mohamed, Amr and Abdine, Hadi and Shang, Guokan and Vazirgiannis, Michalis},
  title={Beyond Random Sampling: Efficient Language Model Pretraining via Curriculum Learning},
  journal={arXiv preprint}, year={2025}, note={arXiv:2506.11300}
}

@article{goodcurriculum2025,
  author={Jia, Yaning and Zhang, Chunhui and Diao, Xingjian and Yuan, Xiangchi and Ouyang, Zhongyu and Ma, Chiyu and Vosoughi, Soroush},
  title={What Makes a Good Curriculum? Disentangling the Effects of Data Ordering on
         {LLM} Mathematical Reasoning},
  journal={arXiv preprint}, year={2025}, note={arXiv:2510.19099}
}

@inproceedings{bengio2009curriculum,
  title={Curriculum Learning},
  author={Bengio, Yoshua and Louradour, J{\'e}r{\^o}me and Collobert, Ronan and Weston, Jason},
  booktitle={International Conference on Machine Learning (ICML)}, year={2009}
}

@article{implicitcurriculum2026,
  author={Liu, Emmy and Sun, Kaiser and Li, Millicent and Lee, Isabelle and Tjuatja, Lindia and Huang, Jen-tse and Neubig, Graham},
  title={What Do Language Models Learn and When? The Implicit Curriculum Hypothesis},
  journal={arXiv preprint}, year={2026}, note={arXiv:2604.08510}
}

@inproceedings{doremi,
  title={{DoReMi}: Optimizing Data Mixtures Speeds Up Language Model Pretraining},
  author={Xie, Sang Michael and Pham, Hieu and Dong, Xuanyi and Du, Nan and Liu, Hanxiao and
          Lu, Yifeng and Liang, Percy and Le, Quoc V. and Ma, Tengyu and Yu, Adams Wei},
  booktitle={Advances in Neural Information Processing Systems (NeurIPS)}, year={2023},
  note={arXiv:2305.10429}
}

@article{doge,
  title={{DoGE}: Domain Reweighting with Generalization Estimation},
  author={Fan, Simin and Pagliardini, Matteo and Jaggi, Martin},
  journal={arXiv preprint}, year={2024}, note={arXiv:2310.15393}
}

@article{midtraining2025survey,
  author={Mo, Kaixiang and Shi, Yuxin and Weng, Weiwei and Zhou, Zhiqiang and Liu, Shuman and Zhang, Haibo and Zeng, Anxiang},
  title={Mid-Training of Large Language Models: A Survey},
  journal={arXiv preprint}, year={2025}, note={arXiv:2510.06826}
}

@article{sftrlquagmire2025,
  author={Kang, Feiyang and Kuchnik, Michael and Padthe, Karthik and Vlastelica, Marin and Jia, Ruoxi and Wu, Carole-Jean and Ardalani, Newsha},
  title={Quagmires in {SFT}-{RL} Post-Training: When High {SFT} Scores Mislead and What to Use Instead},
  journal={arXiv preprint}, year={2025}, note={arXiv:2510.01624}
}

@article{conklin2026forgetting,
  title={Learning is Forgetting: {LLM} Training As Lossy Compression},
  author={Conklin, Henry C. and Hosking, Tom and Yi-Chern, Tan and Gold, Julian and
          Cohen, Jonathan D. and Griffiths, Thomas L. and Bartolo, Max and Goldfarb-Tarrant, Seraphina},
  journal={arXiv preprint}, year={2026}, note={arXiv:2604.07569}
}

@article{tishby2015ib,
  title={Deep Learning and the Information Bottleneck Principle},
  author={Tishby, Naftali and Zaslavsky, Noga},
  journal={IEEE Information Theory Workshop (ITW)}, year={2015}
}

@article{forgettingmech2026,
  author={{Gustav Olaf Yunus Laitinen-Fredriksson Lundstrom-Imanov}},
  title={Mechanistic Analysis of Catastrophic Forgetting in Large Language Models During
         Continual Fine-tuning},
  journal={arXiv preprint}, year={2026}, note={arXiv:2601.18699}
}

@misc{openr1math,
  title={{OpenR1-Math-220k}}, author={{Hugging Face}}, year={2025},
  note={\url{https://huggingface.co/datasets/open-r1/OpenR1-Math-220k}}
}

@article{qwen25,
  title={{Qwen2.5} Technical Report},
  author={{Qwen Team}}, journal={arXiv preprint}, year={2024}, note={arXiv:2412.15115}
}

@inproceedings{kwon2023vllm,
  title={Efficient Memory Management for Large Language Model Serving with {PagedAttention}},
  author={Kwon, Woosuk and Li, Zhuohan and Zhuang, Siyuan and others},
  booktitle={Symposium on Operating Systems Principles (SOSP)}, year={2023}
}

@article{chen2026cch,
  title   = {The Capability Convergence Hypothesis: Capability from Access Structure, Not Scale},
  author  = {Chen, Wenhui and Chen, Jianlin and Lin, Ziyao and Vong, Chi Man},
  journal = {arXiv preprint}, year = {2026}, note = {arXiv:2607.14144},
  doi     = {10.5281/zenodo.21714418},
  url     = {https://github.com/wenhui-ml/Capability-Convergence-Hypothesis}
}

@article{gu2025datamixphase,
  title   = {Data Mixing Can Induce Phase Transitions in Knowledge Acquisition},
  author  = {Gu, Xinran and Lyu, Kaifeng and Li, Jiazheng and Zhang, Jingzhao},
  journal = {arXiv preprint arXiv:2505.18091},
  year    = {2025}
}

@article{belkin2019doubledescent,
  title   = {Reconciling Modern Machine-Learning Practice and the Classical Bias--Variance Trade-off},
  author  = {Belkin, Mikhail and Hsu, Daniel and Ma, Siyuan and Mandal, Soumik},
  journal = {Proceedings of the National Academy of Sciences},
  volume  = {116}, number = {32}, pages = {15849--15854}, year = {2019}
}

@inproceedings{nakkiran2020doubledescent,
  title     = {Deep Double Descent: Where Bigger Models and More Data Hurt},
  author    = {Nakkiran, Preetum and Kaplun, Gal and Bansal, Yamini and Yang, Tristan
               and Barak, Boaz and Sutskever, Ilya},
  booktitle = {International Conference on Learning Representations (ICLR)},
  year      = {2020}
}

@article{power2022grokking,
  title   = {Grokking: Generalization Beyond Overfitting on Small Algorithmic Datasets},
  author  = {Power, Alethea and Burda, Yuri and Edwards, Harri and Babuschkin, Igor
             and Misra, Vedant},
  journal = {arXiv preprint arXiv:2201.02177},
  year    = {2022}
}

@inproceedings{schaeffer2023mirage,
  title     = {Are Emergent Abilities of Large Language Models a Mirage?},
  author    = {Schaeffer, Rylan and Miranda, Brando and Koyejo, Sanmi},
  booktitle = {Advances in Neural Information Processing Systems (NeurIPS)},
  year      = {2023}
}

@inproceedings{jacot2018ntk,
  title     = {Neural Tangent Kernel: Convergence and Generalization in Neural Networks},
  author    = {Jacot, Arthur and Gabriel, Franck and Hongler, Cl{\'e}ment},
  booktitle = {Advances in Neural Information Processing Systems (NeurIPS)},
  year      = {2018}
}

@inproceedings{chizat2019lazy,
  title     = {On Lazy Training in Differentiable Programming},
  author    = {Chizat, L{\'e}na{\"i}c and Oyallon, Edouard and Bach, Francis},
  booktitle = {Advances in Neural Information Processing Systems (NeurIPS)},
  year      = {2019}
}

@inproceedings{evron2022catastrophic,
  title     = {How Catastrophic Can Catastrophic Forgetting Be in Linear Regression?},
  author    = {Evron, Itay and Moroshko, Edward and Ward, Rachel and Srebro, Nathan
               and Soudry, Daniel},
  booktitle = {Conference on Learning Theory (COLT)},
  year      = {2022}
}

@inproceedings{lee2021continual,
  title     = {Continual Learning in the Teacher-Student Setup: Impact of Task Similarity},
  author    = {Lee, Sebastian and Goldt, Sebastian and Saxe, Andrew},
  booktitle = {International Conference on Machine Learning (ICML)},
  year      = {2021}
}

@article{shrout1979icc,
  title   = {Intraclass Correlations: Uses in Assessing Rater Reliability},
  author  = {Shrout, Patrick E. and Fleiss, Joseph L.},
  journal = {Psychological Bulletin},
  volume  = {86},
  number  = {2},
  pages   = {420--428},
  year    = {1979}
}

@inproceedings{zhao2024swift,
  title     = {SWIFT: A Scalable Lightweight Infrastructure for Fine-Tuning},
  author    = {Zhao, Yuze and Huang, Jintao and Hu, Jinghan and Wang, Xingjun and Others},
  booktitle = {Proceedings of the AAAI Conference on Artificial Intelligence},
  year      = {2025}
}

@inproceedings{lu2022grab,
  title={{GraB}: Finding Provably Better Data Permutations than Random Reshuffling},
  author={Lu, Yucheng and Guo, Wentao and De Sa, Christopher},
  booktitle={Advances in Neural Information Processing Systems (NeurIPS)}, year={2022},
  note={arXiv:2205.10733}
}

@inproceedings{commonit2024,
  author={Rao, Jun and Liu, Xuebo and Lian, Lian and Cheng, Shengjun and Liao, Yunjie and Zhang, Min},
  title={{CommonIT}: Commonality-Aware Instruction Tuning for Large Language Models
         via Data Partitions},
  booktitle={Proceedings of EMNLP}, year={2024}, note={arXiv:2410.03077}
}

@inproceedings{dataorg2026,
  author={Dai, Yalun and Huang, Yangyu and Yang, Tongshen and Wang, Yonghan and Zhang, Xin and Wu, Wenshan and Zhao, Qihao and Li, Hao and Gao, Yuanyuan and Yap, Kim-Hui and Li, Scarlett},
  title={Demystifying Data Organization for Enhanced {LLM} Training},
  booktitle={Proceedings of ACL}, year={2026}, note={arXiv:2605.30334}
}

@article{sweeney2026shuffle,
  title={Optimizer Memory Makes Shuffle Order a First-Order Source of Fine-Tuning Noise},
  author={Sweeney, John}, journal={arXiv preprint}, year={2026}, note={arXiv:2606.29554}
}

@inproceedings{sweeney2026liebracket,
  title={The Geometry of Sequential Learning: {Lie}-Bracket Prediction of Transfer Order},
  author={Sweeney, John},
  booktitle={International Conference on Machine Learning (ICML)}, year={2026},
  note={arXiv:2606.24993}
}

@inproceedings{pruthi2020tracin,
  title={Estimating Training Data Influence by Tracing Gradient Descent},
  author={Pruthi, Garima and Liu, Frederick and Kale, Satyen and Sundararajan, Mukund},
  booktitle={Advances in Neural Information Processing Systems (NeurIPS)}, year={2020}
}

@article{influencedynamics2025,
  author={Lee, Jin Hwa and Smith, Matthew and Adam, Maxwell and Hoogland, Jesse},
  title={Influence Dynamics and Stagewise Data Attribution},
  journal={arXiv preprint}, year={2025}, note={arXiv:2510.12071}
}

@inproceedings{yu2020pcgrad,
  title={Gradient Surgery for Multi-Task Learning},
  author={Yu, Tianhe and Kumar, Saurabh and Gupta, Abhishek and Levine, Sergey and
          Hausman, Karol and Finn, Chelsea},
  booktitle={Advances in Neural Information Processing Systems (NeurIPS)}, year={2020}
}

@inproceedings{liu2021cagrad,
  title={Conflict-Averse Gradient Descent for Multi-task Learning},
  author={Liu, Bo and Liu, Xingchao and Jin, Xiaojie and Stone, Peter and Liu, Qiang},
  booktitle={Advances in Neural Information Processing Systems (NeurIPS)}, year={2021}
}

@inproceedings{navon2022nashmtl,
  title={Multi-Task Learning as a Bargaining Game},
  author={Navon, Aviv and Shamsian, Aviv and Achituve, Idan and Maron, Haggai and
          Kawaguchi, Kenji and Chechik, Gal and Fetaya, Ethan},
  booktitle={International Conference on Machine Learning (ICML)}, year={2022}
}

@article{pairwisefragile2026,
  author={Piontkovskaia, Irina and Nikolenko, Sergey},
  title={First-Order Predictable but Pairwise Fragile: Local Task Adaptation in Trained
         Transformers},
  journal={arXiv preprint}, year={2026}, note={arXiv:2607.16821}
}

@article{incompletelearning2026,
  title={Why Supervised Fine-Tuning Fails to Learn: A Systematic Study of Incomplete Learning
         in Large Language Models},
  author={Xue, et al.},
  journal={arXiv preprint}, year={2026}, note={arXiv:2604.10079}
}

@book{bogoliubov1961asymptotic,
  title={Asymptotic Methods in the Theory of Non-Linear Oscillations},
  author={Bogoliubov, Nikolay N. and Mitropolsky, Yuri A.},
  publisher={Gordon and Breach}, year={1961}
}

@article{kapitza1951pendulum,
  title={Dynamic Stability of a Pendulum with an Oscillating Point of Suspension},
  author={Kapitza, Pyotr L.},
  journal={Journal of Experimental and Theoretical Physics}, volume={21}, pages={588--597},
  year={1951}
}

@misc{frl2026,
  title  = {The Free-Recipe Limit: Every Measured Recipe Effect Is a Gauge of One Broken
            Premise of the Ideal},
  author = {Chen, Wenhui and Chen, Jianlin and Lin, Ziyao and Vong, Chi Man},
  year   = {2026},
  note   = {Code and artefacts archived; every number cited here is checkable there},
  doi    = {10.5281/zenodo.21905881},
  url    = {https://github.com/wenhui-ml/free-recipe-limit}
}

@inproceedings{xu2024knowconflict,
  title={Knowledge Conflicts for {LLM}s: A Survey},
  author={Xu, Rongwu and Qi, Zehan and Guo, Zhijiang and Wang, Cunxiang and Wang, Hongru
          and Zhang, Yue and Xu, Wei},
  booktitle={Proceedings of the 2024 Conference on Empirical Methods in Natural Language
             Processing (EMNLP)},
  year={2024}, note={arXiv:2403.08319}
}

@inproceedings{plank2022variation,
  title={The ``Problem'' of Human Label Variation: On Ground Truth in Data, Modeling and
         Evaluation},
  author={Plank, Barbara},
  booktitle={Proceedings of the 2022 Conference on Empirical Methods in Natural Language
             Processing (EMNLP)},
  year={2022}, note={arXiv:2211.02570}
}

@inproceedings{hendrycks2021math,
  title={Measuring Mathematical Problem Solving With the MATH Dataset},
  author={Hendrycks, Dan and Burns, Collin and Kadavath, Saurav and Arora, Akul and Basart,
          Steven and Tang, Eric and Song, Dawn and Steinhardt, Jacob},
  booktitle={Advances in Neural Information Processing Systems, Datasets and Benchmarks Track
             (NeurIPS)}, year={2021}
}

@article{keskar2019ctrl,
  title  = {{CTRL}: A Conditional Transformer Language Model for Controllable Generation},
  author = {Keskar, Nitish Shirish and McCann, Bryan and Varshney, Lav R. and Xiong, Caiming and Socher, Richard},
  journal= {arXiv preprint arXiv:1909.05858},
  year   = {2019}
}

@inproceedings{korbak2023pretraining,
  title    = {Pretraining Language Models with Human Preferences},
  author   = {Korbak, Tomasz and Shi, Kejian and Chen, Angelica and Bhalerao, Rasika and Buckley, Christopher L. and Phang, Jason and Bowman, Samuel R. and Perez, Ethan},
  booktitle= {Proceedings of the 40th International Conference on Machine Learning (ICML)},
  year     = {2023}
}

@article{gao2025meco,
  title  = {Metadata Conditioning Accelerates Language Model Pre-training},
  author = {Gao, Tianyu and Wettig, Alexander and He, Luxi and Dong, Yihe and Malladi, Sadhika and Chen, Danqi},
  journal= {arXiv preprint arXiv:2501.01956},
  year   = {2025}
}

@inproceedings{khalifa2024sourceaware,
  title    = {Source-Aware Training Enables Knowledge Attribution in Language Models},
  author   = {Khalifa, Muhammad and Wadden, David and Strubell, Emma and Lee, Honglak and Wang, Lu and Beltagy, Iz and Peng, Hao},
  booktitle= {Conference on Language Modeling (COLM)},
  year     = {2024}
}

@article{higuchi2025whenmetadata,
  title  = {When Does Metadata Conditioning ({NOT}) Work for Language Model Pre-Training? A Study with Context-Free Grammars},
  author = {Higuchi, Rei and Kawata, Ryotaro and Nishikawa, Naoki and Oko, Kazusato and Yamaguchi, Shoichiro and Kobayashi, Sosuke and Tokui, Seiya and Hayashi, Kohei and Okanohara, Daisuke and Suzuki, Taiji},
  journal= {arXiv preprint arXiv:2504.17562},
  year   = {2025}
}

@article{krestnikov2026truth,
  title  = {Truth as a Compression Artifact in Language Model Training},
  author = {Krestnikov, Konstantin},
  journal= {arXiv preprint arXiv:2603.11749},
  year   = {2026}
}

@article{ledoux2026order,
  title={The Order Is The Message},
  author={LeDoux, Jordan},
  journal={arXiv preprint arXiv:2603.25047},
  year={2026}
}

@article{ju2024imbalance,
  title={Mitigating Training Imbalance in {LLM} Fine-Tuning via Selective Parameter Merging},
  author={Ju, Yiming and Ni, Ziyi and Xing, Xingrun and Zeng, Zhixiong and Zhao, Hanyu and Fan, Siqi and Zhang, Zheng},
  journal={arXiv preprint arXiv:2410.03743},
  year={2024}
}

@inproceedings{holtzman2021surface,
  title={Surface Form Competition: Why the Highest Probability Answer Isn't Always Right},
  author={Holtzman, Ari and West, Peter and Shwartz, Vered and Choi, Yejin and Zettlemoyer, Luke},
  booktitle={Proceedings of EMNLP},
  year={2021}
}

@article{janeiro2026paraeval,
  title={Are We Evaluating Knowledge or Phrasing? Mitigating {MCQA} Sensitivity with {ParaEval}},
  author={Janeiro, Jo{\~a}o Maria and Videau, Mathurin and Caciolai, Andrea and Piwowarski, Benjamin and Gallinari, Patrick and Barrault, Loic},
  journal={arXiv preprint arXiv:2606.10657},
  year={2026}
}

@unpublished{gold2026,
  title={Which Corpus Supplies the Gold: A Free Variable in Exact-Match Evaluation},
  author={Chen, Wenhui}, year={2026},
  note={Companion manuscript. The measurement half of the underdetermined stream: the corpus survey, the paired read-out, and the leaderboard inversions}
}

@unpublished{mark2026,
  title={Writing a Convention Into Weights: A Dose Response in Marker Reliability, and One Undecided Capacity Row},
  author={Chen, Wenhui}, year={2026},
  note={Companion manuscript. The training-side half: the key's information against its presence, the dose law, and the capacity row its corpus could not decide}
}

@article{krasheninnikov2025recency,
  title   = {Fresh in Memory: Training-Order Recency is Linearly Encoded in Language Model
             Activations},
  author  = {Krasheninnikov, Dmitrii and Turner, Richard E. and Krueger, David},
  journal = {arXiv preprint}, year = {2025}, note = {arXiv:2509.14223}
}

@article{ajroldi2025averaging,
  title   = {When, Where and Why to Average Weights?},
  author  = {Ajroldi, Niccol{\`o} and Orvieto, Antonio and Geiping, Jonas},
  journal = {arXiv preprint}, year = {2025}, note = {arXiv:2502.06761}
}

@incollection{bertsekas2011incremental,
  title     = {Incremental Gradient, Subgradient, and Proximal Methods for Convex
               Optimization: A Survey},
  author    = {Bertsekas, Dimitri P.},
  booktitle = {Optimization for Machine Learning},
  publisher = {MIT Press}, pages = {85--119}, year = {2011}
}

@article{nedic2001incremental,
  title   = {Incremental Subgradient Methods for Nondifferentiable Optimization},
  author  = {Nedi{\'c}, Angelia and Bertsekas, Dimitri P.},
  journal = {SIAM Journal on Optimization}, volume = {12}, number = {1},
  pages   = {109--138}, year = {2001}
}

@article{gurbuzbalaban2021why,
  title   = {Why Random Reshuffling Beats Stochastic Gradient Descent},
  author  = {G{\"u}rb{\"u}zbalaban, Mert and Ozdaglar, Asuman and Parrilo, Pablo A.},
  journal = {Mathematical Programming}, volume = {186}, pages = {49--84}, year = {2021}
}

@inproceedings{haochen2019random,
  title     = {Random Shuffling Beats {SGD} after Finite Epochs},
  author    = {HaoChen, Jeff Z. and Sra, Suvrit},
  booktitle = {International Conference on Machine Learning (ICML)}, year = {2019}
}

@inproceedings{mishchenko2020random,
  title     = {Random Reshuffling: Simple Analysis with Vast Improvements},
  author    = {Mishchenko, Konstantin and Khaled, Ahmed and Richt{\'a}rik, Peter},
  booktitle = {Advances in Neural Information Processing Systems (NeurIPS)}, year = {2020},
  note      = {arXiv:2006.05988}
}

@article{nguyen2026shuffle,
  title   = {Learning to Shuffle: Block Reshuffling and Reversal Schemes for Stochastic
             Optimization},
  author  = {Nguyen, Lam M. and Phan, Dzung T. and Kalagnanam, Jayant},
  journal = {arXiv preprint}, year = {2026}, note = {arXiv:2604.00260}
}

@article{emmanouilidis2026shuffling,
  title   = {Shuffling the Data, Stretching the Step-Size: Sharper Bias in Constant
             Step-Size {SGD}},
  author  = {Emmanouilidis, Konstantinos and Vlatakis-Gkaragkounis, Emmanouil-Vasileios and
             Vidal, Rene},
  journal = {arXiv preprint}, year = {2026}, note = {arXiv:2604.10373}
}

@article{trotter1959product,
  title   = {On the Product of Semi-Groups of Operators},
  author  = {Trotter, H. F.},
  journal = {Proceedings of the American Mathematical Society}, volume = {10}, number = {4},
  pages   = {545--551}, year = {1959}
}

@article{strang1968construction,
  title   = {On the Construction and Comparison of Difference Schemes},
  author  = {Strang, Gilbert},
  journal = {SIAM Journal on Numerical Analysis}, volume = {5}, number = {3},
  pages   = {506--517}, year = {1968}
}

@book{hairer2006geometric,
  title     = {Geometric Numerical Integration: Structure-Preserving Algorithms for
               Ordinary Differential Equations},
  author    = {Hairer, Ernst and Lubich, Christian and Wanner, Gerhard},
  publisher = {Springer}, edition = {2nd}, year = {2006}
}

@article{mclachlan2002splitting,
  title   = {Splitting Methods},
  author  = {McLachlan, Robert I. and Quispel, G. Reinout W.},
  journal = {Acta Numerica}, volume = {11}, pages = {341--434}, year = {2002}
}

@article{blanes2009magnus,
  title   = {The {M}agnus Expansion and Some of Its Applications},
  author  = {Blanes, S. and Casas, F. and Oteo, J. A. and Ros, J.},
  journal = {Physics Reports}, volume = {470}, number = {5--6}, pages = {151--238},
  year    = {2009}
}

@article{yeom2026commitment,
  title   = {Hallucination as Commitment Failure: Larger {LLM}s Misfire Despite Knowing
             the Answer},
  author  = {Yeom, Jewon and Sok, Jaewon and Kim, Heejun and Park, Seonghyeon and
             Park, Jeongjae and Kim, Taesup},
  journal = {arXiv preprint}, year = {2026}, note = {arXiv:2605.22007}
}

\end{document}